\documentclass{article}
\usepackage[preprint]{neurips_2026}
\usepackage[T1]{fontenc}
\usepackage[utf8]{inputenc}

\usepackage{amsmath,amsfonts,bm}

\def\eqref#1{equation~\ref{#1}}
\def\1{\bm{1}}

\DeclareMathAlphabet{\mathsfit}{\encodingdefault}{\sfdefault}{m}{sl}
\SetMathAlphabet{\mathsfit}{bold}{\encodingdefault}{\sfdefault}{bx}{n}

\usepackage{amsmath,amssymb,amsthm}
\usepackage{graphicx}
\usepackage{booktabs}
\usepackage{array}
\usepackage{xcolor}
\usepackage{colortbl}
\usepackage{pifont}
\usepackage{xspace}
\usepackage{caption}
\usepackage{wrapfig}
\usepackage{fvextra}
\usepackage{placeins}
\usepackage{url}
\usepackage{hyperref}

\usepackage{fvextra}
\usepackage{tcolorbox}
\tcbuselibrary{skins,breakable}
\usepackage{needspace}

\definecolor{mrPromptInk}{HTML}{293D36}
\definecolor{mrPromptBorder}{HTML}{BECBC3}
\definecolor{mrPromptPaper}{HTML}{FAFBF9}
\definecolor{mrPromptHeader}{HTML}{E5EFD9}

\newcounter{mrprompt}
\renewcommand{\themrprompt}{\arabic{mrprompt}}
\tcbset{
  mr prompt/.style={
    enhanced,breakable,
    colback=mrPromptPaper,colframe=mrPromptBorder,
    colbacktitle=mrPromptHeader,coltitle=mrPromptInk,
    boxrule=0.45pt,titlerule=0pt,arc=1.2mm,
    left=3mm,right=3mm,top=2mm,bottom=2mm,
    toptitle=2mm,bottomtitle=2mm,
    boxsep=0pt,before skip=8pt,after skip=10pt,
    fonttitle=\normalfont\small\bfseries,
    pad at break*=2mm,lines before break=4,
  },
  casebox/repro/.style={},
  casebox/interface/.style={},
  casebox/search/.style={colbacktitle=blue!7!white},
  casebox/review/.style={colbacktitle=mrPromptHeader},
  casebox/memory/.style={colbacktitle=orange!10!white},
  casebox/clause/.style={colbacktitle=black!6!white},
}

\newenvironment{casebox}[3]{%
  \VerbatimEnvironment
  \refstepcounter{mrprompt}%
  \begin{tcolorbox}[mr prompt,casebox/#1,
    title={Prompt \themrprompt\enspace\textbar\enspace #2},
    label={#3},
    title after break={Prompt \themrprompt\enspace\textbar\enspace continued}]
  \begin{Verbatim}[
    formatcom={\renewcommand{\textemdash}{{\rmfamily\char124}}},
    fontsize=\fontsize{8}{10}\selectfont,
    breaklines=true,breakafter={,/_.},
    breaksymbolleft={},breakindent=1em,
    xleftmargin=0pt,xrightmargin=0pt]%
}{%
  \end{Verbatim}
  \end{tcolorbox}%
}

\newcommand{\mrpromptsource}[1]{%
  \\[2pt]{\normalfont\fontsize{7.2}{9}\selectfont\ttfamily #1}%
}

\hypersetup{
  hidelinks,
  pdftitle={MERID: Multimodal Exploration via Recursive Self-Improvement Agents for Major Depression Analysis},
  pdfauthor={Lei Liu, Zhaokang Liang, Qingcheng Zeng, Chenda Duan, Lu Mi, Zhen Tan, Tianyu Liu}
}

\newtheorem{mrproposition}{Proposition}

\newtheorem{mrassumption}{Assumption}

\theoremstyle{definition}
\newtheorem{mrdefinition}{Definition}

\theoremstyle{plain}

\newcommand{\mrinputsection}[1]{%
  \IfFileExists{#1.tex}
    {\input{#1.tex}}
    {\input{section_folders/#1.tex}}%
}

\DeclareRobustCommand{\method}{%
  {\fontfamily{lmtt}\selectfont\bfseries MERID}\xspace%
}
\title{MERID: Multimodal Exploration via Recursive Self-Improvement Agents for Major Depression Analysis}

\newcommand{\paperauthors}{%
  \parbox{\dimexpr\textwidth-2\tabcolsep\relax}{%
    \centering\normalfont\small
    \setlength{\parskip}{0pt}%

    \mbox{Lei Liu$^{1,*}$} \quad
    \mbox{Zhaokang Liang$^{2,*}$} \quad
    \mbox{Qingcheng Zeng$^{3}$} \quad
    \mbox{Chenda Duan$^{4}$} \\[0.4em]
    \mbox{Lu Mi$^{5}$} \quad
    \mbox{Zhen Tan$^{6}$} \quad
    \mbox{Tianyu Liu$^{1,5, \dagger}$}

    \par\vspace{0.6em}

    {\footnotesize
      \mbox{$^{1}$Yale University} \quad
      \mbox{$^{2}$Zhejiang University} \\[0.2em]
      \mbox{$^{3}$Northwestern University} \quad
      \mbox{$^{4}$University of California, Los Angeles} \\[0.2em]
      \mbox{$^{5}$Tsinghua University} \quad
      \mbox{$^{6}$Stevens Institute of Technology}
      \par
    }

    \vspace{0.4em}
    {\footnotesize $^{*}$Equal contribution.\quad$^{\dagger}$Corresponding author.\par}
  }%
}

\author{\paperauthors}
\date{}

\makeatletter
\let\mrOriginalTopTitleBar\@toptitlebar
\renewcommand{\@toptitlebar}{%
  \hbox to \textwidth{%
    \hfil
    \includegraphics[width=0.96\textwidth]{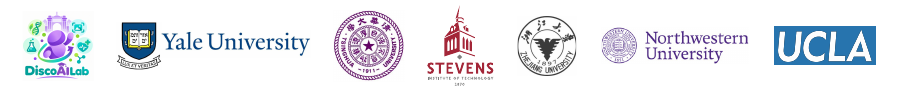}%
    \hfil
  }%
  \vskip 8pt
  \mrOriginalTopTitleBar
}
\makeatother

\begin{document}

\maketitle

\begin{abstract}
  Major depressive disorder (MDD) severely impacts daily activities and quality of life. Detecting MDD involves multimodal data, such as interview recordings and sensor measurements. This is particularly challenging, as these heterogeneous modalities often demand distinct, customized prediction pipelines. Existing efforts to address this challenge have explored both manually engineered multimodal architectures and agent-assisted pipeline development. Despite their progress, it remains challenging to autonomously revise pipelines based on experimental feedback and carry verified improvements forward into subsequent designs. To this end, we propose \underline{M}ultimodal \underline{E}xploration via \underline{R}ecursive Self-\underline{I}mprovement Agents for Major \underline{D}epression Analysis (\method{}). The framework develops depression pipelines through experience-based recursive self-improvement (RSI). Grounded State Construction (GSC) grounds experience by aligning multimodal records with subject-level depression targets. Coupled Pipeline Exploration (CPE) jointly modifies representations, fusion, and predictors to build successor pipelines for classification and severity estimation. Evidence-Guided Evolution (EGE) guides revisions through feedback and verifies gains under uncertainty in small depression cohorts before inheritance. Extensive experiments on depression benchmarks show that \method{} achieves the best results on multiple tasks compared with multimodal and agent-based baselines. Further analysis highlights the value of acoustic and linguistic cues for depression detection. Our code is available at \url{https://github.com/DiscoAILab/MERID.git}
% \url{https://anonymous.4open.science/r/MERID}
\begin{center}
\begin{minipage}{\linewidth}
\centering
\includegraphics[width=\linewidth]{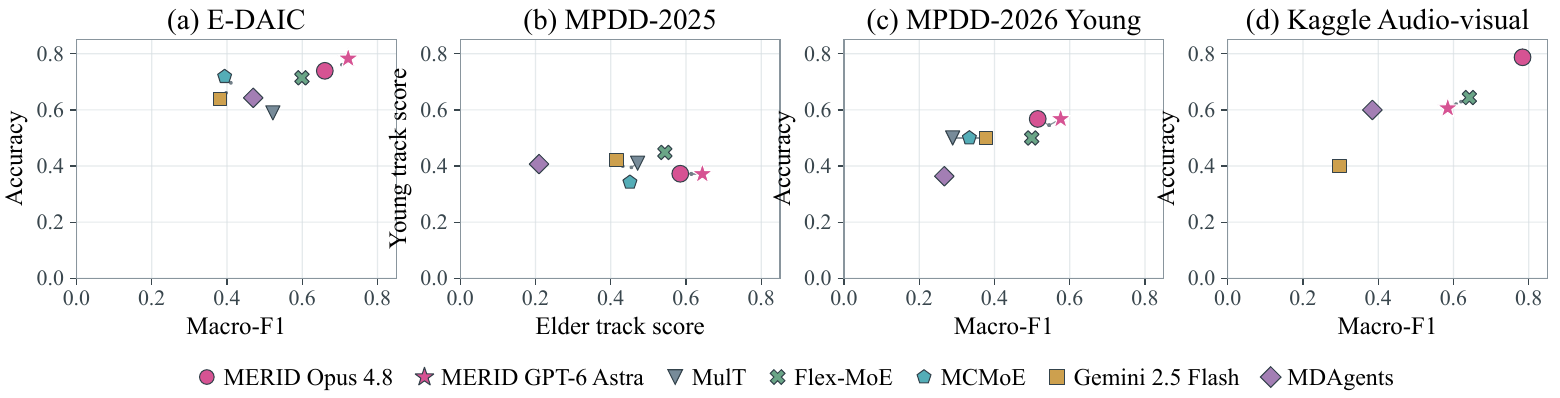}
\captionsetup{hypcap=false,type=figure,position=bottom,skip=3pt}
% \caption{Performance on depression benchmarks and Kaggle emotion proxy task.}
\caption{\method{} achieves leading performance against state-of-the-art methods on multimodal depression benchmarks. Markers show paired scores in each panel.}
\label{fig:metric-profiles}
\label{fig:performance-teaser}
\vspace{-1em}
\end{minipage}
\end{center}

\end{abstract}

\section{Introduction}

Depression, as a representative example of mental health problem, is a complex clinical construct whose observable manifestations are distributed across heterogeneous and often weak behavioral signals~\citep{shaffer2022allostasis}. Accordingly, multimodal detection combines speech, facial behavior, and language, each reflecting distinct aspects of the disorder~\citep{cohn2018multimodal,ringeval2019avec}. Slower facial movements, for instance, can reflect psychomotor slowing, while reduced vocal prosody and longer pauses reveal altered speech behavior~\citep{kacem2018detecting,alpert2001reflections}. However, these modalities differ substantially in temporal resolution and feature structure, requiring prediction pipelines with tailored preprocessing, fusion, and prediction components to map multimodal observations to a diagnostic category or symptom-severity score~\citep{perezrua2019mfas}. Designing such pipelines is inherently iterative: each experiment evaluates a particular combination of representations, fusion strategies, and predictors, and the resulting evidence guides subsequent design refinements~\citep{guo2024dsagent,lu2026aiscientist,trirat2025automlagent}.

\begin{figure}[t]
\centering
\includegraphics[width=0.9\linewidth]{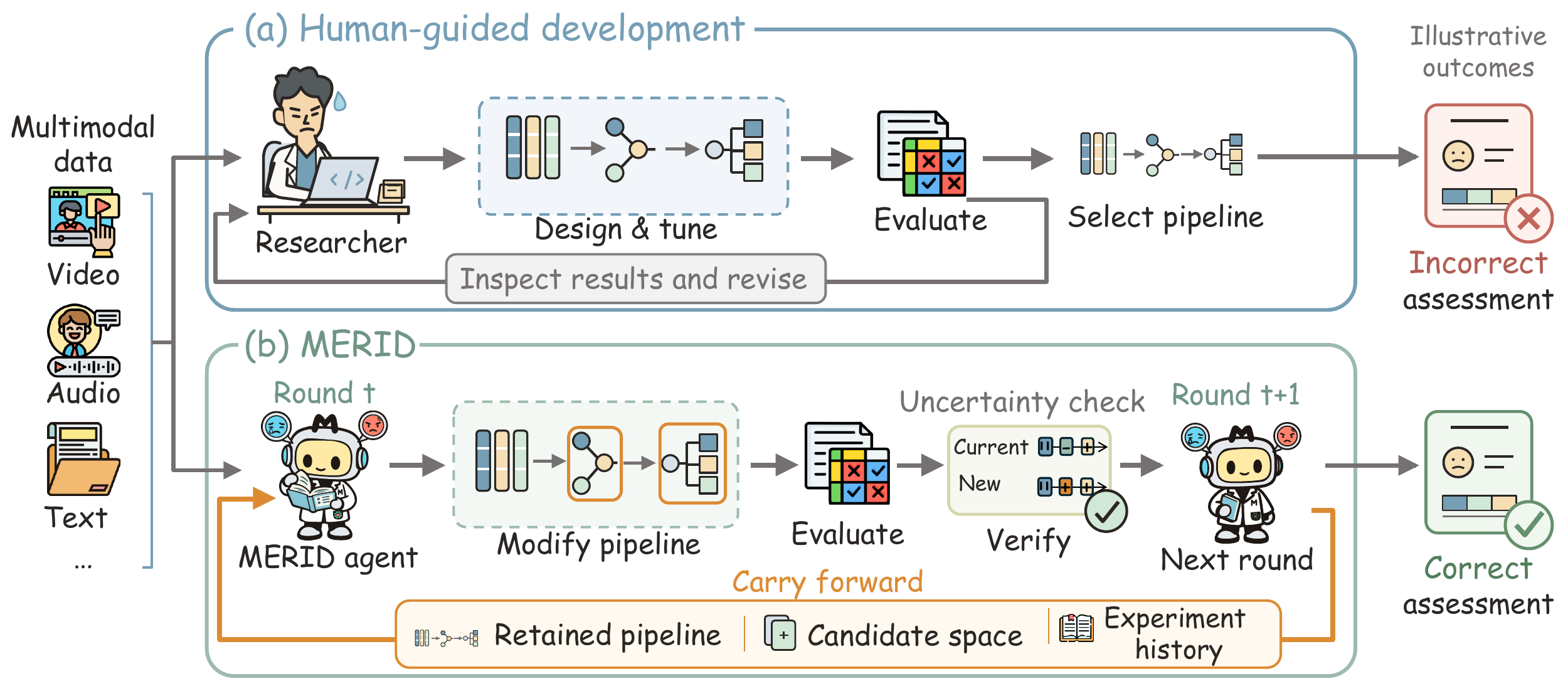}
\caption{Conventional approaches require repeated experiments and manual pipeline tuning for depression detection. \method{} uses experimental feedback to revise pipelines and verify improvements. Retained pipelines and experience guide recursive self-improvement for depression assessment.}
\label{fig:teaser}
\vspace{-2em}
\end{figure}

Existing research relevant to multimodal depression modeling spans three levels of automation. At the model level, most work relies on manually designed multimodal architectures that target specific design choices such as cross-modal interaction~\citep{liu2025depformer}, personality-aware fusion~\citep{shi2025hope}, and temporal aggregation~\citep{liu2025msfats}; \citet{perezrua2019mfas} further extend this by posing the fusion design itself as a neural architecture search problem. At the pipeline level, LLM-based agents automate broader development workflows: \citet{trirat2025automlagent} and \citet{guo2024dsagent} coordinate pipeline construction and iteratively refine solutions through experimental feedback, while \citet{lu2026aiscientist} chain hypothesis generation, implementation, and experimentation into a single automated loop. At the agent level, recent work treats the agent's own design as an optimization target. \citet{hu2025adas} generate new agent implementations from an evolving archive of prior designs, \citet{zhang2025darwingodel} iteratively rewrite the agent's own codebase and validate each modification empirically, and recent work on recursive self-improvement formalizes how verified gains can be inherited across successive improvement rounds~\citep{duan2026rsi,zhu2026rsiagent}.

Despite these advances, turning experimental feedback into reliable recursive pipeline improvement remains challenging in multimodal depression modeling. Executable workflows do not by themselves ensure comparable experiments: in depression data, changes in how recordings are organized can alter subject-target correspondence, distorting candidate comparisons and the feedback used in later revisions~\citep{saeb2017usecase,danylenko2026pitfalls}. Existing search methods cover only part of the revision space, typically optimizing fusion while holding other pipeline stages fixed~\citep{perezrua2019mfas}, even though interactions among stages can determine the value of a proposed change. Furthermore, a higher development score does not necessarily justify retaining a revision, as repeated selection on small depression cohorts can favor sampling fluctuations over genuine improvements~\citep{varoquaux2018crossvalidation,cawley2010overfitting,dwork2015reusable}. Reliable recursive self-improvement therefore requires comparable experience, coordinated pipeline revisions, and uncertainty-aware inheritance decisions. This motivates the central question of this work:
\vspace{-2mm}
\begin{quote}
\itshape
How can experimental experience drive reliable recursive self-improvement for multimodal depression detection?
\end{quote}
\vspace{-2mm}
To answer this question, we propose \underline{M}ultimodal \underline{E}xploration via \underline{R}ecursive Self-\underline{I}mprovement Agents for Major \underline{D}epression Analysis (\method{}). The framework organizes pipeline development as an experience-based RSI loop that combines coupled exploration with uncertainty-aware inheritance. Grounded State Construction (GSC) preserves subject-target-record correspondence across recipes, providing a consistent basis for comparing revisions. Coupled Pipeline Exploration (CPE) treats representation, fusion, and predictor choices as a joint improvement target, evaluating combinations of newly proposed and retained components under a shared protocol. Evidence-Guided Evolution (EGE) uses experimental feedback to guide revisions and an acceptance margin based on subject-level score uncertainty to verify pipeline replacements. Across rounds, the retained pipeline and expanded candidate space are carried forward with updated experience to guide subsequent proposals. Experiments across depression benchmarks show that \method{} achieves the best results on multiple tasks compared with multimodal and agent-based baselines. Our contributions are:

\hangindent=1.1em\hangafter=1\noindent\ding{182}\textbf{ \emph{New Perspective}}: We formulate multimodal depression pipeline development as an experience-based RSI process, using experimental feedback to guide autonomous revisions and verified improvements to inform subsequent designs. The retained pipeline, expanded candidate space, and accumulated experience jointly provide the basis for the next round of exploration.

\hangindent=1.1em\hangafter=1\noindent\ding{183}\textbf{ \emph{Novel Methodology}}: We propose \method{}, integrating correspondence invariance, coupled pipeline exploration, and uncertainty-aware inheritance to evaluate coordinated pipeline modifications under a shared protocol across successive RSI rounds on fixed validation splits. An explicit uncertainty-scaled acceptance margin governs pipeline replacement, and we derive a retention-margin bound for the gap to the best evaluated development score under additive revisions.

\hangindent=1.1em\hangafter=1\noindent\ding{184}\textbf{ \emph{Empirical Validation}}: We validate \method{} on multimodal depression benchmarks, demonstrating leading performance on multiple classification and symptom-severity tasks. Ablation and further studies characterize module contributions, hyperparameter sensitivity, and the cost of recursive improvement. Further analyses link retained pipelines to clinically meaningful depression cues, including facial dynamics associated with psychomotor slowing, supporting \method{}'s evolution toward pipeline designs aligned with clinical understanding of depression.
\vspace{-1em}

% \vspace{-1mm}
\section{Related Work}
\label{sec:related}
\vspace{-3mm}

\textbf{Multimodal Depression Detection.} Multimodal depression detection combines behavioral signals with clinical or personal context~\citep{cohn2018multimodal}. General methods explore cross-modal interactions~\citep{tsai2019multimodal,nagrani2021attention}, quality-aware fusion~\citep{zhang2023qmf,cao2024pdf}, and expert selection~\citep{yun2024flexmoe}; depression-specific designs incorporate temporal fusion~\citep{liu2025msfats}, personality~\citep{shi2025hope}, and symptom semantics~\citep{nerella2026symptom}. Benchmarks provide interviews~\citep{gratch2014distress,ringeval2019avec}, electroencephalography and speech~\citep{cai2022modma}, or personality information~\citep{fu2025mpdd}. Comparable experiments also require consistent subject--target correspondence as recipes change. \method{} preserves this correspondence through Grounded State Construction (GSC), grounding subsequent revisions in comparable feedback across different input representations and fusion strategies.

\textbf{Automated Modeling and Clinical Agents.} Research in this area spans configuration search, agent-driven pipeline development, and case-level clinical assessment. Auto-WEKA jointly selects algorithms and hyperparameters~\citep{thornton2013autoweka}, while MFAS searches multimodal fusion structures~\citep{perezrua2019mfas}. For pipeline generation, DS-Agent reuses historical cases~\citep{guo2024dsagent}, and AutoML-Agent coordinates planning and verification~\citep{trirat2025automlagent}. In clinical applications, ClinPreAI develops postpartum depression predictors through planning, coding, and debugging~\citep{palacios2026clinpreai}, while DepressionAgent revises case-level assessments through multimodal deliberation~\citep{zhu2026depressionagent}. Because representation, fusion, and predictor choices interact, fixed companion components can obscure gains from coordinated changes. \method{} addresses these dependencies through Coupled Pipeline Exploration (CPE), evaluating complete combinations of recipes and predictor configurations under a shared protocol to guide revisions.

\textbf{Recursive Self-Improvement and Experimental Feedback.} Recursive self-improvement (RSI) links experience to inherited changes in the system's improvement process~\citep{duan2026rsi}. Experience retention includes verbal reflections~\citep{shinn2023reflexion}, tertiary memory for depression diagnosis dialogues~\citep{lan2024amc}, and reusable experimental skills~\citep{liu2026labagentcustomizeresearchhubs}. ADAS searches agent designs~\citep{hu2025adas}, the Darwin G\"odel Machine modifies and evaluates its own code~\citep{zhang2025darwingodel}, and the AI Scientist automates research through experimental feedback~\citep{lu2026aiscientist}. In small depression cohorts, sampling fluctuations can inflate scores and steer later revisions toward unstable candidates~\citep{cawley2010overfitting,varoquaux2018crossvalidation}. \method{} therefore uses Evidence-Guided Evolution (EGE) to translate candidate rankings and prediction diagnostics into pipeline revisions and scale inheritance thresholds with subject-level uncertainty.

\vspace{-3mm}
\section{The Proposed \method{}}
\label{sec:method}
\label{sec:method:problem}
\vspace{-3mm}

\textbf{Problem formulation.} We formulate multimodal depression modeling as an experience-based recursive self-improvement (RSI) process~\citep{duan2026rsi}. Each round evaluates pipeline modifications and uses their outcomes to guide the next round. The successor state retains the candidate space, selected pipeline, and accumulated experience. Corpus $\mathcal{D}$ contains $N$ subjects and modality set $\mathcal{M}$. Subject $i$ has observations $X_i^{(m)}$ from available modalities $m\in\mathcal{M}_i\subseteq\mathcal{M}$. The target $y_i$ is a diagnostic category or symptom-severity score. A pipeline $\theta=(r,c)$ combines representation--fusion recipe $r$ with predictor configuration $c$. Before round $t$, system state is $\mathcal{S}_t=(\mathcal{A}_{\mathcal{D}},\mathcal{Q}_t,I_t,\mathcal{L}_t,\mathcal{K}_t)$. It contains checked adapter $\mathcal{A}_{\mathcal{D}}$, candidate space $\mathcal{Q}_t$, incumbent pipeline $I_t$, revision history $\mathcal{L}_t$, and research memory $\mathcal{K}_t$. History records prior outcomes, while memory stores revisable lessons linked to evidence. Experimental experience $\mathcal{E}_t$ guides the proposed modification, its verification, and the construction of the next state:
\begin{equation}
\begin{aligned}
\Delta_t &= \operatorname{Improver}(\mathcal{S}_t,\mathcal{E}_t),\\
(v_t,\mathcal{E}_{t+1}) &= \operatorname{Verifier}(\mathcal{S}_t,\Delta_t;\Pi,\mu),\\
\mathcal{S}_{t+1} &= \operatorname{Inherit}(\mathcal{S}_t,\Delta_t,v_t,\mathcal{E}_{t+1}),
\end{aligned}
\label{eq:selection}
\end{equation}
where $\Delta_t$ is a proposed modification and $v_t$ records validity and replacement decisions. Under the fixed development protocol $\Pi$, the verifier evaluates modifications using task metric $\mu$ and updates experience $\mathcal{E}_{t+1}$. Higher scores are preferred, so root mean squared error (RMSE) is negated for selection. The objective is a high-quality retained pipeline $I_T$ at termination round $T$ within $B$ candidate evaluations, including initialization and failed evaluation attempts. Large language model (LLM) weights, role definitions, and the evaluation protocol remain fixed within a run.

\subsection{Framework Overview}
\label{sec:method:overview}
\begin{figure}[t]
\centering
\includegraphics[width=0.9\linewidth]{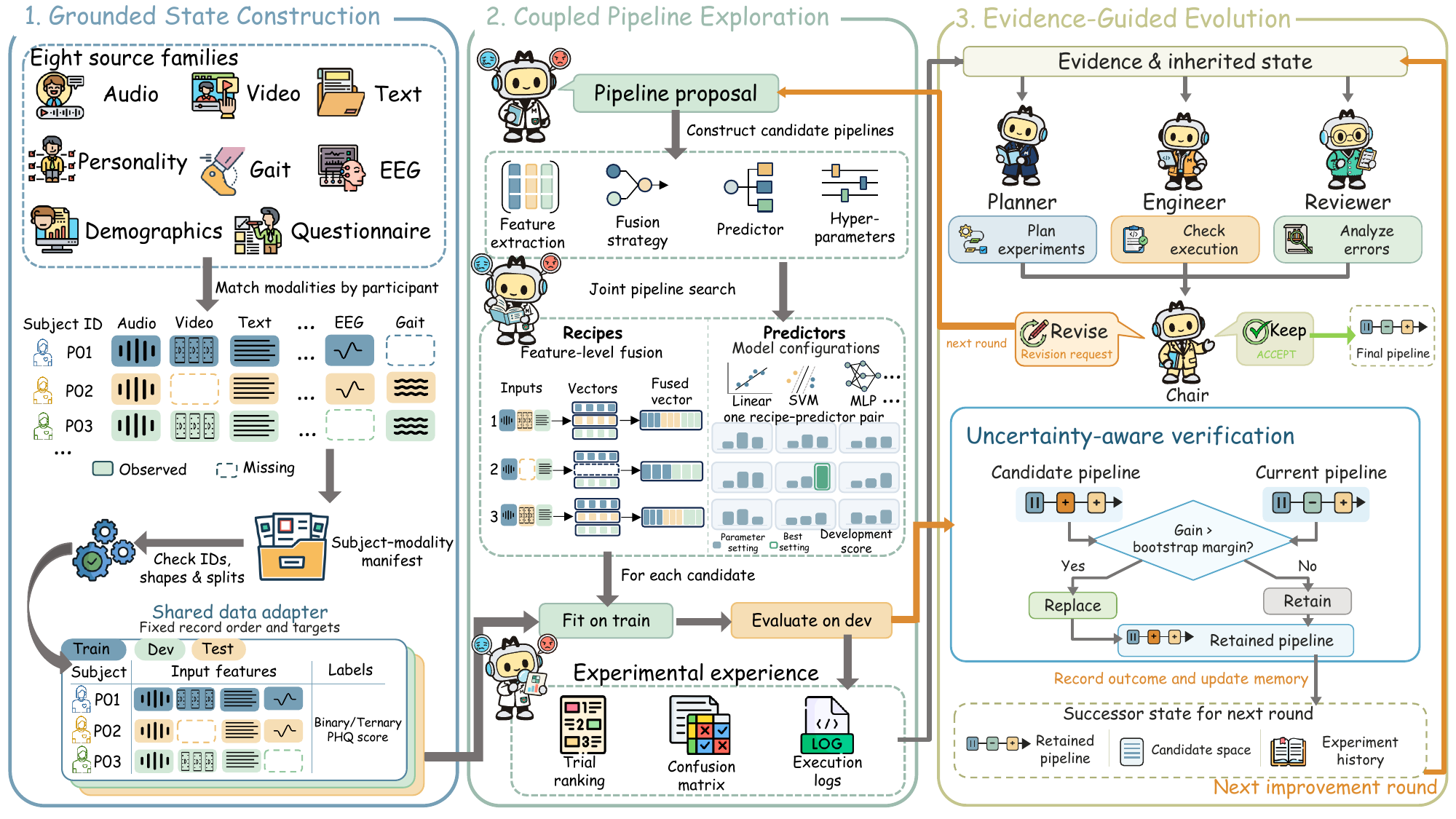}
\caption{The framework of the proposed \method{}. GSC aligns experimental records, and CPE jointly explores pipeline components. EGE uses feedback to guide revisions and verifies gains under uncertainty. The retained pipeline and updated experience guide the next RSI round.}
\label{fig:overview}
\vspace{-1em}
\end{figure}
Multimodal depression modeling raises three issues for RSI. \textbf{\ding{182}~Experience reliability.} Differences in record granularity can distort subject-level comparisons. \textbf{\ding{183}~Coupled modification targets.} Representation, fusion, and predictor choices interact. \textbf{\ding{184}~Retention uncertainty.} Small validation cohorts make apparent gains difficult to verify. \method{} addresses these issues through three modules (Figure~\ref{fig:overview}). \emph{GSC} preserves subject--target correspondence across recipes through a checked adapter. \emph{CPE} evaluates complete pipeline modifications and records comparable outcomes. \emph{EGE} uses these outcomes to propose revisions and verify pipeline replacements under uncertainty. The retained pipeline becomes the next round's reference, while admitted components and experience guide further exploration. This retains evidence from every evaluated pipeline decision.

\subsection{Grounded State Construction (GSC)}
\label{sec:method:adapter}

Depression assessment typically uses subject-level targets, while multimodal inputs contain records at different granularities. Different representation and fusion choices can change which records enter an evaluation or how they are aligned. If this changes their correspondence with the targets, score differences become difficult to attribute to pipeline modifications. The RSI loop uses these outcomes to guide later modifications, so inconsistent comparisons can mislead subsequent revisions. We use GSC to keep each subject's records aligned with the corresponding target while allowing feature representations to change. This provides comparable experimental evidence across RSI rounds.

We implement this alignment through a shared checked data adapter $\mathcal{A}_{\mathcal{D}}$ constructed from the inspected task dataset. Its input specification separates predictive observations from measurements used to define the target. Recipes convert variable-length sequences into fixed-width features. The first valid output establishes the fixed reference ordering and target assignments for each split. Subsequent recipes preserve this reference through the adapter's \texttt{load(recipe)} interface:
\begin{equation}
\begin{aligned}
\mathcal{A}_{\mathcal{D}}(r)
&=\bigl\{(\mathbf{X}_{r}^{(s)},\mathbf{y}^{(s)},\mathbf{u}^{(s)},\mathbf{v}^{(s)})\bigr\}_{s\in\mathcal{V}},\\
\operatorname{Id}\!\left(\mathcal{A}_{\mathcal{D}}(r)\right)
&=\operatorname{Id}\!\left(\mathcal{A}_{\mathcal{D}}(r')\right),
\end{aligned}
\label{eq:adapter-interface}
\end{equation}
where $\mathcal{V}$ contains development data splits indexed by $s$, and $r$ and $r'$ are any two admitted recipes. Matrix $\mathbf{X}_{r}^{(s)}\in\mathbb{R}^{n_s\times p_r}$ contains $n_s$ records with $p_r$ features determined by recipe $r$. The corresponding vectors $\mathbf{y}^{(s)}$, $\mathbf{u}^{(s)}$, and $\mathbf{v}^{(s)}$ contain targets, subject identifiers, and row identifiers, respectively, each with length $n_s$ and aligned with the matrix rows. The operator $\operatorname{Id}$ extracts these ordered vectors for each split. The subject identifiers also determine validation groups. Execution checks verify shapes, finite values, and alignment, while recipe probes check expected feature-width changes to detect ignored configurations. Failed checks trigger repair before the adapter is retained or reused. These checks establish a consistent evaluation basis for subsequent pipeline exploration.

\subsection{Coupled Pipeline Exploration (CPE)}
\label{sec:method:proposal}
We have a consistent basis for comparing pipelines through GSC. However, component effects in multimodal depression detection depend on their specific combination. Changes in representation and fusion alter the input features and affect the predictor's feature selection and regularization. A component with little benefit in one combination may become useful with new components in later RSI rounds. To address this, we introduce CPE to treat the complete pipeline as the RSI improvement target. CPE expands the candidate space by combining new and retained components. It evaluates these combinations under a shared protocol to guide subsequent revisions.

CPE implements exploration by pairing recipes with predictor configurations. A proposer initializes recipe set $\mathcal{R}_0$ and predictor-configuration set $\mathcal{C}_0$ from the inspected data and task. Recipes specify sources, feature operations, aggregation, and fusion. Predictor configurations determine model families, feature selection, and hyperparameters. Retrieved references can inform proposals, and generated classification estimators pass execution checks before admission. The candidate space is $\mathcal{Q}_0=\mathcal{R}_0\times\mathcal{C}_0$. For a subject represented by one aligned record, pipeline $\theta=(r,c)\in\mathcal{Q}_0$ predicts
\begin{equation}
\widehat y_i=h_c\!\left(F_r(\{\mathbf{z}_{i,r}^{(m)}\}_{m\in\mathcal{M}_r})\right),
\label{eq:fusion-levels}
\end{equation}
where $\mathbf{z}_{i,r}^{(m)}$ is subject $i$'s processed vector for modality $m$, $\mathcal{M}_r$ contains the modalities selected by recipe $r$, and $F_r$ is its fusion function. Predictor $h_c$ maps fused features to the predicted category or symptom score $\widehat y_i$. For multi-record tasks, predictions are produced per record, with records grouped by subject during validation. In later RSI rounds, revision requests expand the recipe and predictor sets. CPE pairs new recipes with retained predictors and new predictors with retained recipes, selecting untried pairs within the remaining candidate budget:
\begin{equation}
\begin{aligned}
\widetilde{\mathcal{Q}}_{t+1}
&=(\mathcal{R}_t\cup\Delta\mathcal{R}_t)\times(\mathcal{C}_t\cup\Delta\mathcal{C}_t),\\
\boldsymbol{\Theta}_t&=(\theta_{t,j})_{j=1}^{m_t},
\quad \theta_{t,j}\in\widetilde{\mathcal{Q}}_{t+1}\setminus\mathcal{H}_t,
\quad m_t\leq B-b_t,
\end{aligned}
\label{eq:coupled-expansion}
\end{equation}
where $\widetilde{\mathcal{Q}}_{t+1}$ is the expanded candidate space before evaluation under the fixed protocol begins. Modification $\Delta_t=(\Delta\mathcal{R}_t,\Delta\mathcal{C}_t)$ adds to the current recipe set $\mathcal{R}_t$ and predictor-configuration set $\mathcal{C}_t$, with empty additions for unchanged sets. Set $\mathcal{H}_t$ contains pairs attempted before round $t$, and $b_t$ counts prior candidate evaluations, including initialization and failed attempts. Sequence $\boldsymbol{\Theta}_t$ contains $m_t$ distinct candidates $\theta_{t,j}$ indexed by evaluation order $j$. To make these trials comparable, all candidates share fixed validation subjects and protocol $\Pi=\{(T_k,V_k)\}_{k=1}^{K}$. Here $K$ is the number of evaluation splits, with $k$ indexing the fitting and validation row sets $T_k$ and $V_k$. A designated development split gives $K=1$. Otherwise, fixed subject-disjoint folds within training data keep each subject's records together. The verifier computes development score $\widehat{S}_{\mu}(\theta;\mathcal{D})$ as
\begin{equation}
\widehat{S}_{\mu}(\theta;\mathcal{D})
=\mu\!\left(\bigoplus_{k=1}^{K}\mathbf{y}_{V_k},
\bigoplus_{k=1}^{K}\widehat{\mathbf{y}}_{\theta,V_k}\right),
\label{eq:joint-evaluation}
\end{equation}
where $\mathbf{y}_{V_k}$ contains the targets for validation records $V_k$, and $\widehat{\mathbf{y}}_{\theta,V_k}$ contains their predictions from fitting $\theta$ on $T_k$. The operator $\bigoplus$ concatenates matching targets and predictions across splits before applying metric $\mu$. Reference candidates are freshly fitted under the same protocol. Each trial records its configuration, predictions, score, runtime, and execution errors. Initialization supplies incumbent $I_0$ and experience $\mathcal{E}_0$. In later RSI rounds, search rankings, coverage, and prediction diagnostics form experience $\mathcal{E}_{t+1}$ for EGE to guide retention and further exploration.

\subsection{Evidence-Guided Evolution (EGE)}
\label{sec:method:retro}
We establish a consistent basis for pipeline evaluation through GSC. CPE jointly explores interacting components and provides feedback on complete pipelines. However, these results must guide modifications and inheritance decisions. In multimodal depression detection, similar overall scores can hide different errors across subjects. Limited labeled data make apparent gains uncertain. During RSI, each accepted pipeline becomes the reference for later comparisons, while rejected candidates can provide useful experience. To address these challenges, we introduce EGE to guide modifications using experimental feedback and account for evaluation uncertainty in pipeline replacement. The retained pipeline and updated experience guide the next round~\citep{duan2026rsi}.

Building on CPE's experimental feedback, EGE uses agent review to guide proposals and measured gains to determine pipeline replacement. At round $t$, experimental experience $\mathcal{E}_t$, revision history $\mathcal{L}_t$, and active research memory $\mathcal{K}_t$ guide modifications to incumbent $I_t$. The Chair combines role reports into a decision that guides the proposer. For additive revisions, this process is
\begin{equation}
\begin{aligned}
z_t^a&=\operatorname{Agent}_a(\mathcal{E}_t,\mathcal{I}_t^a,\mathcal{L}_t,\mathcal{K}_t),
\quad a\in\{\mathrm{P},\mathrm{E},\mathrm{R}\},\\
d_t&=\operatorname{Chair}(\{z_t^a\}_a,\mathcal{L}_t,\mathcal{K}_t),\\
\Delta_t&=\operatorname{Proposer}(\chi_{\mathcal{D}},I_t,d_t,\mathcal{K}_t),
\end{aligned}
\label{eq:review-decision}
\end{equation}
where $a$ indexes review roles, $\mathcal{I}_t^a$ is role $a$'s evidence packet, and $z_t^a$ is its report. The Planner ($\mathrm{P}$), Engineer ($\mathrm{E}$), and Reviewer ($\mathrm{R}$) examine search coverage, execution, and subject-level errors, respectively. Their bounded read-only checks use packets programmatically excluding test fields. Context $\chi_{\mathcal{D}}$ supplies the inspected data profile and run constraints. Decision $d_t$ ends review with $\mathrm{ACCEPT}$ or issues a $\mathrm{REVISE}$ request. The proposer generates $\Delta_t$ only for $\mathrm{REVISE}$, using the request and memory to add recipes or predictors. Recipe operations must be supported by the fixed adapter. CPE expands the space through~\eqref{eq:coupled-expansion} and evaluates candidates under the shared protocol. Verified source exclusions invalidate dependent candidates and trigger reselection among valid candidates before further evaluation.\phantomsection\label{sec:method:selection} EGE compares each valid candidate's gain with an acceptance margin based on the incumbent's score variation. $v_{t,j}\in\{0,1\}$ records whether candidate $j$ completes a valid evaluation. Gain $g_{t,j}$ is computed only when $v_{t,j}=1$. Verification applies
\begin{equation}
\begin{aligned}
g_{t,j}&=\widehat{S}_{\mu}(\theta_{t,j};\mathcal{D})-\widehat{S}_{\mu}(I_{t,j-1};\mathcal{D}),\\
\tau_{t,j-1}&=\max\!\left\{\epsilon,\lambda\widehat{\sigma}_{\mathrm{boot}}(I_{t,j-1})\right\},\\
I_{t,j}&=\begin{cases}
\theta_{t,j}, & v_{t,j}=1\ \text{and}\ g_{t,j}>\tau_{t,j-1},\\
I_{t,j-1}, & \text{otherwise}.
\end{cases}
\end{aligned}
\label{eq:method-retention}
\end{equation}
Here $I_{t,j}$ is the incumbent after candidate $j$. Additive revisions start with $I_{t,0}=I_t$. If a source exclusion invalidates the incumbent, the system reselects a valid $I_{t,0}$. The acceptance margin $\tau_{t,j-1}$ uses uncertainty multiplier $\lambda\geq0$, numerical tolerance $\epsilon=10^{-12}$, and bootstrap standard deviation $\widehat{\sigma}_{\mathrm{boot}}$. Resampling keeps each subject's records and saved predictions together, and recomputes it after each replacement. Later candidate comparisons use the updated bootstrap estimate. Under additive revisions with stored scores, Proposition~\ref{prop:mrtheory-margin} bounds the gap between the retained score and the best evaluated score by the largest margin. Each revision request and its outcome are stored in round record $\ell_t$, even when the incumbent remains unchanged. When research-memory learning is enabled, EGE uses this record and updated experience $\mathcal{E}_{t+1}$ to distill conditional lessons for subsequent RSI revisions and reconcile them with research memory~\citep{zhu2026rsiagent}:
\begin{equation}
\begin{aligned}
\mathcal{U}_t&=\operatorname{Distill}(\mathcal{E}_{t+1},\ell_t,\mathcal{K}_t),\\
\mathcal{K}_{t+1}&=\operatorname{Reconcile}(\mathcal{K}_t,\mathcal{U}_t),
\end{aligned}
\label{eq:memory-revision}
\end{equation}
where $\mathcal{U}_t$ contains lesson drafts from round $t$, each specifying conditions, interpretation, a next experiment, and supporting evidence. Reconciliation adds, revises, or retires entries as evidence supports or challenges earlier lessons. Updates must pass structural and evidence-reference checks before changing memory. If these checks fail, the previous bank is preserved. Memory starts at $\mathcal{K}_0=\varnothing$, retains prior versions, and holds at most 16 active entries. The two bounded LLM calls reuse recorded evidence without additional candidate evaluations. This preserves a traceable record of memory changes across rounds for later review. \phantomsection\label{sec:method:provenance} EGE then carries the updated memory into the successor state alongside the retained pipeline, expanded candidate space, and revision history. For additive revisions, the resulting state and next proposal, if the loop continues, are
\begin{equation}
\begin{aligned}
\mathcal{Q}_{t+1}&=\widetilde{\mathcal{Q}}_{t+1},
\qquad \mathcal{L}_{t+1}=\mathcal{L}_t\mathbin{\Vert}\ell_t,\\
\mathcal{S}_{t+1}&=(\mathcal{A}_{\mathcal{D}},\mathcal{Q}_{t+1},I_{t,m_t},\mathcal{L}_{t+1},\mathcal{K}_{t+1}),\\
\Delta_{t+1}&=\operatorname{Improver}(\mathcal{S}_{t+1},\mathcal{E}_{t+1}),
\end{aligned}
\label{eq:state-inheritance}
\end{equation}
where $\Vert$ appends $\ell_t$ to the history sequence initialized as $\mathcal{L}_0=()$, and $I_{t,m_t}$ becomes incumbent $I_{t+1}$ for subsequent verification. Admitted components remain available for new combinations even when the incumbent is unchanged, while revision history and updated memory inform the next proposal. Later proposals can use recorded experimental evidence from candidates that did not replace the incumbent. The loop returns $I_T$ when the Chair ends review or the evaluation or review-round budget is exhausted, with all accepted evidence retained in the final run record for auditing.

\providecolor{mrHeader}{RGB}{232,238,242}
\providecolor{mrOurs}{RGB}{220,239,249}
\providecolor{mrBand}{RGB}{245,247,248}
\section{Experiments}
\label{sec:experiments}
\begingroup
\setlength{\textfloatsep}{8pt plus 2pt minus 2pt}

\begin{table*}[t]
\caption{Performance comparison on depression benchmarks and supplementary tasks, with five-seed means for \method{} using Claude Opus 4.8 and GPT-6 Astra. Best and second-best results are \textbf{bold} and \underline{underlined}. See Appendices~\ref{app:seed_study} and~\ref{app:base_model_astra} for per-seed results.}
\vspace{-1em}
\label{tab:main-results}
\centering
\begingroup
\definecolor{mrMainHeader}{RGB}{229,239,217}
\definecolor{mrMainOurs}{RGB}{255,245,215}
\definecolor{mrMainBand}{RGB}{243,243,243}
\setlength{\aboverulesep}{0pt}
\setlength{\belowrulesep}{0pt}
\fontsize{8.5}{10}\selectfont
\setlength{\tabcolsep}{2.5pt}
\renewcommand{\arraystretch}{0.979125}
\begin{tabular}{>{\raggedright\arraybackslash}p{.22\linewidth}*{6}{>{\centering\arraybackslash}p{\dimexpr(.68\linewidth-16\tabcolsep)/6\relax}}>{\centering\arraybackslash}p{.10\linewidth}}
\toprule
\rowcolor{mrMainHeader}
Dataset & \multicolumn{2}{c}{E-DAIC} & \multicolumn{2}{c}{MPDD-2025} & \shortstack{Mental\\Health} & Kaggle & \\
\cmidrule(lr){2-3}\cmidrule(lr){4-5}\cmidrule(lr){6-6}\cmidrule(lr){7-7}
\rowcolor{mrMainHeader}
Method & \shortstack{Binary\\F1 $\uparrow$} & \shortstack{PHQ-8\\CCC $\uparrow$} & \shortstack{Elder\\Track $\uparrow$} & \shortstack{Young\\Track $\uparrow$} & \shortstack{Status\\F1 $\uparrow$} & \shortstack{A--V\\F1 $\uparrow$} & \shortstack{Avg. rank\\$\downarrow$} \\
\midrule
\rowcolor{mrMainBand}
MulT & 0.5221 & -0.0047 & 0.4549 & 0.3952 & --- & --- & 4.50 {\scriptsize (4)} \\
Flex-MoE & 0.5993 & 0.0135 & 0.5440 & \textbf{0.4483} & 0.9548 & \underline{0.6201} & 2.67 {\scriptsize (6)} \\
\rowcolor{mrMainBand}
MCMoE & 0.4105 & 0.0000 & 0.4511 & 0.3419 & \emph{n/a} & \emph{n/a} & 5.75 {\scriptsize (4)} \\
\midrule
Gemini 2.5 Flash & 0.3978 & -0.0917 & 0.4322 & 0.3995 & 0.0843 & 0.2966 & 5.33 {\scriptsize (6)} \\
\rowcolor{mrMainBand}
MDAgents & 0.4697 & 0.0520 & 0.2096 & \underline{0.4069} & 0.0607 & 0.3845 & 4.33 {\scriptsize (6)} \\
\midrule
\rowcolor{mrMainOurs}
\textbf{MERID (Opus 4.8)} & \underline{0.6768} & \textbf{0.3268} & \underline{0.6137} & 0.3724 & \textbf{0.9655} & \textbf{0.7836} & \textbf{2.00} {\scriptsize (6)} \\
\rowcolor{mrMainOurs}
\textbf{MERID (GPT-6 Astra)} & \textbf{0.7057} & \underline{0.3201} & \textbf{0.6148} & 0.3714 & \underline{0.9565} & 0.6070 & \underline{2.50} {\scriptsize (6)} \\
\midrule[1.2pt]
\rowcolor{mrMainHeader}
Dataset / cohort & \multicolumn{3}{c}{MPDD-2026 / Elder} & \multicolumn{3}{c}{MPDD-2026 / Young} & \\
\cmidrule(lr){2-4}\cmidrule(lr){5-7}
\rowcolor{mrMainHeader}
Method & \shortstack{Binary\\F1 $\uparrow$} & \shortstack{Ternary\\F1 $\uparrow$} & \shortstack{PHQ-9\\CCC $\uparrow$} & \shortstack{Binary\\F1 $\uparrow$} & \shortstack{Ternary\\F1 $\uparrow$} & \shortstack{PHQ-9\\CCC $\uparrow$} & \shortstack{Avg. rank\\$\downarrow$} \\
\midrule
\rowcolor{mrMainBand}
MulT & \textbf{0.5165} & \underline{0.2735} & 0.0000 & 0.3333 & 0.1905 & \underline{0.0760} & 3.83 \\
Flex-MoE & 0.3947 & \underline{0.2735} & 0.0000 & \underline{0.4990} & 0.2083 & 0.0000 & 4.17 \\
\rowcolor{mrMainBand}
MCMoE & 0.4103 & \underline{0.2735} & 0.0000 & 0.3333 & \underline{0.3404} & 0.0003 & 4.00 \\
\midrule
Gemini 2.5 Flash & 0.3002 & 0.1262 & 0.0154 & 0.3333 & 0.1978 & -0.0124 & 5.17 \\
\rowcolor{mrMainBand}
MDAgents & 0.2333 & 0.0988 & -0.0054 & 0.2667 & 0.1429 & -0.1121 & 7.00 \\
\midrule
\rowcolor{mrMainOurs}
\textbf{MERID (Opus 4.8)} & \underline{0.4559} & \textbf{0.2978} & \underline{0.1321} & \textbf{0.5455} & \textbf{0.4327} & \textbf{0.6050} & \textbf{1.58} \\
\rowcolor{mrMainOurs}
\textbf{MERID (GPT-6 Astra)} & 0.4377 & 0.2523 & \textbf{0.1999} & \textbf{0.5455} & \textbf{0.4327} & \textbf{0.6050} & \underline{2.25} \\
\bottomrule
\end{tabular}
\endgroup
% \vspace{1em}d
\end{table*}

\subsection{Experimental Setup}
\label{sec:experiments:setup}

\textbf{Datasets and evaluation.} We evaluate \method{} on E-DAIC~\citep{ringeval2019avec}, MPDD-2025/2026~\citep{fu2025mpdd,fu2026mpddavg}, Mental Health Multimodal tabular data with clinical scales, and Kaggle TESS/RAVDESS acted-emotion tasks~\citep{pichorafuller2020tess,livingstone2018ravdess}. Selection uses development splits or training cross-validation. Metrics are MPDD-2025 track scores, classification Macro-F1, and PHQ concordance correlation coefficient (CCC); protocols appear in Appendix~\ref{app:baseline_setup}. All reported scores follow their corresponding task protocols.

\textbf{Baselines.} We compare learning-based MulT, Flex-MoE, and MCMoE~\citep{tsai2019multimodal,yun2024flexmoe,xu2026mcmoe}, and LLM-based Gemini 2.5 Flash~\citep{google2025gemini25} and MDAgents~\citep{kim2024mdagents} with GPT-5.5~\citep{openai2026gpt55}. Input adaptations and task-specific baselines appear in Appendix~\ref{app:baseline_setup}. All baselines use dataset-specific inputs and protocols described there.

\textbf{Implementation details.} Candidates share task-specific adapters, validation splits, selection metrics, and fixed evaluation protocols within each reported experimental task. Reviews use development rankings, prediction diagnostics, and execution logs. After search, held-out-test predictions aggregate up to five selected candidates per task. Research-memory learning is evaluated in the dedicated studies in Appendix~\ref{app:revision_case}. Appendix~\ref{app:run_details} specifies eligible inputs, budgets, and review settings.

\subsection{Performance Comparison}
\label{sec:experiments:comparison}
\begin{figure}[t]
\centering
\captionsetup{position=bottom,skip=3pt}
\includegraphics[width=0.9\linewidth]{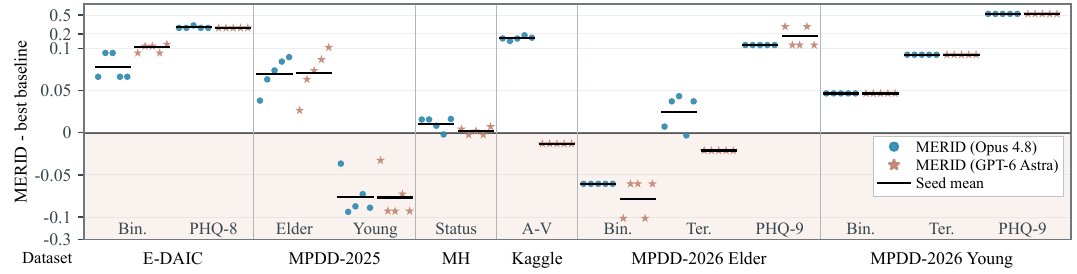}
\caption{Five-seed score differences from the best baseline in Table~\ref{tab:main-results} (markers: seeds; bars: means). Positive values favor \method{}. The axis is linear within $\pm0.1$ and logarithmic beyond.}
\label{fig:seed-robustness}
% \vspace{-1em}
\end{figure}
We compare \method{} using Opus 4.8 and GPT-6 Astra in Table~\ref{tab:main-results}.
\textit{\textbf{Obs.\ding{182} \method{} achieves strong depression detection performance.}} We first observe this advantage on E-DAIC, where binary Macro-F1 averages 0.6768 and 0.7057, respectively, exceeding Flex-MoE's 0.5993. Both configurations also lead binary and ternary classification on MPDD-2026 Young, extending these gains from depression status to severity categories.
\textit{\textbf{Obs.\ding{183} \method{}'s gains extend to symptom-severity estimation.}} Beyond classification, we obtain E-DAIC PHQ-8 CCC means of 0.3268 and 0.3201, versus 0.0520 for the best baseline, MDAgents. Similarly, MPDD-2026 Young PHQ-9 CCC reaches 0.6050 versus MulT's 0.0760. Together, these gains extend \method{}'s prediction to graded symptom burden beyond binary depression status.
\textit{\textbf{Obs.\ding{184} \method{} shows a relative advantage in Elder severity assessment.}} At the cohort level, \method{} (Opus 4.8) leads ternary classification on the MPDD-2026 Elder cohort, while \method{} (GPT-6 Astra) leads PHQ-9 CCC at 0.1999. Across the two configurations, the best binary Macro-F1 is 0.4559, compared with MulT's 0.5165. The ternary and PHQ-9 results highlight \method{}'s strength in grading depression severity and estimating continuous symptom burden.

\subsection{Ablation Studies}
\label{sec:experiments:ablation}

\begin{figure}[t]
\centering
\includegraphics[width=0.9\linewidth]{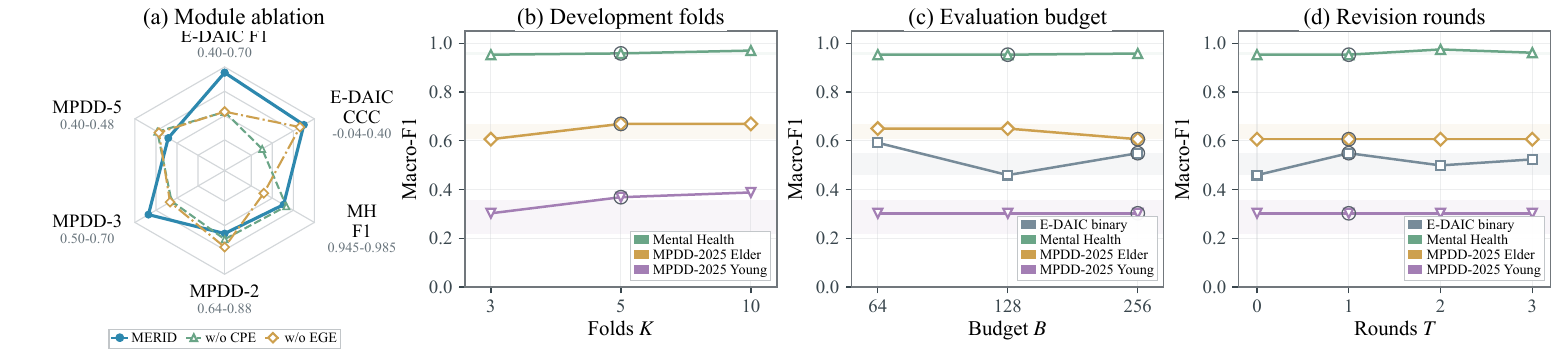}
\vspace{-0.5em}
\caption{Ablation visualization and sensitivity analyses with Opus 4.8; see Appendix~\ref{app:analysis-protocols}.}
\label{fig:ablation-analysis}
\label{fig:combined-analysis}
\par\vspace{0.5em}
\begin{minipage}{\linewidth}
\captionsetup{hypcap=false,type=table,position=bottom,skip=3pt}
\centering
\begingroup
\definecolor{mrAblHeader}{RGB}{229,239,217}
\definecolor{mrAblOurs}{RGB}{255,245,215}
\definecolor{mrAblBand}{RGB}{243,243,243}
\setlength{\aboverulesep}{0pt}
\setlength{\belowrulesep}{0pt}
\setlength{\arrayrulewidth}{.6pt}
\fontsize{7.5}{9}\selectfont
\setlength{\tabcolsep}{1.7pt}
\renewcommand{\arraystretch}{1.08}
\begin{tabular}{*{3}{>{\centering\arraybackslash}p{.04\linewidth}}*{6}{>{\centering\arraybackslash}p{\dimexpr(.77\linewidth-28\tabcolsep-.6pt)/9\relax}}|*{2}{>{\centering\arraybackslash}p{.055\linewidth}}*{3}{>{\centering\arraybackslash}p{\dimexpr(.77\linewidth-28\tabcolsep-.6pt)/9\relax}}}
\toprule
\rowcolor{mrAblHeader}
\multicolumn{9}{c|}{Module ablation} & \multicolumn{5}{c}{GSC construction} \\
\cmidrule(lr){1-9}\cmidrule(lr){10-14}
\rowcolor{mrAblHeader}
\multicolumn{3}{c}{Modules} & \multicolumn{2}{c}{E-DAIC} & MH & \multicolumn{3}{c|}{MPDD-25 Elder} & \multicolumn{2}{c}{GSC} & \multicolumn{2}{c}{E-DAIC} & MH \\
\rowcolor{mrAblHeader}
GSC & CPE & EGE & F1 & CCC & F1 & Binary & Ternary & Quinary & Insp. & Repair & F1 & CCC & F1 \\
\midrule
\rowcolor{mrAblBand}
\ding{51} & \ding{55} & \ding{51} & 0.5151 & 0.0361 & \textbf{0.9672} & \underline{0.7648} & 0.5839 & \textbf{0.4516} & \ding{55} & \ding{51} & \underline{0.6190} & \textbf{0.2771} & \underline{0.9653} \\
\ding{51} & \ding{51} & \ding{55} & \underline{0.5162} & \underline{0.2997} & 0.9530 & \textbf{0.7908} & \underline{0.5887} & \underline{0.4496} & \ding{51} & \ding{55} & 0.5359 & \textbf{0.2771} & 0.9622 \\
\midrule
\rowcolor{mrAblOurs}
\ding{51} & \ding{51} & \ding{51} & \textbf{0.6768} & \textbf{0.3268} & \underline{0.9655} & 0.7462 & \textbf{0.6572} & 0.4377 & \ding{51} & \ding{51} & \textbf{0.6316} & \underline{0.2464} & \textbf{0.9787} \\
\bottomrule
\end{tabular}
\endgroup
\vspace{1mm}
\caption{Ablations with Opus 4.8 module test (left) and GSC development (right); see Appendix~\ref{app:ablation_selection}.}
\label{tab:ablation-studies}
\vspace{-1em}
\end{minipage}
\end{figure}

We evaluate \method{} with Opus 4.8 in Table~\ref{tab:ablation-studies} and Figure~\ref{fig:ablation-analysis}(a). Full \method{} uses an inherited pipeline and seeds 0--4. Both removals start without one, using seeds 0--2. Search costs differ.
\textit{\textbf{Obs.\ding{185} GSC improves depression detection.}} GSC checks alignment and repairs adapter outputs. E-DAIC development Macro-F1 falls from 0.6316 to 0.5359 without repair. Depression labels belong to participants, while speech and facial features can contain multiple records. GSC pairs observations with the target. PHQ-8 CCC rises after either GSC removal.
\textit{\textbf{Obs.\ding{186} CPE improves prediction of depressive symptom burden.}} Joint exploration of representations and predictors achieves E-DAIC PHQ-8 CCC of 0.3268 versus 0.0361 without CPE. PHQ-8 retains differences in symptom burden among participants with the same binary depression label. The gap favors joint search for this graded target. Without CPE, seed SD is larger on seven of nine tasks. Mental Health and average MPDD-2025 Elder binary and quinary scores are higher without CPE.
\textit{\textbf{Obs.\ding{187} EGE benefits PHQ-8 and ternary severity assessment.}} EGE guides revisions using prediction errors and uncertainty. Full \method{} achieves higher PHQ-8 CCC and an average MPDD-2025 Elder ternary score of 0.6572 versus 0.5887 without EGE. These results concern two resolutions of depression assessment: a continuous symptom score and severity categories. Binary and quinary averages favor removal.
\textit{\textbf{Obs.\ding{188} Research memory produces mixed results.}} In a separate experiment, evolving memory yields seven test gains, thirteen ties, and seven losses across 27 combinations of task and seed, adding 7.2 LLM calls per run on average. Memory's contribution varies across tasks and seeds.

\subsection{Further Analysis}
\label{sec:experiments:selection-sensitivity}

\begin{wrapfigure}{r}{0.5\textwidth}
\vspace{-1.5\intextsep}
\centering
\includegraphics[width=\linewidth]{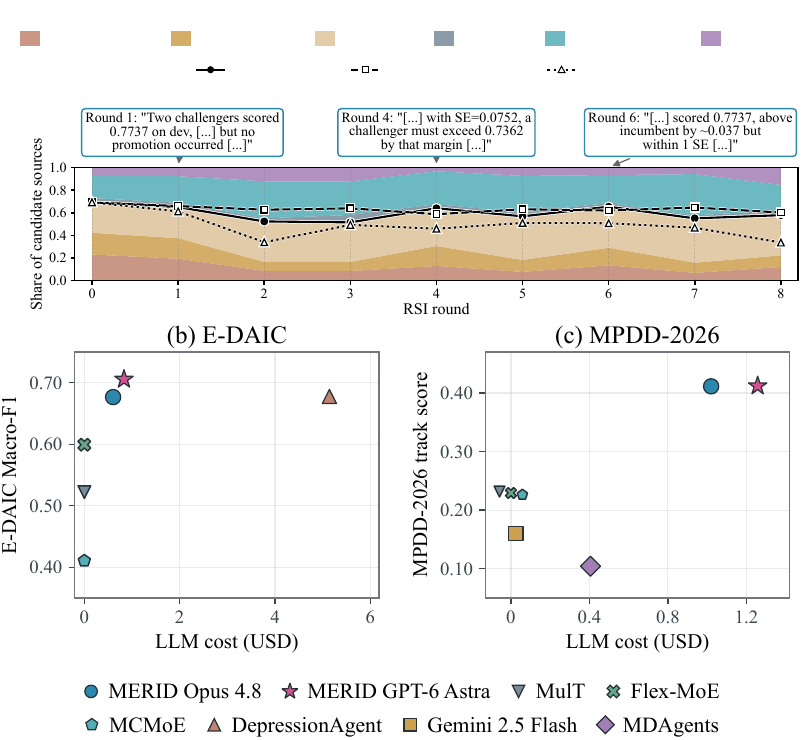}
\caption{RSI source evolution (a) and score--cost comparisons (b,c). \method{} points are five-seed means. Details: Appendices~\ref{app:rsi-eight-round} and~\ref{app:efficiency}.}
\vspace{-2em}
\label{fig:rsi-score-cost}
\label{fig:rsi-round-evolution}
\label{fig:score-cost}
\end{wrapfigure}

\textbf{Sensitivity to Hyperparameter Choices.} We vary validation-fold counts, candidate budgets, and revision limits and report Macro-F1 in Figure~\ref{fig:ablation-analysis}(b--d).
\textit{\textbf{Obs.\ding{189} We find no consistent gain from larger budgets.}} Five validation folds match ten on MPDD-2025 Elder. E-DAIC Macro-F1 is 0.5919 at a budget of 64 candidate evaluations, versus 0.5492 at 256. Its round-limit scan peaks at one revision. MPDD-2025 cohorts have unchanged scores across that scan.

\textbf{Evolution of Multimodal Search Across RSI Rounds.} \textit{\textbf{Obs.\ding{190} We trace changes in search composition across rounds.}} In Figure~\ref{fig:rsi-score-cost}(a), stacked areas show MERID's candidate source-family shares, while lines sum eGeMAPS, OpenFace, and TF--IDF shares for each arm across six E-DAIC runs. Boxes trace one memory entry from binary seed 1, comparing gains with the uncertainty margin and retaining the incumbent despite higher scores. MERID's eGeMAPS, OpenFace, and TF--IDF share dips early and recovers in middle rounds, without a sustained increase through round eight. For depression assessment, RSI compares acoustic, facial, and lexical representations while replacing the incumbent only when a candidate clears the uncertainty margin. Preregistered tests do not establish a consistent shift toward physiological evidence.

\textbf{Comparison of LLM Costs and Performance.} We summarize LLM usage across nine task-level means: median calls per run are 11.6 for full \method{} (the runs of Table~\ref{tab:main-results}), 43 without CPE, and 1 without EGE. Corresponding median costs are \$0.65, \$2.32, and \$0.01. The variant without CPE permits more revision rounds; these costs exclude local computation.
\textit{\textbf{Obs.\ding{191} We compare LLM costs at development and inference.}} Across five seeds, \method{} achieves E-DAIC binary Macro-F1 of 0.6768 at a mean development LLM cost of \$0.60 per run (Figure~\ref{fig:rsi-score-cost}(b)). DepressionAgent's text branch scores 0.6778 at a cohort-inference cost of \$5.14. On MPDD-2026, \method{} has the highest plotted score, 0.4115, with \$1.02 of development LLM usage per seed (Figure~\ref{fig:rsi-score-cost}(c)). Non-LLM methods appear at zero LLM cost; DepressionAgent's cost comes from a separate accounting rerun.

\endgroup

\section{Conclusion}
\label{sec:conclusion}

This paper studies multimodal depression detection and symptom-severity estimation. Autonomously adapting prediction pipelines to heterogeneous multimodal signals and carrying verified improvements into subsequent designs remains challenging. We propose \method{}, which uses recursive self-improvement (RSI) to coordinate pipeline exploration, feedback-guided revision, and verified inheritance. Experiments across depression benchmarks demonstrate strong performance, while case studies link revisions to clinical input constraints and ordered severity targets. RSI thus guides \method{}'s self-evolution toward pipelines tailored to psychiatric assessment.

\subsection*{AI use statement}
We used Claude Opus 4.8 and GPT-6 Astra as the base LLMs for \method{}, and Gemini 2.5 Flash and GPT-5.5 in the baseline experiments. Generative AI tools also assisted with language polishing and grammar correction. The research ideas and methodology were developed by the authors, who take responsibility for the final text, claims, artifacts, and reported results.

\subsection*{Ethics statement}
This study uses existing datasets for methodological research on depression assessment and related prediction tasks. Mental-health data, feature summaries, and sampled frames can contain sensitive or identifying information. Their secondary use and external API processing require appropriate permissions and privacy safeguards. Results on acted-emotion proxy tasks do not establish clinical diagnostic validity. Clinical use would require prospective validation, assessment of demographic disparities, and qualified professional oversight to address risks of missed cases, false alarms, and stigma. Model predictions and LLM-generated rationales should not serve as standalone diagnoses.

\subsubsection*{Reproducibility statement}
We document dataset inputs, splits, metrics, model settings, and search budgets, together with seed-level results and ablations in the results and appendix sections. Run records preserve candidate configurations, development scores, predictions, review decisions, execution failures, and LLM usage. Codes are provided in the abstract section.

\subsection*{Author Contribution}
L.L. designed this study with T.L. and Z.T. L.L. implemented the method. L.L. and Z.L. performed the experiments and analyzed the results. All authors read and approved the final manuscript. T.L. supervised the project.

% Add acknowledgments here when the text is ready.
% \section*{Acknowledgments}

\bibliographystyle{iclr2027_conference}
\bibliography{references}

\clearpage
\appendix

% Use your appendix entry point, with a fallback for the local layout.
\IfFileExists{ablation_appendix.tex}{%
  \input{ablation_appendix.tex}%
}{%
  \clearpage
\section*{Appendix}

This appendix presents case studies, seed variability and supplementary analyses, RSI experiments, efficiency analysis, complete benchmark results, experimental settings, theoretical analysis, and implementation details. The contents below link to each section and subsection.

% Linked contents use the section labels and update after compilation.
\begingroup\hypersetup{hidelinks}
\setlength{\parindent}{0pt}\setlength{\parskip}{1pt}
\newcommand{\mrappendixsection}[1]{\par\addvspace{0.5em}\noindent{\bfseries\hyperref[#1]{\makebox[2.2em][l]{\ref*{#1}}\nameref*{#1}}\hfill\pageref{#1}}\par}
\newcommand{\mrappendixsubsection}[1]{\noindent\hspace*{1.2em}\hyperref[#1]{\makebox[3em][l]{\ref*{#1}}\nameref*{#1}}\nobreak\dotfill\pageref{#1}\par}
\newcommand{\mrappendixsubsubsection}[1]{\noindent\hspace*{2.4em}\hyperref[#1]{\makebox[3.8em][l]{\ref*{#1}}\nameref*{#1}}\nobreak\dotfill\pageref{#1}\par}
\mrappendixsection{app:case_studies}
\mrappendixsection{app:supplementary_analyses}
\mrappendixsubsection{app:seed_study}
\mrappendixsubsection{app:base_model_astra}
\mrappendixsubsection{app:ablation_selection}
\mrappendixsubsection{app:modality_study}
\mrappendixsubsection{app:supplementary_sensitivity}
\mrappendixsubsection{app:selection_deployment}
\mrappendixsubsubsection{app:selection_rules}
\mrappendixsubsubsection{app:deployment_rules}
\mrappendixsubsubsection{app:gate_multiplier}
\mrappendixsubsection{app:search_axes}
\mrappendixsection{app:revision_case}
\mrappendixsubsection{app:rsi_trajectories}
\mrappendixsubsection{app:rsi-eight-round}
\mrappendixsubsubsection{app:rsi_study_design}
\mrappendixsubsubsection{app:rsi_measures}
\mrappendixsubsubsection{app:rsi_test_memory}
\mrappendixsubsubsection{app:rsi_execution_diagnostics}
\mrappendixsection{app:efficiency}
\mrappendixsubsection{app:baseline-cost}
\mrappendixsubsection{app:module_costs}
\mrappendixsubsection{app:budget_effort}
\mrappendixsection{app:full-results}
\mrappendixsection{app:baseline_setup}
\mrappendixsubsection{app:datasets_protocols}
\mrappendixsubsection{app:baseline_implementations}
\mrappendixsubsection{app:shared_controls}
\mrappendixsubsection{app:compute}
\mrappendixsection{sec:mrtheory-appendix}
\mrappendixsubsection{app:rsi_inheritance_process}
\mrappendixsubsection{app:rsi_stored_score_guarantees}
\mrappendixsubsection{app:rsi_population_improvement}
\mrappendixsubsection{app:rsi_theory_implications}
\mrappendixsection{app:method_details}
\mrappendixsubsection{app:method_execution}
\mrappendixsubsection{app:rsi_schedule}
\mrappendixsubsection{app:research_memory}
\mrappendixsubsection{app:rsi_configuration}
\mrappendixsubsection{app:prompts}
\endgroup

% Appendix tables use the main-text palette and type sizes.
\definecolor{mrAppHeader}{RGB}{229,239,217}
\definecolor{mrAppOurs}{RGB}{255,245,215}
\definecolor{mrAppBand}{RGB}{243,243,243}
\captionsetup[table]{font=normalsize,labelfont=normalfont,textfont=normalfont,position=top,skip=3pt}
\newcommand{\mrAppendixTableStyle}{%
  \fontsize{8.5}{10}\selectfont
  \setlength{\tabcolsep}{2.5pt}%
  \renewcommand{\arraystretch}{1.12}%
  \setlength{\aboverulesep}{0pt}%
  \setlength{\belowrulesep}{0pt}%
  \setlength{\arrayrulewidth}{.6pt}%
}
\newcommand{\mrAppendixDenseTableStyle}{%
  \mrAppendixTableStyle
  \fontsize{7.5}{9}\selectfont
  \setlength{\tabcolsep}{2pt}%
  \renewcommand{\arraystretch}{1.08}%
}

\makeatletter
\setlength{\@fptop}{0pt} % Align full-page appendix floats at the top.
\makeatother
\captionsetup[figure]{skip=4pt}
\setlength{\textfloatsep}{12pt plus 2pt minus 2pt}
\setlength{\intextsep}{10pt plus 2pt minus 2pt}
\setlength{\floatsep}{10pt plus 2pt minus 2pt}
\raggedbottom % Keep short appendix pages from stretching paragraph spacing.

\clearpage
\section{Case Studies of Evidence-Guided Evolution}
\label{app:case_studies}

\paragraph{Evidence-guided evolution.}
\phantomsection\label{app:evidence_revision}
Figure~\ref{fig:revision-case} illustrates a revision on MPDD-2025 Elder 1\,s ternary classification. The best candidate scored 0.4569, compared with 0.4178 for the incumbent, leaving the gain below the 0.0470 acceptance margin. The Planner and Engineer checked which features were loaded and whether clinical scales were excluded. Following review, the Chair requested an ordinal model and additional tree-depth settings. Another 64 evaluations raised the best development score to 0.4814, while the recorded test score remained unchanged. The source checks addressed separation of predictive inputs from clinical targets; the ordinal proposal addressed the ordering of severity categories.

\begin{figure}[!htbp]
\centering
\includegraphics[width=\linewidth]{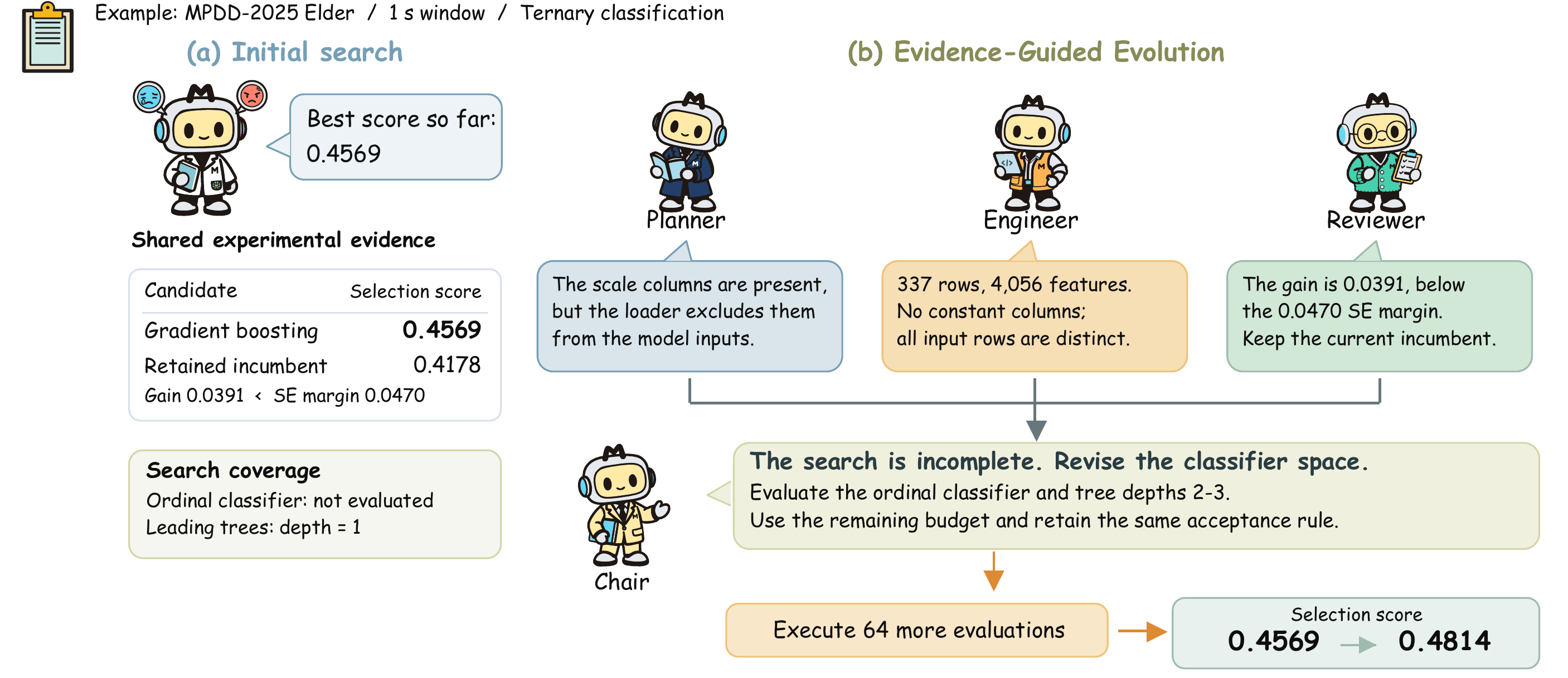}
\caption{A recorded MPDD-2025 evidence-guided evolution step. Dialogue is condensed from the review record. The final arrow tracks the best development candidate, which is distinct from the retained incumbent.}
\label{fig:revision-case}
\end{figure}

\paragraph{Depression-oriented evolution and case-level deliberation.}
\phantomsection\label{app:depressionagent_case}
\method{} and DepressionAgent~\citep{zhu2026depressionagent} use feedback for different revision targets. DepressionAgent uses specialized agents to extract, challenge, and reconsider evidence for the current participant. \method{} uses experimental results to revise a prediction pipeline and carry accepted changes into subsequent development rounds. Table~\ref{tab:depressionagent_case} contrasts these revision targets and the information retained.

\paragraph{A recorded DepressionAgent workflow.}
We examine a text-branch adaptation on the 56 E-DAIC test participants using the deployment \texttt{claude-opus-4-8} and the published prompts A2/A4/A5/A6/A12/A13. The text observer extracts self-reported states, functional changes, protective evidence, and contextual qualifications. A hypothesis agent and a critic examine this evidence before transcript arbitration. Initially negative cases receive an independent review, whose gate can request re-evaluation. Inference is zero-shot, with labels accessed only for scoring. The adaptation omits the original audio--visual branch and cross-modal arbitration.

\begin{table}[!htbp]
\caption{How depression-related evidence enters revision. DepressionAgent describes the evaluated text-branch adaptation. The \method{} column describes its pipeline-development loop, with the concrete example in Figure~\ref{fig:revision-case}.}
\label{tab:depressionagent_case}
\centering
\begingroup
\mrAppendixTableStyle
\begin{tabular}{>{\raggedright\arraybackslash}p{.18\linewidth}*{2}{>{\raggedright\arraybackslash}p{\dimexpr(.82\linewidth-6\tabcolsep)/2\relax}}}
\toprule
\rowcolor{mrAppHeader}
\textbf{Aspect} & \textbf{DepressionAgent (text branch)} & \cellcolor{mrAppOurs}\textbf{\method{}} \\
\midrule
\rowcolor{mrAppBand}
Evidence & Interview statements and supporting or countervailing interpretations & \cellcolor{mrAppOurs}Task constraints, candidate scores, prediction diagnostics, and execution checks \\
Revision target & The current participant's risk judgment & \cellcolor{mrAppOurs}Representation--fusion recipes and predictor configurations \\
\rowcolor{mrAppBand}
Depression-specific action & Re-read the transcript for overlooked risk evidence & \cellcolor{mrAppOurs}Check clinical-scale exclusion and propose an ordinal severity predictor \\
Verification & Reconsider evidence through review and re-evaluation & \cellcolor{mrAppOurs}Evaluate proposed pipelines under a shared protocol and replacement margin \\
\rowcolor{mrAppBand}
Retained consequence & A final case judgment and its reasoning record & \cellcolor{mrAppOurs}A reusable pipeline, admitted components, and experience for the next round \\
\bottomrule
\end{tabular}
\endgroup
\end{table}

\paragraph{From a revised case to an evolving predictor.}
The evaluated DepressionAgent run initially classified 14 participants as positive and 42 as negative. Review changed two negative predictions to positive, giving 16 final positive predictions. Final Macro-F1 was 0.6778 and accuracy was 0.7321, with 268 LLM calls. The summary identifies the revised predictions but does not establish whether either revision corrected an error. In this adaptation, prompts and inference procedures remain fixed across participants; case-level outcomes do not update a shared prediction pipeline.

In \method{}, source checks and the ordinal-model proposal guide revisions to a reusable predictor (Appendix~\ref{app:evidence_revision}). Subsequent proposals use the retained pipeline, admitted components, and accumulated experimental experience. In the recorded PHQ runs, accepted revisions increase retained development CCC (Appendix~\ref{app:rsi_trajectories}). These records connect pipeline revisions to severity-prediction objectives, but do not by themselves establish improved test performance.

\FloatBarrier
\section{Supplementary Results and Ablations}
\label{app:supplementary_analyses}

\FloatBarrier
\subsection{Seed Variability of the Main Results}
\label{app:seed_study}

\begin{center}\begin{minipage}{\linewidth}
\centering
\includegraphics[width=\linewidth]{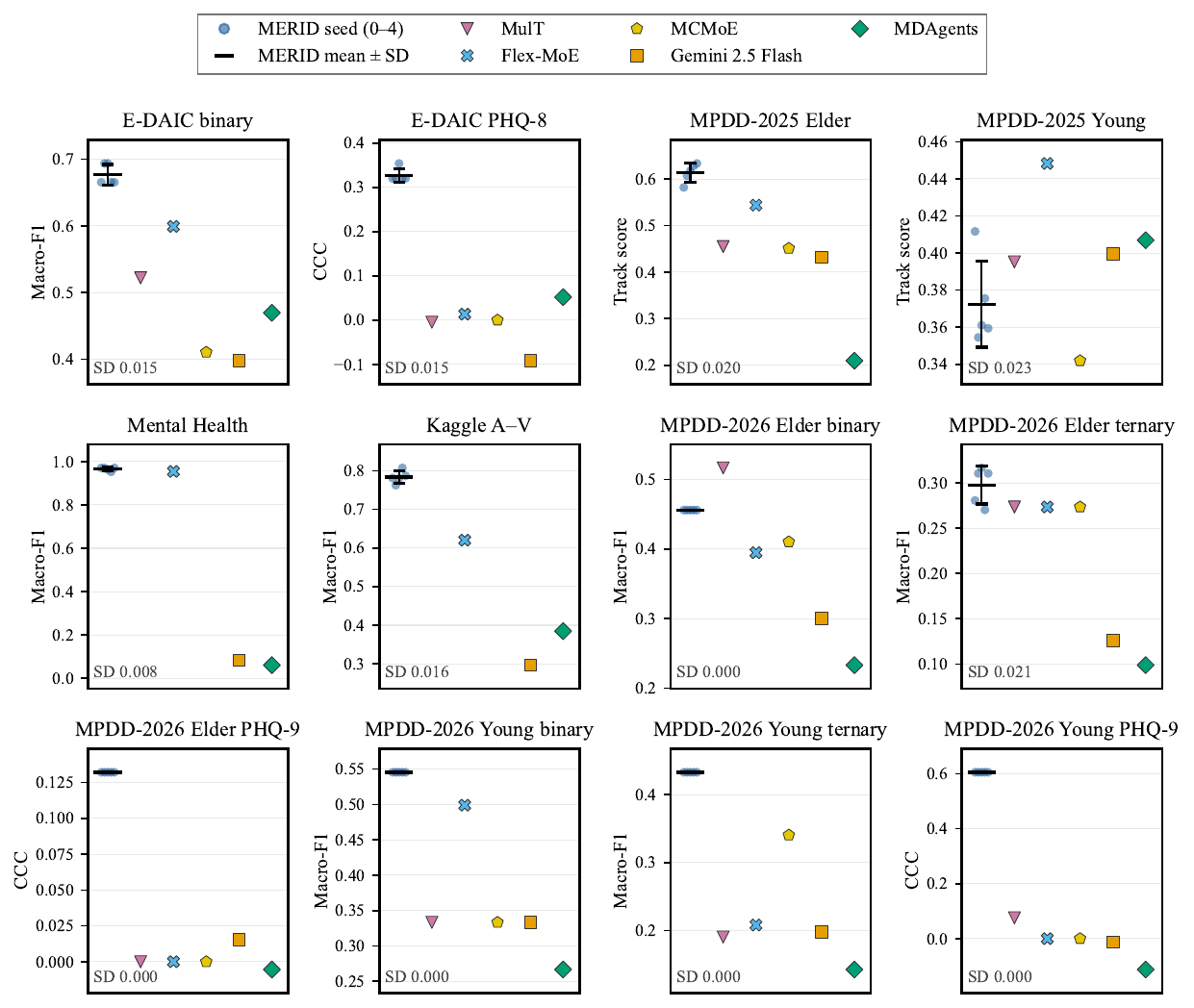}
\captionof{figure}{Test scores of \method{} under seeds 0--4 (dots) with their mean $\pm$ SD (black), against the baselines of Table~\ref{tab:main-results}. Each panel keeps its task's metric and scale and prints the seed SD. Baselines are single runs, except Flex-MoE on MPDD-2025 (mean of three runs).}
\label{fig:seed-study}
\end{minipage}\end{center}

The five-seed study uses complete runs with seeds 0--4. Table~\ref{tab:main-results} reports their means, while Table~\ref{tab:seed-study} provides individual seed scores and sample standard deviations. Figure~\ref{fig:seed-robustness} shows each run against the strongest baseline. Each run repeats the development search, the review and one test evaluation of the deployed candidates under the rule of the main results: the carried-over reference is retained unless a challenger exceeds it by one standard error, and the one-standard-error tied set of the selected candidate (at most five members) is deployed. Kaggle is reported through its nested cross-validation (Appendix~\ref{app:kaggle}). The seed sets the run's resampling, namely the development folds where cross-validation is used and the bootstrap estimates; LLM responses are not seeded. All 80 runs (16 run units, five seeds) used one code version and completed, with 1,006 LLM calls and \$59.18 of LLM usage in total. Baselines are single runs, except Flex-MoE on MPDD-2025, which averages three runs.

Figure~\ref{fig:seed-study} and Table~\ref{tab:seed-study} give every run. Standard deviations range from 0.000 to 0.023. On MPDD-2026 the incumbent rule retained the carried-over binary and PHQ-9 heads in all ten runs and the Young ternary head in all five Young runs, so these five columns deploy the same model under every seed. The Elder ternary head was replaced in every Elder run and varies from 0.2703 to 0.3167. E-DAIC binary takes two values, 0.6655 in three runs and 0.6937 in two. The largest spread is MPDD-2025 Young (0.3545--0.4116), where the five-seed mean of 0.3724 lies below Flex-MoE, MDAgents, Gemini 2.5 Flash and MulT.

\begin{table}[!htbp]
\caption{Test scores of \method{} under seeds 0--4 for every column of Table~\ref{tab:main-results}. Each seed is one complete run (development search, review and a single test evaluation). Best and second-best distinct seed scores within each row are bold and underlined, including ties. The last column reports the mean and sample standard deviation over the five runs.}
\label{tab:seed-study}
\centering
\begingroup
\mrAppendixDenseTableStyle
\resizebox{\linewidth}{!}{%
\begin{tabular}{>{\raggedright\arraybackslash}p{.25\linewidth}>{\raggedright\arraybackslash}p{.1\linewidth}*{5}{>{\centering\arraybackslash}p{\dimexpr(.45\linewidth-14\tabcolsep)/5\relax}}>{\centering\arraybackslash}p{.2\linewidth}}
\toprule
\rowcolor{mrAppHeader}
\textbf{Task} & \textbf{Metric} & \textbf{Seed 0} & \textbf{Seed 1} & \textbf{Seed 2} & \textbf{Seed 3} & \textbf{Seed 4} & \cellcolor{mrAppOurs}\textbf{Mean $\pm$ SD} \\
\midrule
\rowcolor{mrAppBand}
E-DAIC binary & Macro-F1 & \underline{0.6655} & \textbf{0.6937} & \textbf{0.6937} & \underline{0.6655} & \underline{0.6655} & \cellcolor{mrAppOurs}0.6768 $\pm$ 0.0154 \\
E-DAIC PHQ-8 & CCC & \underline{0.3201} & \underline{0.3201} & \textbf{0.3538} & \underline{0.3201} & \underline{0.3201} & \cellcolor{mrAppOurs}0.3268 $\pm$ 0.0151 \\
\rowcolor{mrAppBand}
MPDD-2025 Elder & Track score & 0.5820 & 0.6071 & 0.6177 & \underline{0.6282} & \textbf{0.6335} & \cellcolor{mrAppOurs}0.6137 $\pm$ 0.0204 \\
MPDD-2025 Young & Track score & \textbf{0.4116} & 0.3545 & 0.3610 & \underline{0.3755} & 0.3594 & \cellcolor{mrAppOurs}0.3724 $\pm$ 0.0233 \\
\rowcolor{mrAppBand}
Mental Health & Macro-F1 & \underline{0.9704} & \underline{0.9704} & 0.9630 & 0.9527 & \textbf{0.9710} & \cellcolor{mrAppOurs}0.9655 $\pm$ 0.0079 \\
Kaggle A--V & Macro-F1 & 0.7811 & 0.7624 & 0.7803 & \textbf{0.8074} & \underline{0.7869} & \cellcolor{mrAppOurs}0.7836 $\pm$ 0.0162 \\
\midrule
\rowcolor{mrAppBand}
MPDD-2026 Elder binary & Macro-F1 & \textbf{0.4559} & \textbf{0.4559} & \textbf{0.4559} & \textbf{0.4559} & \textbf{0.4559} & \cellcolor{mrAppOurs}0.4559 $\pm$ 0.0000 \\
MPDD-2026 Elder ternary & Macro-F1 & 0.2807 & \underline{0.3106} & \textbf{0.3167} & 0.2703 & \underline{0.3106} & \cellcolor{mrAppOurs}0.2978 $\pm$ 0.0208 \\
\rowcolor{mrAppBand}
MPDD-2026 Elder PHQ-9 & CCC & \textbf{0.1321} & \textbf{0.1321} & \textbf{0.1321} & \textbf{0.1321} & \textbf{0.1321} & \cellcolor{mrAppOurs}0.1321 $\pm$ 0.0000 \\
MPDD-2026 Young binary & Macro-F1 & \textbf{0.5455} & \textbf{0.5455} & \textbf{0.5455} & \textbf{0.5455} & \textbf{0.5455} & \cellcolor{mrAppOurs}0.5455 $\pm$ 0.0000 \\
\rowcolor{mrAppBand}
MPDD-2026 Young ternary & Macro-F1 & \textbf{0.4327} & \textbf{0.4327} & \textbf{0.4327} & \textbf{0.4327} & \textbf{0.4327} & \cellcolor{mrAppOurs}0.4327 $\pm$ 0.0000 \\
MPDD-2026 Young PHQ-9 & CCC & \textbf{0.6050} & \textbf{0.6050} & \textbf{0.6050} & \textbf{0.6050} & \textbf{0.6050} & \cellcolor{mrAppOurs}0.6050 $\pm$ 0.0000 \\
\bottomrule
\end{tabular}%
}
\endgroup
\end{table}

% Generated by E:/depressionAgent/figures/base_model_astra_20260925.py -- regenerate, do not edit by hand.
\FloatBarrier
\subsection{Base-model swap: GPT-6 Astra}
\label{app:base_model_astra}

We repeated the five-seed study of Table~\ref{tab:main-results} with GPT-6 Astra as the base model: the same 16 run units, seeds 0--4, code, carried-over references and deployment rule, with every LLM call of a run sent to the Azure OpenAI deployment \texttt{gpt-6-astra} through the same system prompts (Appendix~\ref{app:prompts}). That endpoint offers no web-search tool, so decisions that may search were made without it. Aborted runs were re-run under the same seed at most twice, a rule fixed before any re-run; all attempts are reported. All 80 tasks completed within that limit.

Table~\ref{tab:astra-scores} and Figure~\ref{fig:astra-seeds} compare the two base models. Astra is higher on E-DAIC binary Macro-F1 (0.7057 versus 0.6768), and lower on Mental Health Macro-F1 (0.9565 versus 0.9655), Kaggle A--V Macro-F1 (0.6070 versus 0.7836) and MPDD-2026 Elder ternary Macro-F1 (0.2523 versus 0.2978). The other 8 columns differ by less than the larger of the two seed standard deviations; the three MPDD-2026 Young columns are identical because both models' runs retained the carried-over heads. On MPDD-2026 Elder, Astra replaced the binary/PHQ-9 head in 2 of five seeds (PHQ-9 CCC 0.3015 in those seeds) and kept the carried-over head in the others, which is why its Elder PHQ-9 mean lies above Opus's with a large spread. The Kaggle gap arises before model selection. The frozen Kaggle adapter could not load in the read-only run environment, so every run derived its own. Opus vetoed none of the six source families, while Astra vetoed eGeMAPS, facial action units and facial landmarks as requiring extractors that were not installed (three vetoes per run) and searched the remaining acoustic and CNN recipes only.

Table~\ref{tab:astra-failures} separates the two models' format failures. Five of 85 Astra attempts aborted, all on MPDD-2026, whose search rejects a proposal outside its bounded action space and ends the run rather than substituting a default. Five of Astra's 135 audited MPDD-2026 proposals were invalid: three misspelled the \texttt{wav2vec} extractor and two were malformed JSON (Table~\ref{tab:astra-invalid}); none of Opus's 112 were. Conversely, the review chair's reply had to be re-asked for format 35 times across the Opus runs and never in the Astra runs; that step re-asks, so those runs continued. Which failures end a run therefore depends both on the model and on where the pipeline retries. Per completed run, Astra used fewer calls and about half the tokens, at twice the list price per token, so the two studies cost about the same: \$0.76 against \$0.74 per run at uncached list prices (Table~\ref{tab:astra-failures}). The usage of aborted attempts is not included.

\begin{figure}[!htbp]
\centering
\includegraphics[width=\linewidth]{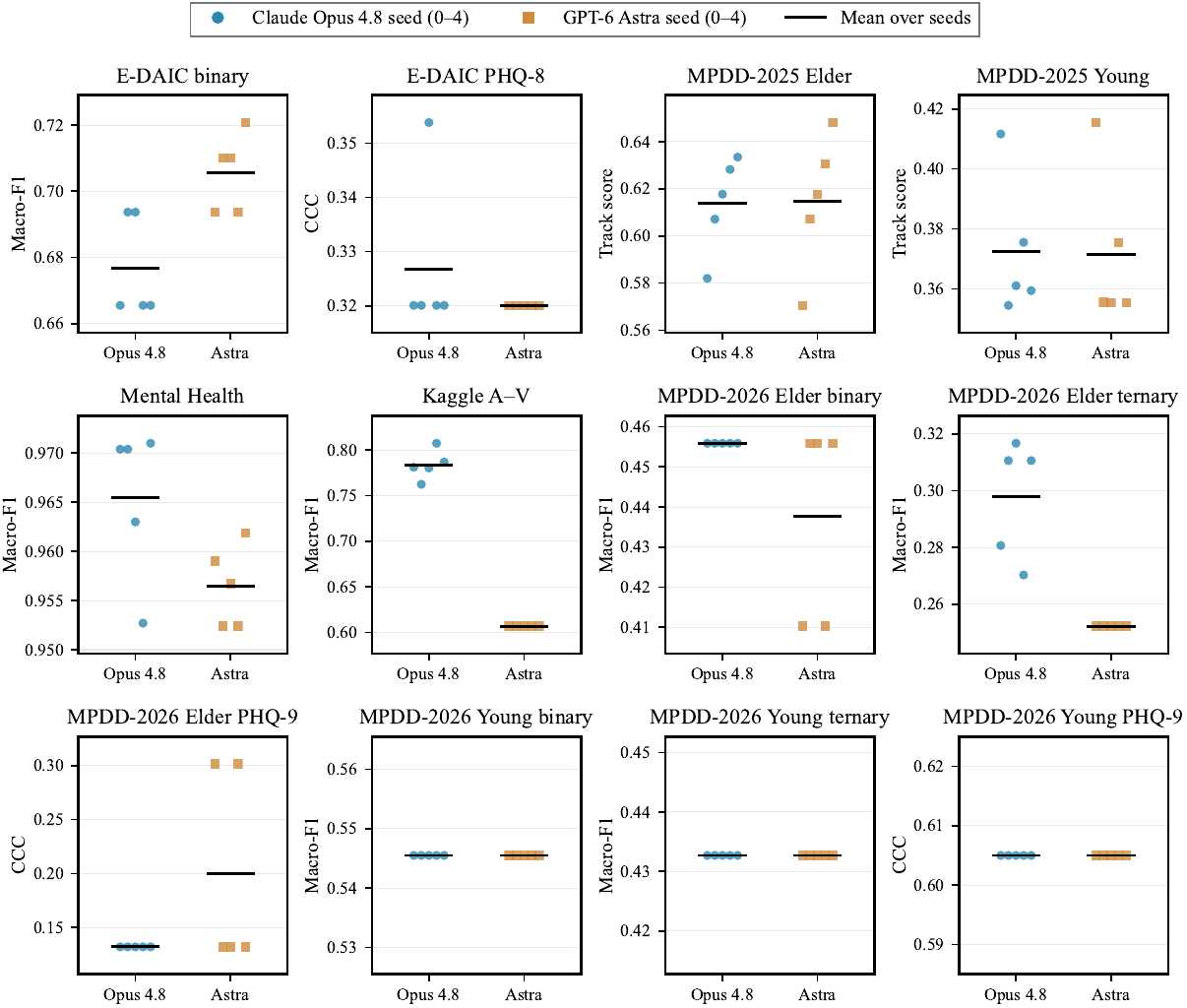}
\caption{Test scores under seeds 0--4 with Claude Opus 4.8 (blue) and GPT-6 Astra (orange) as the base model; bars are seed means. Each panel keeps its task's metric and scale.}
\label{fig:astra-seeds}
\end{figure}

\begin{table}[!htbp]
\caption{Test scores of \method{} with Claude Opus 4.8 and GPT-6 Astra as the base model: mean $\pm$ sample SD over seeds 0--4 of the same 16 run units, protocol and code. $\Delta$ is Astra minus Opus.}
\label{tab:astra-scores}
\centering
\begingroup
\mrAppendixDenseTableStyle
\begin{tabular}{>{\raggedright\arraybackslash}p{.27\linewidth}>{\raggedright\arraybackslash}p{.11\linewidth}*{2}{>{\centering\arraybackslash}p{.19\linewidth}}>{\centering\arraybackslash}p{.1\linewidth}}
\toprule
\rowcolor{mrAppHeader}
\textbf{Task} & \textbf{Metric} & \textbf{Opus 4.8} & \textbf{GPT-6 Astra} & $\boldsymbol{\Delta}$ \\
\midrule
\rowcolor{mrAppBand}
E-DAIC binary & Macro-F1 & 0.6768 $\pm$ 0.0154 & 0.7057 $\pm$ 0.0118 & +0.0289 \\
E-DAIC PHQ-8 & CCC & 0.3268 $\pm$ 0.0151 & 0.3201 $\pm$ 0.0000 & -0.0067 \\
\rowcolor{mrAppBand}
MPDD-2025 Elder & Track score & 0.6137 $\pm$ 0.0204 & 0.6148 $\pm$ 0.0291 & +0.0011 \\
MPDD-2025 Young & Track score & 0.3724 $\pm$ 0.0233 & 0.3714 $\pm$ 0.0261 & -0.0010 \\
\rowcolor{mrAppBand}
Mental Health & Macro-F1 & 0.9655 $\pm$ 0.0079 & 0.9565 $\pm$ 0.0042 & -0.0090 \\
Kaggle A--V & Macro-F1 & 0.7836 $\pm$ 0.0162 & 0.6070 $\pm$ 0.0000 & -0.1766 \\
\rowcolor{mrAppBand}
MPDD-2026 Elder binary & Macro-F1 & 0.4559 $\pm$ 0.0000 & 0.4377 $\pm$ 0.0250 & -0.0182 \\
MPDD-2026 Elder ternary & Macro-F1 & 0.2978 $\pm$ 0.0208 & 0.2523 $\pm$ 0.0000 & -0.0455 \\
\rowcolor{mrAppBand}
MPDD-2026 Elder PHQ-9 & CCC & 0.1321 $\pm$ 0.0000 & 0.1999 $\pm$ 0.0928 & +0.0678 \\
MPDD-2026 Young binary & Macro-F1 & 0.5455 $\pm$ 0.0000 & 0.5455 $\pm$ 0.0000 & +0.0000 \\
\rowcolor{mrAppBand}
MPDD-2026 Young ternary & Macro-F1 & 0.4327 $\pm$ 0.0000 & 0.4327 $\pm$ 0.0000 & +0.0000 \\
MPDD-2026 Young PHQ-9 & CCC & 0.6050 $\pm$ 0.0000 & 0.6050 $\pm$ 0.0000 & +0.0000 \\
\bottomrule
\end{tabular}
\endgroup
\end{table}

\begin{table}[!htbp]
\caption{Failure analysis of the two five-seed studies. An attempt is one run of a task; aborted attempts were re-run under the same seed, at most twice. MPDD-2026 audits every proposal: an invalid one ends the run, a duplicate is re-asked once. The chair is re-asked when its reply lacks the required format. Token counts and costs cover completed runs; costs use uncached list prices, an upper bound when part of the input was cached.}
\label{tab:astra-failures}
\centering
\begingroup
\mrAppendixDenseTableStyle
\begin{tabular}{>{\raggedright\arraybackslash}p{.5\linewidth}*{2}{>{\centering\arraybackslash}p{.2\linewidth}}}
\toprule
\rowcolor{mrAppHeader}
\textbf{Quantity} & \textbf{Opus 4.8} & \textbf{GPT-6 Astra} \\
\midrule
\rowcolor{mrAppBand}
Run attempts for 80 tasks & 80 & 85 \\
Aborted attempts & 0 & 5 \\
\rowcolor{mrAppBand}
MPDD-2026 proposals audited & 112 & 135 \\
\quad invalid (run aborts) & 0 & 5 (3.7\%) \\
\rowcolor{mrAppBand}
\quad duplicates (re-asked once) & 0 & 0 \\
Chair replies re-asked for format & 35 & 0 \\
\rowcolor{mrAppBand}
Failed reviews / parse fallbacks / truncation warnings & 0 / 0 / 0 & 0 / 0 / 0 \\
Kaggle adapter turns / failed turns / vetoed sources & 15 / 10 / 0 & 6 / 1 / 15 \\
\rowcolor{mrAppBand}
LLM calls per completed run & 12.6 & 11.0 \\
Input / output tokens per completed run (k) & 118 / 6.1 & 60 / 3.2 \\
\rowcolor{mrAppBand}
List price per 1M input / output tokens (USD) & 5 / 25 & 10 / 50 \\
LLM cost per completed run / all 80 runs (USD) & 0.74 / 59.18 & 0.76 / 60.57 \\
\bottomrule
\end{tabular}
\endgroup
\end{table}

\begin{table}[!htbp]
\caption{The invalid GPT-6 Astra proposals that aborted a run: task, seed, attempt, the proposal's position in the run, and the offending part of its final JSON line, checked against the action space.}
\label{tab:astra-invalid}
\centering
\begingroup
\mrAppendixDenseTableStyle
\begin{tabular}{>{\raggedright\arraybackslash}p{.17\linewidth}*{3}{>{\centering\arraybackslash}p{.05\linewidth}}>{\raggedright\arraybackslash}p{.16\linewidth}>{\raggedright\arraybackslash}p{.34\linewidth}}
\toprule
\rowcolor{mrAppHeader}
\textbf{Task} & \textbf{Seed} & \textbf{Att.} & \textbf{Prop.} & \textbf{Failure} & \textbf{Offending fragment} \\
\midrule
\rowcolor{mrAppBand}
MPDD-2026 Elder & 0 & 1 & 7 & out-of-space value & \texttt{"sev\_fusion": "wav2emb:densenet"} \\
MPDD-2026 Elder & 2 & 1 & 6 & out-of-space value & \texttt{"audio": "wav2span"} \\
\rowcolor{mrAppBand}
MPDD-2026 Elder & 2 & 2 & 1 & out-of-space value & \texttt{"audio": "wav2.0"} \\
MPDD-2026 Elder & 4 & 1 & 6 & malformed JSON & \texttt{gp\_stats\_pz",\allowbreak{}"model":"plsn4":"bad"\}\}} \\
\rowcolor{mrAppBand}
MPDD-2026 Young & 2 & 1 & 9 & malformed JSON & \texttt{1.0",\allowbreak{}"sev\_model":"gb0.05d1,\allowbreak{}"sev\_fusion":} \\
\bottomrule
\end{tabular}
\endgroup
\end{table}

\FloatBarrier
\subsection{Additional Ablation Analysis}
\label{app:ablation_selection}

Table~\ref{tab:ablation-full} gives the complete module and GSC ablation results. Figure~\ref{fig:ablation-analysis} in the main text summarizes these comparisons and the sensitivity scans.

\textbf{Module ablations.} Table~\ref{tab:ablation-full} reports all nine test tasks and the three GSC development tasks. Table~\ref{tab:ablation-studies} summarizes these comparisons in the main text. The full configuration is the five-seed main benchmark (Appendix~\ref{app:seed_study}, seeds 0--4), which carries the previously selected pipeline; the two removals use seeds 0--2 and start without it. The CPE removal skips initial joint exploration and tests each proposed recipe with the incumbent predictor or each proposed predictor with the incumbent recipe. It keeps review and verification and permits up to eight rounds. The EGE removal keeps exploration but disables review, the replacement margin, and the terminal gate. These are configuration comparisons with different realized search and review costs. A seed's score averages its completed runs. Reported SD is the sample SD over seed means: five for the full configuration and three for each removal. For radar axes, the 1\,s and 5\,s MPDD scores are averaged within each seed before computing the mean and SD. Radar axes use independent linear scales with their raw bounds labeled. All MPDD module results concern the Elder cohort. Fresh-adapter GSC comparisons use the official E-DAIC development split or grouped Mental Health out-of-fold predictions.

\begin{table}[!htbp]
\caption{Module ablations. Test results report mean $\pm$ sample SD over seed means: \method{} (full) is Table~\ref{tab:main-results}'s five-seed result (seeds 0--4); w/o CPE and w/o EGE use seeds 0--2 and start without an inherited pipeline. Repeated runs are averaged within each seed. MPDD-2025 uses the official composite score. GSC results use development data with one run per configuration. Best and second-best distinct means within each row are \textbf{bold} and \underline{underlined}.}
\label{tab:ablation-full}
\centering
\begingroup
\mrAppendixTableStyle
\begin{tabular}{>{\raggedright\arraybackslash}p{.32\linewidth}*{3}{>{\centering\arraybackslash}p{\dimexpr(.68\linewidth-8\tabcolsep)/3\relax}}}
\toprule
\rowcolor{mrAppHeader}
Task / test metric $\uparrow$ & w/o CPE & w/o EGE & \cellcolor{mrAppOurs}\textbf{\method{} (full)} \\
\midrule
\rowcolor{mrAppBand}
E-DAIC binary (F1) & 0.5151 $\pm$ 0.0760 & \underline{0.5162} $\pm$ 0.0345 & \cellcolor{mrAppOurs}\textbf{0.6768} $\pm$ 0.0154 \\
E-DAIC PHQ-8 (CCC) & 0.0361 $\pm$ 0.0452 & \underline{0.2997} $\pm$ 0.0196 & \cellcolor{mrAppOurs}\textbf{0.3268} $\pm$ 0.0151 \\
\rowcolor{mrAppBand}
Mental Health (F1) & \textbf{0.9672} $\pm$ 0.0149 & 0.9530 $\pm$ 0.0005 & \cellcolor{mrAppOurs}\underline{0.9655} $\pm$ 0.0079 \\
\midrule
MPDD-25 Elder 1\,s binary & 0.7231 $\pm$ 0.0846 & \textbf{0.7850} $\pm$ 0.0042 & \cellcolor{mrAppOurs}\underline{0.7593} $\pm$ 0.0349 \\
\rowcolor{mrAppBand}
MPDD-25 Elder 1\,s ternary & \underline{0.6089} $\pm$ 0.0822 & 0.5955 $\pm$ 0.0667 & \cellcolor{mrAppOurs}\textbf{0.6295} $\pm$ 0.0673 \\
MPDD-25 Elder 1\,s quinary & \underline{0.4583} $\pm$ 0.0341 & 0.4472 $\pm$ 0.0482 & \cellcolor{mrAppOurs}\textbf{0.4738} $\pm$ 0.0647 \\
\rowcolor{mrAppBand}
MPDD-25 Elder 5\,s binary & \textbf{0.8065} $\pm$ 0.0958 & \underline{0.7966} $\pm$ 0.0319 & \cellcolor{mrAppOurs}0.7330 $\pm$ 0.0264 \\
MPDD-25 Elder 5\,s ternary & 0.5590 $\pm$ 0.1041 & \underline{0.5819} $\pm$ 0.0161 & \cellcolor{mrAppOurs}\textbf{0.6849} $\pm$ 0.0327 \\
\rowcolor{mrAppBand}
MPDD-25 Elder 5\,s quinary & \underline{0.4448} $\pm$ 0.0411 & \textbf{0.4521} $\pm$ 0.0467 & \cellcolor{mrAppOurs}0.4017 $\pm$ 0.0493 \\
\midrule[1.2pt]
\rowcolor{mrAppHeader}
Task / development metric $\uparrow$ & w/o inspection & w/o repair & \cellcolor{mrAppOurs}\textbf{Full GSC} \\
\midrule
\rowcolor{mrAppBand}
E-DAIC binary (F1) & \underline{0.6190} & 0.5359 & \cellcolor{mrAppOurs}\textbf{0.6316} \\
E-DAIC PHQ-8 (CCC) & \textbf{0.2771} & \textbf{0.2771} & \cellcolor{mrAppOurs}\underline{0.2464} \\
\rowcolor{mrAppBand}
Mental Health (F1) & \underline{0.9653} & 0.9622 & \cellcolor{mrAppOurs}\textbf{0.9787} \\
\bottomrule
\end{tabular}
\endgroup
\end{table}

\paragraph{Task-dependent effects.}
The full configuration has the highest mean on five of the nine test tasks: E-DAIC binary, E-DAIC PHQ-8, both MPDD-2025 ternary tasks, and the 1\,s quinary task. Its largest observed advantage over CPE removal is on PHQ-8 CCC (0.3268 versus 0.0361). The other four tasks favor a removal configuration: Mental Health and the 5\,s binary task without CPE, and the 1\,s binary and 5\,s quinary tasks without EGE. GSC inspection and repair improve development classification F1, whereas PHQ-8 CCC is higher after either removal. These comparisons use each task's own metric; module test outcomes and GSC development checks are reported separately.

\FloatBarrier
\subsection{Modality Availability Analysis}
\label{app:modality_study}

We compare seven available-source sets: audio (A), video (V), transcript text (T), and their four combinations. Binary classification and PHQ-8 estimation each use three seeds (0--2), giving 42 runs. All runs share the checked adapter and official 163/56/56 training/development/test split. Each starts without a historical reference, permits up to 256 initial candidate evaluations, and reserves up to 64 additional evaluations for one review round. Research memory is disabled. Binary selection uses Macro-F1; PHQ-8 selection uses negative RMSE. Search retains the incumbent with a one-standard-error promotion margin and the terminal reliability gate. Each locked one-standard-error deployment is evaluated once on the test set; all 42 evaluations completed successfully. These modality runs form a separate comparison from the main benchmark.

Available modalities constrain the search; a deployed pipeline may use a nonempty subset. Table~\ref{tab:modality-study} therefore reports both available and deployed sources. The six text-only runs were repeated after an adapter check incorrectly required aggregation changes to alter TF--IDF features. This correction preceded all test evaluations and restored the shared adapter; the superseded runs are excluded.

\paragraph{Modality-dependent outcomes.}
Making audio and text available yields the highest mean binary test Macro-F1 (0.6468), compared with 0.6107 for text and 0.5336 for audio alone. Video alone scores 0.6086, close to the text-only result. For PHQ-8, text-enabled arms achieve CCC of 0.2503--0.2864, while arms without text remain below 0.12. Text alone has the highest CCC and lowest RMSE. Adding video to the available T or A+T inputs lowers mean binary test Macro-F1. For PHQ-8, A+V+T exceeds A+T in CCC (0.2851 versus 0.2507) and remains close to text alone (0.2864). The A+V+T arm has the highest binary development Macro-F1 (0.7363) but the lowest binary test Macro-F1 (0.4855). Thus, its development-set advantage does not carry over to this test set. All three seeds share the same 56 test participants; their SD measures run variability, and these mean comparisons do not establish population-level differences.

\begin{table}[!htbp]
\caption{Available modalities on E-DAIC. Values are mean $\pm$ sample SD over three seeds. Development scores describe the selected candidate; test scores describe the locked deployment. PHQ-8 development RMSE is the negated selection score. Deployed sources are the union across ensemble members within a run; semicolons separate configurations observed across seeds. Best and second-best test means within each task and metric are bold and underlined.}
\label{tab:modality-study}
\centering
\begingroup
\mrAppendixTableStyle
\begin{tabular}{>{\raggedright\arraybackslash}p{.10\linewidth}*{3}{>{\centering\arraybackslash}p{.22\linewidth}}>{\raggedright\arraybackslash}p{\dimexpr.24\linewidth-10\tabcolsep\relax}}
\toprule
\rowcolor{mrAppHeader}
\multicolumn{5}{l}{\textbf{Binary classification}} \\
\rowcolor{mrAppHeader}
\textbf{Available} & \textbf{Dev Macro-F1} & \textbf{Test Macro-F1 $\uparrow$} & \textbf{Test Accuracy $\uparrow$} & \textbf{Deployed} \\
\midrule
\rowcolor{mrAppBand}
A & 0.6348 $\pm$ 0.0000 & 0.5336 $\pm$ 0.0679 & 0.5655 $\pm$ 0.0412 & A \\
V & 0.5754 $\pm$ 0.0476 & 0.6086 $\pm$ 0.0092 & 0.6488 $\pm$ 0.0103 & V \\
\rowcolor{mrAppBand}
T & 0.6750 $\pm$ 0.0531 & \underline{0.6107} $\pm$ 0.0535 & 0.6607 $\pm$ 0.0309 & T \\
AV & 0.6348 $\pm$ 0.0000 & 0.5612 $\pm$ 0.1064 & \underline{0.6726} $\pm$ 0.0103 & A+V \\
\rowcolor{mrAppBand}
AT & 0.7348 $\pm$ 0.0000 & \textbf{0.6468} $\pm$ 0.0153 & \textbf{0.6964} $\pm$ 0.0179 & A+T \\
VT & 0.7056 $\pm$ 0.0531 & 0.5888 $\pm$ 0.0456 & 0.6548 $\pm$ 0.0206 & V+T \\
\rowcolor{mrAppBand}
AVT & 0.7363 $\pm$ 0.0000 & 0.4855 $\pm$ 0.0148 & 0.6369 $\pm$ 0.0103 & A+T; A+V+T \\
\midrule[1.2pt]
\rowcolor{mrAppHeader}
\multicolumn{5}{l}{\textbf{PHQ-8 estimation}} \\
\rowcolor{mrAppHeader}
\textbf{Available} & \textbf{Dev RMSE} & \textbf{Test RMSE $\downarrow$} & \textbf{Test CCC $\uparrow$} & \textbf{Deployed} \\
\midrule
\rowcolor{mrAppBand}
A & 5.4463 $\pm$ 0.0681 & 6.3957 $\pm$ 0.0151 & 0.1113 $\pm$ 0.0450 & A \\
V & 6.0250 $\pm$ 0.0000 & 6.6659 $\pm$ 0.0945 & 0.0111 $\pm$ 0.0167 & V \\
\rowcolor{mrAppBand}
T & 5.4154 $\pm$ 0.0000 & \textbf{5.8081} $\pm$ 0.0265 & \textbf{0.2864} $\pm$ 0.0281 & T \\
AV & 5.5456 $\pm$ 0.3156 & 6.5999 $\pm$ 0.3840 & 0.0530 $\pm$ 0.0605 & A; A+V \\
\rowcolor{mrAppBand}
AT & 5.4544 $\pm$ 0.0676 & 5.8854 $\pm$ 0.0627 & 0.2507 $\pm$ 0.0400 & A+T \\
VT & 5.4544 $\pm$ 0.0676 & 5.8934 $\pm$ 0.0996 & 0.2503 $\pm$ 0.0773 & V+T \\
\rowcolor{mrAppBand}
AVT & 5.4154 $\pm$ 0.0000 & \underline{5.8266} $\pm$ 0.0075 & \underline{0.2851} $\pm$ 0.0055 & A+V+T \\
\bottomrule
\end{tabular}
\endgroup
\end{table}

\FloatBarrier
\subsection{Hyperparameter Sensitivity Analysis}
\label{app:supplementary_sensitivity}

\paragraph{Performance under different hyperparameters.}
Figure~\ref{fig:ablation-analysis}(b--d) reports the fold, budget, and round-limit scans. Five validation folds match ten on MPDD-2025 Elder. E-DAIC binary Macro-F1 is 0.5919 with 64 candidate evaluations versus 0.5492 with 256, and the round-limit scan peaks at one revision round. Both MPDD-2025 cohorts remain unchanged across this scan.

Figure~\ref{fig:appendix-metric-sensitivity} extends the main-text Macro-F1 scans with PHQ severity estimation and MPDD task scores. E-DAIC PHQ CCC remains 0.3201 across the budget settings. The two runs allowing three revisions yield 0.3201 and 0.3059, showing that a longer revision allowance need not improve severity prediction. MPDD-2025 ternary official scores remain unchanged across the round-limit scan, while the fold scan raises Young from 0.3615 at $K=5$ to 0.4048 at $K=10$ and leaves Elder at 0.6842.

\begin{figure}[!htbp]
\centering
\includegraphics[width=\linewidth]{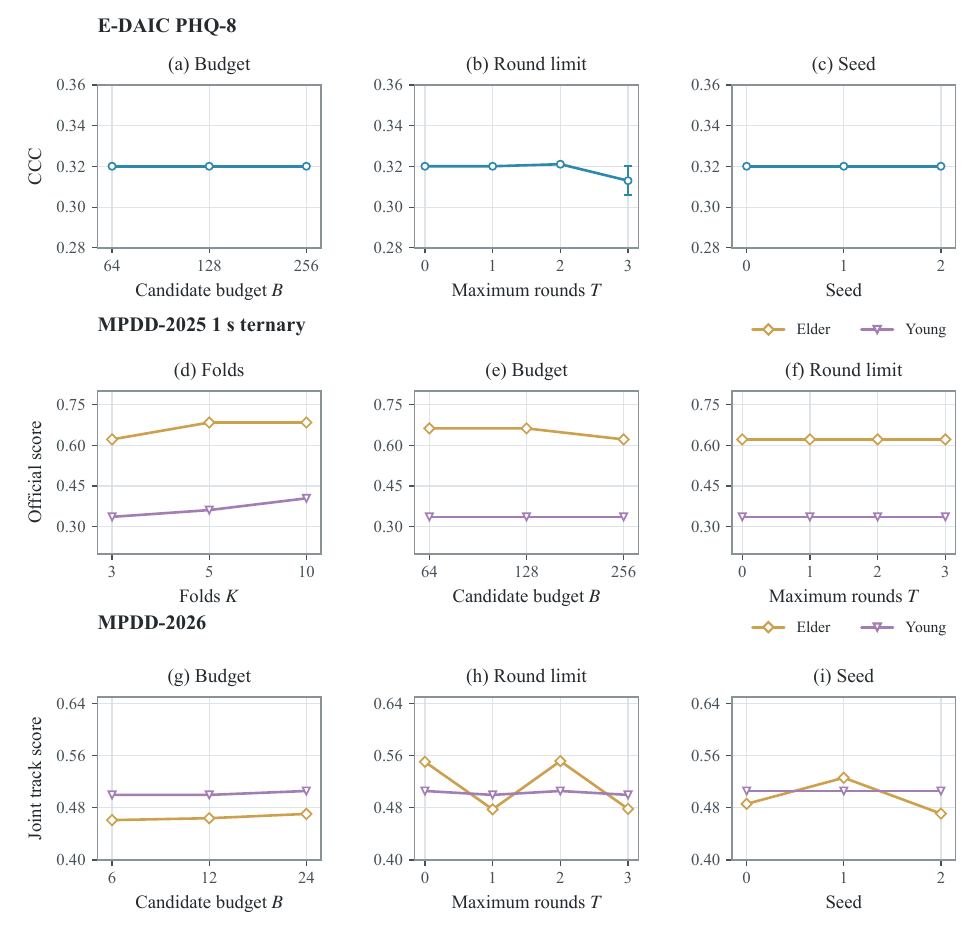}
\caption{Additional sensitivity scans from the recorded configurations. Top: E-DAIC PHQ CCC versus budget, maximum revision rounds, and seed. Middle: MPDD-2025 1\,s ternary official scores versus folds, budget, and maximum rounds. Bottom: MPDD-2026 joint track scores versus budget, maximum rounds, and seed. Points average available runs at each setting, and whiskers show repeated-run ranges where nonzero. Panels use separate metric scales. These are separate runs at different settings, rather than successive states of one trajectory.}
\label{fig:appendix-metric-sensitivity}
\end{figure}

MPDD-2026 uses fixed ten-fold inputs and selection by the maximum development track score. Its independent Elder runs at the same nominal default yield 0.4641, 0.4777, and 0.4861, while the budget scan spans 0.4612--0.4707. This observed repeat variation is larger than the budget-scan range. Using each scan's original baseline configuration, including $K=3$ in the configurable-fold scans, the baseline attains the highest test score in five of 17 scans, including four ties. These results support task-specific assessment of additional search rather than a uniform benefit from larger budgets or more rounds. The protocol below specifies the selection rules and repeated configurations.

\phantomsection\label{app:analysis-protocols}
\textbf{Parameter scans.} Corpus tasks use $K\in\{3,5,10\}$ where applicable, $B\in\{64,128,256\}$, $T\in\{0,1,2,3\}$, and seeds $\{0,1,2\}$. E-DAIC uses its designated development split, so $K$ is fixed. Mental Health defaults to $B=128$, and the other corpus tasks default to $B=256$. Their default revision count is $T=1$ and seed is 0. MPDD-2026 uses fixed ten-fold inputs and $B\in\{6,12,24\}$, with default $B=12$. Each scan changes one parameter and retains the recorded values of the others. We select $K=5$ as the default for configurable folds. Corpus scans retain the reference as a candidate and use one-standard-error selection with the terminal gate disabled. MPDD-2026 scans use the maximum development track score. Repeated jobs at a setting are averaged for plotting, with their observed range shown where nonzero. These ranges describe repeated observations rather than confidence intervals. The three classification sensitivity panels in Figure~\ref{fig:ablation-analysis}(f--h) use Macro-F1 from the same runs. Rings mark default settings, shaded bands show observed seed ranges, and thin vertical lines show repeated-run ranges. MPDD-2025 official composites and Macro-F1 are recorded separately. No averages across different metrics are taken.

\textbf{Repeated configurations.} Independent MPDD-2026 jobs at $K=10$, $B=12$, $T=1$, and seed 0 return Elder track scores of 0.4641, 0.4777, and 0.4861. Young returns 0.4999, 0.4999, and 0.5059. Each scan retains its own measured default, so the budget, round, and seed panels can differ at nominally identical configurations. Seeds index independent runs of the specified configuration.

\FloatBarrier
\subsection{Selection and Deployment Analysis}
\label{app:selection_deployment}

\FloatBarrier
\subsubsection{Selection-Rule Comparison}
\label{app:selection_rules}

An E-DAIC revision raised the best development Macro-F1 from 0.7363 to 0.7737, still below the acceptance threshold of $0.7172+0.0794=0.7966$. The procedure therefore retained the incumbent as the reference pipeline. Figure~\ref{fig:appendix-selection-feedback}(a,b) places this decision alongside a separate selection-rule comparison. Incumbent-based selection averaged 0.6843 test Macro-F1 over three runs; classic one-standard-error selection and maximum-score selection scored 0.4991 and 0.4876 in one run each. Table~\ref{tab:selection-rule} summarizes scores and run counts.

\begin{table}[!htbp]
\caption{Selection-rule sensitivity on E-DAIC binary classification. The incumbent-based rule reports the mean and range of three test runs; alternative rules were each evaluated once. Development and test scores are Macro-F1 ($\uparrow$).}
\label{tab:selection-rule}
\centering
\begingroup
\mrAppendixTableStyle
\begin{tabular}{>{\raggedright\arraybackslash}p{.46\linewidth}>{\centering\arraybackslash}p{.17\linewidth}>{\centering\arraybackslash}p{\dimexpr.37\linewidth-6\tabcolsep\relax}}
\toprule
\rowcolor{mrAppHeader}
\textbf{Selection rule} & \textbf{Development} & \textbf{Test} \\
\midrule
\rowcolor{mrAppOurs}
Incumbent $+$ one standard error & 0.7172 & $\mathbf{0.6843}\;[0.6655,0.6937]$ \\
Classic one-standard-error rule & 0.7737 & 0.4991 \\
\rowcolor{mrAppBand}
Maximum development score & 0.7737 & 0.4876 \\
\bottomrule
\end{tabular}
\endgroup
\end{table}

\FloatBarrier
\subsubsection{Deployment Configurations}
\label{app:deployment_rules}

Table~\ref{tab:deployment-rules} reports three historical deployment configurations. Reference $+$ one SE requires a challenger to exceed the carried reference by one subject-bootstrap standard error; these runs also use a rank-transfer check that can fall back to one-standard-error selection. The winner's-curse configuration adds the terminal margin $\lambda\,\mathrm{SE}\sqrt{2\ln N}$ with $\lambda=1$ for $N$ evaluated candidates. The search-selection configuration keeps the reference as an ordinary candidate, selects the simplest candidate within one SE of the best, and disables the terminal gate. Thus, ``search selection'' here means \texttt{one\_se\_of\_best}, not score argmax.

Each column comes from separate runs, so the table compares recorded configurations rather than reselecting one shared candidate pool. No configuration leads on every task: reference $+$ one SE, the winner's-curse configuration, and search selection are highest or tied on five, five, and four of the nine tasks, respectively. The Mental Health reference score of 0.9630 belongs to this historical run family; the earlier benchmark record reports a separate run at 0.9607.

\begin{table}[!htbp]
\caption{Historical deployment configurations, with one seed-0 run per cell. MPDD-2025 uses the official composite. The Mental Health winner's-curse-family record predates that gate and uses only the rank-transfer check. Best and second-best distinct values per row are bold and underlined. Each cell uses the latest logged run for its protocol, configuration, task, and seed.}
\label{tab:deployment-rules}
\centering
\begingroup
\mrAppendixTableStyle
\begin{tabular}{>{\raggedright\arraybackslash}p{.37\linewidth}*{3}{>{\centering\arraybackslash}p{\dimexpr(.63\linewidth-8\tabcolsep)/3\relax}}}
\toprule
\rowcolor{mrAppHeader}
\textbf{Task / metric} & \cellcolor{mrAppOurs}\shortstack{\textbf{Reference}\\\textbf{$+$ one SE}} & \shortstack{\textbf{Winner's-curse}\\\textbf{configuration}} & \shortstack{\textbf{Search selection}\\\textbf{(one SE)}} \\
\midrule
\rowcolor{mrAppBand}
E-DAIC binary (Macro-F1) & \cellcolor{mrAppOurs}\textbf{0.6937} & \underline{0.6814} & 0.5492 \\
E-DAIC PHQ-8 (CCC) & \cellcolor{mrAppOurs}\textbf{0.3201} & \textbf{0.3201} & \textbf{0.3201} \\
\rowcolor{mrAppBand}
Mental Health (Macro-F1) & \cellcolor{mrAppOurs}\textbf{0.9630} & \underline{0.9535} & \underline{0.9535} \\
MPDD-25 Elder 1\,s binary & \cellcolor{mrAppOurs}\textbf{0.7212} & \textbf{0.7212} & \textbf{0.7212} \\
\rowcolor{mrAppBand}
MPDD-25 Elder 1\,s ternary & \cellcolor{mrAppOurs}\textbf{0.6842} & \textbf{0.6842} & \underline{0.6220} \\
MPDD-25 Elder 1\,s quinary & \cellcolor{mrAppOurs}0.3962 & \textbf{0.5490} & \underline{0.4011} \\
\rowcolor{mrAppBand}
MPDD-25 Elder 5\,s binary & \cellcolor{mrAppOurs}\underline{0.7212} & \underline{0.7212} & \textbf{0.7917} \\
MPDD-25 Elder 5\,s ternary & \cellcolor{mrAppOurs}0.5937 & \textbf{0.7002} & \underline{0.6145} \\
\rowcolor{mrAppBand}
MPDD-25 Elder 5\,s quinary & \cellcolor{mrAppOurs}\underline{0.4268} & 0.3143 & \textbf{0.4326} \\
\bottomrule
\end{tabular}
\endgroup
\end{table}

\clearpage
\subsubsection{Terminal-Margin Sensitivity}
\label{app:gate_multiplier}

We replay the terminal winner's-curse check using the saved development gain $g$, reference standard error $\sigma$, and number of evaluated candidates $N$. A proposed promotion survives when
\begin{equation}
g>\lambda\sigma\sqrt{2\ln N},
\qquad
\lambda^*=\frac{g}{\sigma\sqrt{2\ln N}}.
\label{eq:app-gate-multiplier}
\end{equation}
The replay changes this terminal decision without conducting a new search. For each of five proposed promotions, it compares the promoted pipeline's saved single-model test replay with the reference deployment recorded when the gate reverted that promotion. The five switching points span 0.1483--0.3536, placing $\lambda=1$ beyond every switch.

Figure~\ref{fig:gate-multiplier-replay}(c,d) shows each task separately. Retaining the proposed promotion would lower the test score in three cases (MPDD-2025 Elder 1\,s ternary, 1\,s quinary, and 5\,s ternary), raise it in one (5\,s quinary), and leave it unchanged on E-DAIC binary. Three other tasks had no promotion and remain unchanged across the sweep: E-DAIC PHQ-8 and MPDD-2025 Elder 1\,s/5\,s binary. Mental Health lacks a winner's-curse promotion record and is excluded. The replay describes the empirical tradeoff of this terminal rule; its multiplier does not calibrate the uniform population-error bound in Appendix~\ref{app:rsi_population_improvement}.

\begin{figure}[!htbp]
\centering
\includegraphics[width=\linewidth]{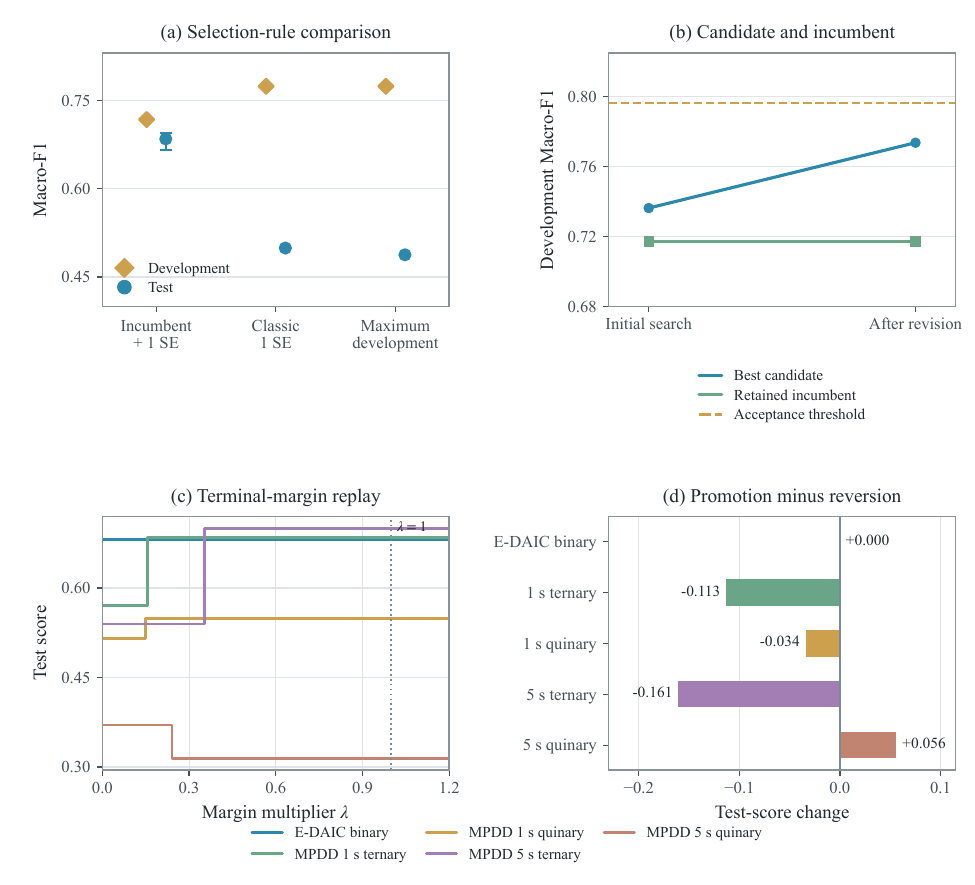}
\caption{Selection and verification analyses. (a,b) E-DAIC selection-rule and revision records; the test whisker spans three incumbent-based runs, while each alternative has one run. (c,d) A separate terminal-margin replay of five proposed promotions. E-DAIC uses Macro-F1; MPDD-2025 Elder uses its official score. Each task is shown separately. Panel (d) compares saved promoted single-model predictions with the recorded reverted deployment.}
\label{fig:appendix-selection-feedback}
\label{fig:gate-multiplier-replay}
\end{figure}

\FloatBarrier
\subsection{Search Strategy Analysis}
\label{app:search_axes}

Table~\ref{tab:search-axes} reports a separate historical search comparison. Staged search screens recipes before tuning predictors; fixed-recipe and fixed-predictor arms hold the named axis constant. The one-SE arm changes selection to the classic one-standard-error rule, while the random-space arms replace the proposed classifier or recipe space with a random space of the same size. Other arms use the reference-based deployment configuration described in Appendix~\ref{app:deployment_rules}. This search comparison is separate from the module ablations in Table~\ref{tab:ablation-full}.

Full search has the highest reported E-DAIC binary Macro-F1, at 0.6937 versus 0.6814 for the next-best arm. The random-predictor arm has higher PHQ-8 CCC, and staged search has higher Mental Health Macro-F1. Some restricted searches also exceed full search on the MPDD-2025 ternary and quinary rows. Broader search therefore does not consistently yield higher test scores in these comparisons. Realized effort also differs: E-DAIC binary full search evaluates 320 candidates, compared with 24 under a fixed recipe and 15 under a fixed predictor. These results reflect differences in both search coverage and realized cost.

\begin{table}[!htbp]
\caption{Search and selection variants. Ordinary arms retain one seed-0 run per task; random-space arms average three seeds. MPDD-2025 rows then average the 1\,s and 5\,s Elder tasks. Metrics are Macro-F1 for binary/status classification, CCC for PHQ-8, and the official composite for MPDD-2025. MPDD random recipes are inapplicable because the proposed recipe space already exhausts the 12 eligible combinations. Best and second-best distinct values are bold and underlined. For each protocol, configuration, task, and seed, the latest logged run is used before seed and window aggregation.}
\label{tab:search-axes}
\centering
\begingroup
\mrAppendixDenseTableStyle
\begin{tabular}{>{\raggedright\arraybackslash}p{.22\linewidth}*{7}{>{\centering\arraybackslash}p{\dimexpr(.78\linewidth-16\tabcolsep)/7\relax}}}
\toprule
\rowcolor{mrAppHeader}
\textbf{Task} & \cellcolor{mrAppOurs}\textbf{Full} & \textbf{Staged} & \shortstack{\textbf{Fixed}\\\textbf{recipe}} & \shortstack{\textbf{Fixed}\\\textbf{predictor}} & \textbf{1-SE} & \shortstack{\textbf{Random}\\\textbf{predictor}} & \shortstack{\textbf{Random}\\\textbf{recipe}} \\
\midrule
\rowcolor{mrAppBand}
E-DAIC binary & \cellcolor{mrAppOurs}\textbf{0.6937} & 0.4991 & \underline{0.6814} & 0.4593 & 0.4991 & 0.5564 & 0.5735 \\
E-DAIC PHQ-8 & \cellcolor{mrAppOurs}\underline{0.3201} & \underline{0.3201} & \underline{0.3201} & \underline{0.3201} & \underline{0.3201} & \textbf{0.3307} & \underline{0.3201} \\
\rowcolor{mrAppBand}
Mental Health & \cellcolor{mrAppOurs}\underline{0.9630} & \textbf{0.9711} & 0.9527 & 0.9305 & 0.9535 & 0.9527 & 0.9539 \\
MPDD-25 binary & \cellcolor{mrAppOurs}0.7212 & \underline{0.7543} & 0.7212 & 0.7212 & \textbf{0.7565} & 0.7266 & --- \\
\rowcolor{mrAppBand}
MPDD-25 ternary & \cellcolor{mrAppOurs}0.6390 & 0.6443 & \textbf{0.6922} & 0.6482 & 0.6183 & \underline{0.6619} & --- \\
MPDD-25 quinary & \cellcolor{mrAppOurs}0.4115 & 0.4066 & \underline{0.4514} & 0.4317 & 0.4169 & \textbf{0.4624} & --- \\
\bottomrule
\end{tabular}
\endgroup
\end{table}

\paragraph{A recorded recipe--predictor interaction.}
\phantomsection
\label{app:recipe_interaction}
Let $f(r,c)=\widehat{S}_{\mu}((r,c);\mathcal{D})$ denote the development score of recipe $r$ with predictor configuration $c$.
For alternatives $r'$ and $c'$, recipe--predictor interaction can be expressed as
\begin{equation}
\kappa(r,r';c,c')
=\bigl[f(r',c')-f(r,c')\bigr]-\bigl[f(r',c)-f(r,c)\bigr].
\label{eq:target-coupling}
\end{equation}
The contrast $\kappa$ measures how the effect of replacing $r$ with $r'$ changes between configurations $c$ and $c'$.
A nonzero value describes a nonseparable score interaction.
CPE evaluates candidate combinations directly; this contrast characterizes the interaction illustrated below.

Four saved E-DAIC evaluations illustrate a recipe--predictor interaction in initial search round 0 of run 2414862. Each candidate was fitted on 163 training participants and evaluated on the same 56 development participants. Both recipe families use linear SVC with $C=0.3$; within each family, only the selected feature count $k$ changes. Increasing $k$ has opposite effects across the two recipes, as shown below.

\begin{center}
\begingroup
\mrAppendixTableStyle
\begin{tabular}{>{\raggedright\arraybackslash}p{.61\linewidth}*{3}{>{\centering\arraybackslash}p{\dimexpr(.39\linewidth-8\tabcolsep)/3\relax}}}
\toprule
\rowcolor{mrAppHeader}
\textbf{Recipe} & $k=32$ & $k=128$ & \textbf{Change} \\
\midrule
\rowcolor{mrAppBand}
BoAW MFCC + BoVW OpenFace + TF-IDF (meanstd) & 0.613793 & 0.536643 & $-0.077150$ \\
eGeMAPS + TF-IDF (meanstd\_delta) & 0.613793 & 0.660606 & $+0.046813$ \\
\bottomrule
\end{tabular}
\endgroup
\end{center}

Raising $k$ from 32 to 128 lowers the first recipe's Macro-F1 by 0.077 and raises the second's by 0.047. The recipes tie at $k=32$, so a staged procedure screening at that setting cannot distinguish them by Macro-F1. If it retains the first recipe, subsequent tuning over these two $k$ values yields 0.6138; the best recorded pair instead scores 0.6606. At $k=32$, the candidates also disagree on 14 of 56 participants despite identical Macro-F1 and confusion matrices. Stored predictions expose these participant-level differences to review, beyond what aggregate scores reveal.

\clearpage
\section{RSI Experiments}
\label{app:revision_case}

\FloatBarrier
\subsection{Research-Memory Ablation}
\label{app:rsi_trajectories}

\textbf{Research memory.} Table~\ref{tab:research-memory-ablation} reports the separate comparison of memory disabled and evolving memory. Each arm uses at most three revision rounds, an evenly divided revision reserve, and the incumbent acceptance rule. There are 27 task--seed pairs and 54 completed runs. Three runs discarded after a response-parser failure or its strict-parser configuration are excluded in favor of their recorded replacements. Final development scores are identical in 23 pairs. Of these, ten have different test scores. The test comparison contains seven gains, thirteen ties, and seven losses for evolving memory. Its two-sided sign-test $p$-value over the fourteen non-tied pairs is 1.000. The E-DAIC PHQ seed-0 pair starts with different recorded incumbent scores, and actual stopping rounds can differ across arms.

\begin{table}[!htbp]
\caption{Research-memory ablation. Test scores and LLM calls are means over three seeds. $\Delta$ is evolving minus off on the indicated task metric.}
\label{tab:research-memory-ablation}
\centering
\begingroup
\mrAppendixTableStyle
\begin{tabular}{>{\raggedright\arraybackslash}p{.34\linewidth}*{3}{>{\centering\arraybackslash}p{\dimexpr(.40\linewidth-10\tabcolsep)/3\relax}}>{\centering\arraybackslash}p{.26\linewidth}}
\toprule
\rowcolor{mrAppHeader}
Task & Memory off & \cellcolor{mrAppOurs}Evolving & $\Delta$ & Calls (off $\to$ evolving) \\
\midrule
\rowcolor{mrAppBand}
E-DAIC binary (F1) & 0.4927 & \cellcolor{mrAppOurs}0.5099 & +0.0172 & 19.7 $\to$ 33.7 \\
E-DAIC PHQ-8 (CCC) & 0.3142 & \cellcolor{mrAppOurs}0.3311 & +0.0168 & 28.7 $\to$ 29.3 \\
\rowcolor{mrAppBand}
Mental Health (F1) & 0.9647 & \cellcolor{mrAppOurs}0.9582 & -0.0065 & 27.0 $\to$ 33.0 \\
MPDD-25 Elder 1\,s binary & 0.7802 & \cellcolor{mrAppOurs}0.7886 & +0.0084 & 22.0 $\to$ 18.0 \\
\rowcolor{mrAppBand}
MPDD-25 Elder 1\,s ternary & 0.6557 & \cellcolor{mrAppOurs}0.6335 & -0.0222 & 18.0 $\to$ 39.7 \\
MPDD-25 Elder 1\,s quinary & 0.4164 & \cellcolor{mrAppOurs}0.4164 & +0.0000 & 22.0 $\to$ 25.7 \\
\rowcolor{mrAppBand}
MPDD-25 Elder 5\,s binary & 0.7827 & \cellcolor{mrAppOurs}0.7827 & +0.0000 & 19.7 $\to$ 21.3 \\
MPDD-25 Elder 5\,s ternary & 0.6255 & \cellcolor{mrAppOurs}0.6330 & +0.0075 & 25.7 $\to$ 34.0 \\
\rowcolor{mrAppBand}
MPDD-25 Elder 5\,s quinary & 0.4514 & \cellcolor{mrAppOurs}0.4514 & +0.0000 & 17.3 $\to$ 30.0 \\
\bottomrule
\end{tabular}
\endgroup
\end{table}

We analyze the 54 completed runs from the research-memory comparison: nine tasks, three seeds, and memory disabled or evolving. Runs allow up to three review rounds, and use an evenly divided revision reserve. Figure~\ref{fig:appendix-rsi-trajectories} shows every retained development-score trajectory. The horizontal axis counts recorded steps that evaluated additional candidates. It starts at initialization and ends at the final available record. Original review-round identifiers and intermediate test scores are absent from the exported summaries, so these curves describe recorded development progress.

\begin{figure}[!htbp]
\centering
\includegraphics[width=\linewidth]{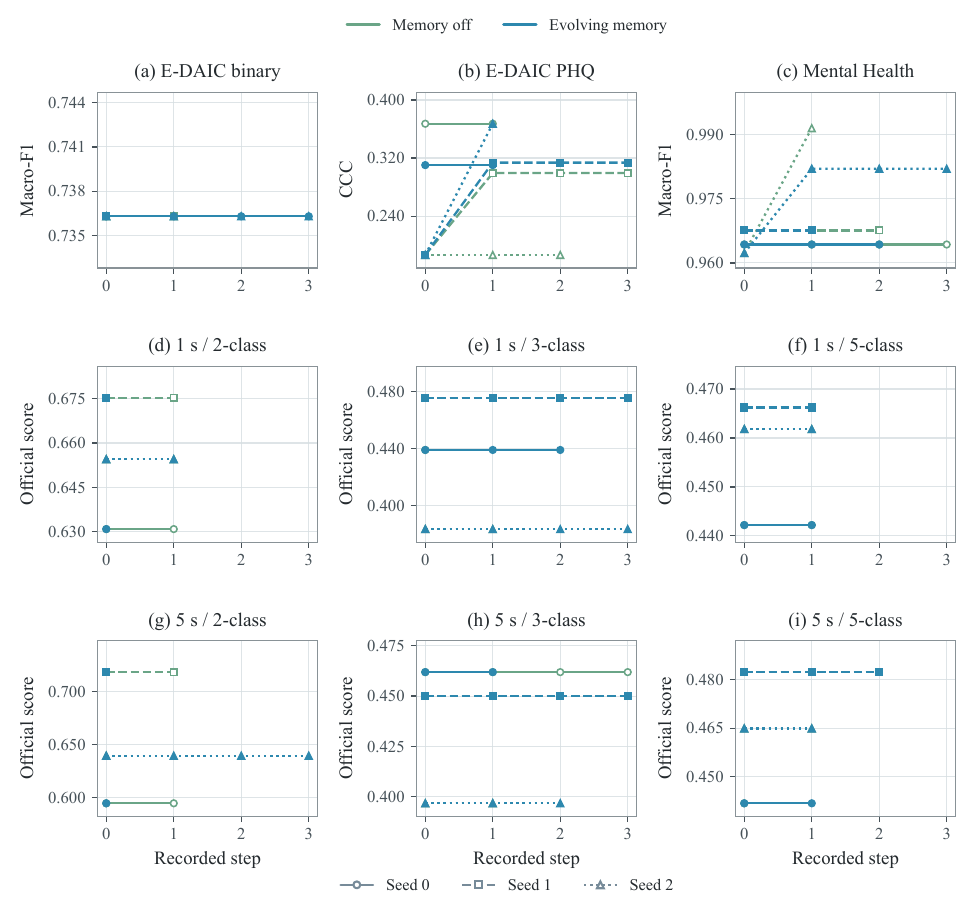}
\caption{Retained development scores across recorded revision steps in the separate research-memory experiment. Colors distinguish memory settings, and line styles and markers identify seeds. Each panel includes all six task-specific runs. Scores are Macro-F1 for E-DAIC binary and Mental Health, CCC for E-DAIC PHQ, and the official composite for MPDD-2025 Elder. Curves stop at their last record, and overlapping trajectories retain their measured values.}
\label{fig:appendix-rsi-trajectories}
\end{figure}

Of the 54 runs, 45 evaluated additional candidates after initialization, and five improved the retained development score. All five gains occurred at the first recorded revision step: three on E-DAIC PHQ severity estimation and two on Mental Health classification. No retained-score increase occurred in the E-DAIC binary or MPDD-2025 classification trajectories. Gains were thus confined to two tasks in these runs; Table~\ref{tab:research-memory-ablation} reports their final test outcomes.\paragraph{Failure modes.}
\phantomsection\label{app:failure_modes}
The saved records reveal three limits of the improvement process. We distinguish an unchanged retained predictor, ambiguity in the evaluation signal, and an unsuccessful memory update because they require different responses.

\paragraph{Additional exploration without a retained gain.}
Across both memory settings, 28 later recorded steps evaluated 586 additional candidates without increasing the retained development score (Table~\ref{tab:recorded_revision_counts}). Extra evaluations therefore did not consistently improve depression classification or severity estimation. Candidate-level scores and acceptance margins would be needed to distinguish weak proposals from useful gains rejected by verification. These counts describe observed stagnation in the retained score.

\begin{table}[!htbp]
\caption{Candidate evaluations and retained-score gains by recorded revision step. Each run contributes only to steps present in its summary. The step counts cover the 54-run research-memory experiment.}
\label{tab:recorded_revision_counts}
\centering
\begingroup
\mrAppendixTableStyle
\begin{tabular}{>{\raggedright\arraybackslash}p{.25\linewidth}*{4}{>{\centering\arraybackslash}p{\dimexpr(.75\linewidth-10\tabcolsep)/4\relax}}}
\toprule
\rowcolor{mrAppHeader}
\textbf{Memory} & \textbf{Step} & \textbf{Runs} & \shortstack{\textbf{New}\\\textbf{evaluations}} & \shortstack{\textbf{Retained-score}\\\textbf{gains}} \\
\midrule
\rowcolor{mrAppBand}
Off & 1 & 23 & 751 & 2 \\
Off & 2 & 5 & 107 & 0 \\
\rowcolor{mrAppBand}
Off & 3 & 3 & 52 & 0 \\
\midrule
Evolving & 1 & 22 & 727 & 3 \\
\rowcolor{mrAppBand}
Evolving & 2 & 12 & 270 & 0 \\
Evolving & 3 & 8 & 157 & 0 \\
\bottomrule
\end{tabular}
\endgroup
\end{table}

\paragraph{Aggregate scores conceal differences between participants.}
In the E-DAIC example in Appendix~\ref{app:recipe_interaction}, two candidates have equal development Macro-F1 and confusion matrices but disagree on 14 of 56 participants. Their aggregate scores therefore conceal differences in which participants are misclassified. The memory comparison contains 23 task--seed pairs with equal final development scores, ten of which have different test scores. These observations motivate inspecting saved predictions alongside aggregate development metrics. The separate test comparison shows mixed memory benefits: seven gains, thirteen ties, and seven losses.

\paragraph{Unsuccessful memory updates.}
The evolving-memory summaries record 59 successful updates and two failed updates. The failures occur in seed 2 of MPDD-2025 Elder 1\,s quinary and 5\,s binary classification, respectively. These runs still produced final predictions, but the summaries do not identify the update-error causes. The implementation retains the previous bank when update validation fails, as described in Appendix~\ref{app:research_memory}. This separates failure to update research experience from failure to produce a depression prediction.

\FloatBarrier
\subsection{Eight-Round RSI Analysis}
\label{app:rsi-eight-round}

\FloatBarrier
\subsubsection{Study Design and Protocol Amendment}
\label{app:rsi_study_design}

This separate study contains 27 runs: E-DAIC binary classification, E-DAIC PHQ-8 estimation, and MPDD-2025 Elder 1\,s ternary classification, each with three seeds and three arms (review with evolving memory, review without memory, and random revision). The amended protocol requests revisions through eight rounds; its timing and the original stopping condition are reported below.

\paragraph{Run budget and deployment.}
Runs use up to 256 initial candidate evaluations and a reserve of 256 evaluations for eight revisions. The reserve releases 32 evaluations per round with carry-over. The review arms differ in whether they update research memory. Random revision samples four recipe or predictor configurations without replacement, choosing the axis with equal probability while both remain available. All arms use the incumbent promotion rule and terminal reliability gate. Final deployment uses up to five of the simplest pipelines within one subject-bootstrap standard error of the selected candidate; classification aggregates their votes. The arms share budget caps, while their realized evaluation counts differ.

\paragraph{Protocol amendment and original stopping condition.}
The original protocol ended review on Chair acceptance or an unusable revision response. The review arms completed a median of one round with new candidate evaluations. Random revision completed eight such rounds on E-DAIC and five on MPDD-2025, where its drawable space was exhausted. The amendment was written after 19 of the 27 original runs had finished. Round counts, Chair verdicts, and logged development and test scores had been observed; the alignment, lexical, activation, and per-round replay endpoints had not yet been computed.

The amendment reran all 18 review-enabled configurations. While budget remained, a nonrevision verdict received one explicit request for a revision; a round that still produced no revision added no candidates, and the loop continued up to round eight. Promotion was unchanged. The nine random-revision runs were reused across conditions. Under the original stopping condition, P1 could be estimated for only one run per review arm; the P2 contrast against random revision was $-0.0122$ [$-0.0346$, $0.0003$]. Neither condition met the preregistered directional-evolution criterion.

\clearpage
\subsubsection{Measures and Primary Results}
\label{app:rsi_measures}

\paragraph{Measurement definitions.}
L1 averages signed Spearman correlations between a fitted pipeline's continuous output and nine prespecified E-DAIC markers: pitch variability, loudness mean and variability, voiced-frame fraction, AU12 and AU06 intensity, head movement, first-person pronoun rate, and absolutist-word rate. The protocol assigns negative signs to the first seven and positive signs to the two lexical markers. Pipelines are refitted on the 163 training participants, and alignment is measured on the 112 development and test participants without using their labels. The incumbent series supplies P2; the round's best new candidate is measured separately. Correlations with training labels provide a descriptive reference, averaging 0.0721 for binary status and 0.0724 for PHQ-8.

L2 is a candidate's fraction of nonempty source slots belonging to eGeMAPS, OpenFace, or TF--IDF, averaged over newly evaluated candidates in the round. Its per-run slope supplies P1. This candidate-weighted measure differs from the source-slot composition in Figure~\ref{fig:rsi-round-evolution}. L3 counts clinical-mechanism keywords relative to all clinical and engineering keyword hits in lesson versions or review minutes. Units with no hits are excluded; task words such as depression, PHQ, and severity belong to neither list. Figure~\ref{fig:rsi-source-memory-endpoints}(a--d) reports these measures.

Figure~\ref{fig:rsi-round-evolution} uses only the six E-DAIC runs per arm and averages the source-family proportions among each round's newly evaluated candidates. Warm-colored families comprise eGeMAPS prosody/voice quality, OpenFace facial actions/pose/gaze, and TF-IDF lexical content; cool-colored families comprise MFCCs, codebooks, and CNN embeddings. Round 0 denotes initialization. Runs contribute only when that round contains new candidates. The black curves sum the three warm-colored source-family shares. These source-slot proportions describe search composition.

The preregistered primary tests use per-run slopes of interpretable-source share (P1) and incumbent alignment with the prespecified markers (P2). The evolving-minus-random slope contrasts are $0.0224$ [$-0.0027$, $0.0466$] and $-0.0003$ [$-0.0031$, $0.0023$], respectively. The brackets are 95\% intervals from 10,000 bootstrap resamples of runs within each arm; each primary contrast uses six E-DAIC runs per arm. Neither primary endpoint meets the preregistered criterion for directional evolution. The plotted case is the evolving E-DAIC run with the most memory revisions and retirements under the prespecified selection rule: binary classification, seed 1. Its callouts excerpt the first, middle, and last versions of its most-revised entry, with omissions marked by ellipses.

For the secondary lexical measure, the first author reviewed AI prelabels and confirmed all 30 sampled lesson versions as engineering under the narrow definition. Agreement with the keyword rule was 30/30; Cohen's $\kappa$ is undefined because both assignments contain only one category. The reviewer already knew the keyword result, and round information remained in the text, departing from the planned blind check. The broader post-hoc categorization was excluded from this check.

\paragraph{Secondary outcomes.}
For the best newly evaluated candidate, evolving memory increases the marker-alignment slope relative to review without memory by 0.0049 [0.0003, 0.0092], whereas its contrast with random revision is 0.0045 [$-0.0020$, 0.0113]. The best-new candidate's source-share contrast against random revision is 0.0034 [$-0.0377$, 0.0418]. The clinical-keyword share in review minutes declines faster with evolving memory than without it, with slope contrast $-0.0007$ [$-0.0014$, $-0.0001$]. These secondary intervals are descriptive and unadjusted for multiplicity; they do not change the primary decision criterion.

\begin{figure}[!htbp]
\centering
\includegraphics[width=\linewidth]{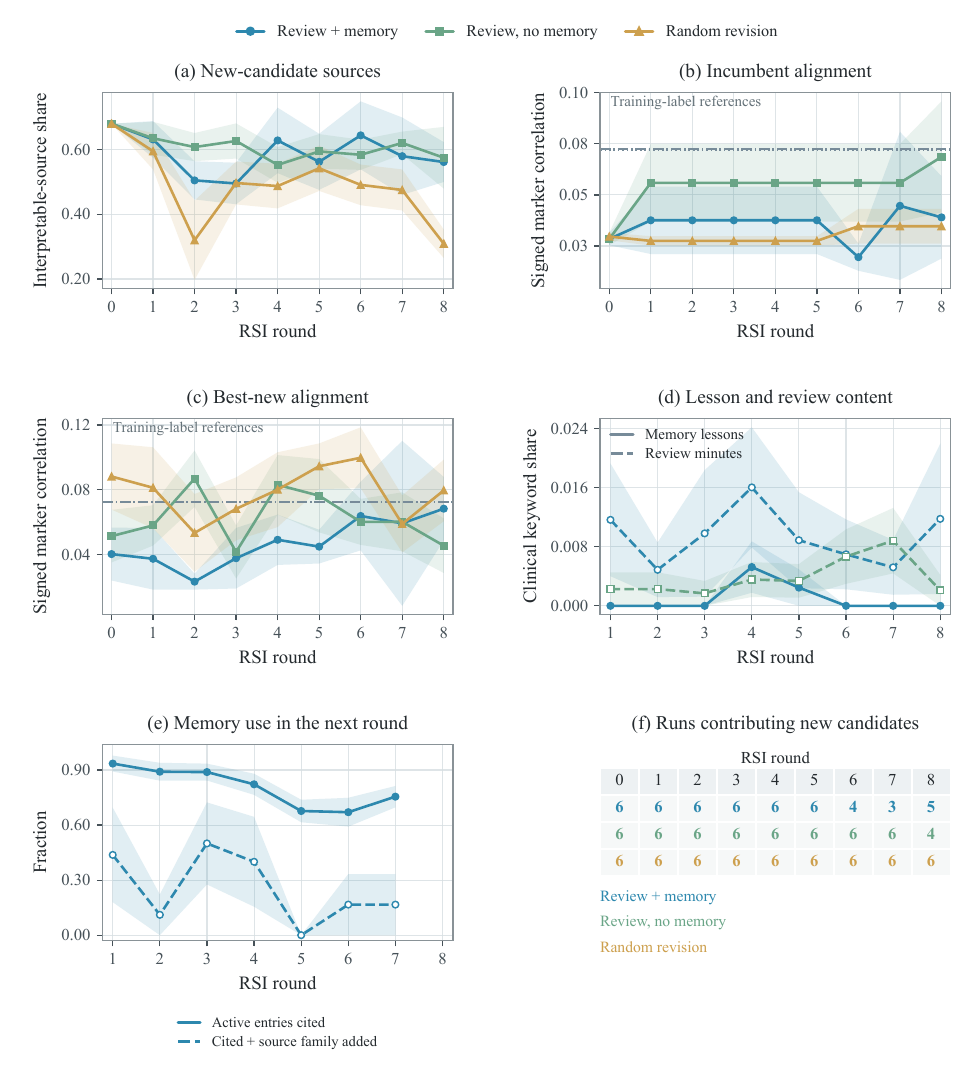}
\caption{Source and memory measures in the amended eight-round study. (a-c) Six E-DAIC runs per arm; curves average available records with $\pm$1 standard-error bands. Panel (a) is candidate-weighted L2; it differs from the source-slot proportions in the main figure. Gray lines in (b,c) are binary and PHQ-8 training-label references. (d) Lessons and review minutes from nine runs per review arm. (e) Citation rates use nine memory-enabled runs; action matching uses the six E-DAIC runs with applicable entries. (f) Run counts for panel (a), with rows following the arm legend. Counts describe available task-run records, not additional independent seeds.}
\label{fig:rsi-source-memory-endpoints}
\end{figure}

\FloatBarrier
\subsubsection{Test Trajectories and Memory Use}
\label{app:rsi_test_memory}

L4 replays each round's deployment on its task's test set, with selection based on development records. Figure~\ref{fig:rsi-test-memory-trajectories} and Table~\ref{tab:rsi-secondary-test} report task-specific trajectories and slope contrasts. E-DAIC PHQ-8 favors evolving memory over review without memory in this comparison. The binary and MPDD-2025 contrasts include zero, and all evolving-minus-random intervals include zero.

\begin{figure}[!htbp]
\centering
\includegraphics[width=\linewidth]{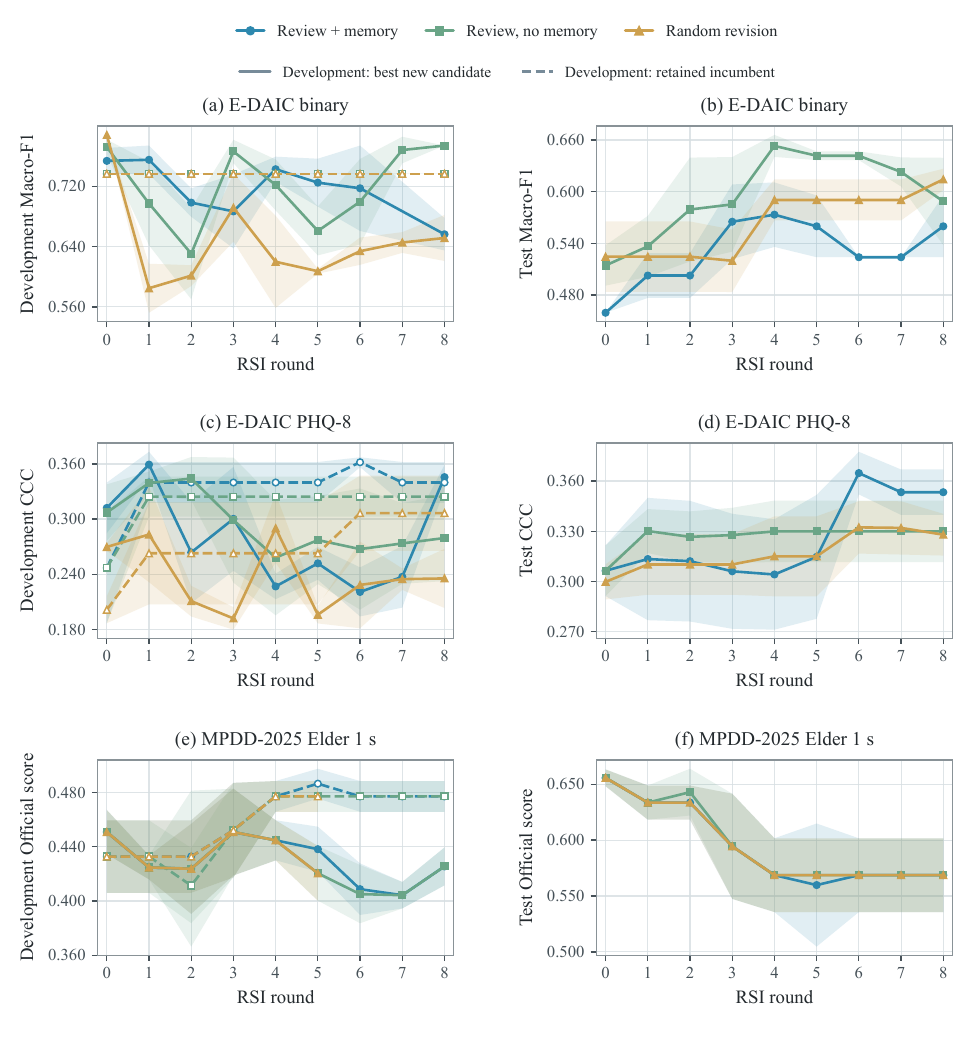}
\caption{Development and test trajectories under the amended eight-round protocol. Rows show E-DAIC binary classification, E-DAIC PHQ-8, and MPDD-2025 Elder 1 s ternary classification. Left: best newly evaluated candidate (solid) and retained incumbent (dashed), using available development records. Right: the locked deployment at each round. After a run ends, its final deployment score is carried forward through round eight; unrecorded intermediate rounds remain excluded from that round's mean. Each task and arm has three runs; intermediate test means use one to three available records. Bands show $\pm$1 standard error when at least two records are available. Each task retains its own metric scale.}
\label{fig:rsi-test-memory-trajectories}
\label{fig:rsi-exploration-effort}
\end{figure}
\FloatBarrier

An active entry is counted as used when its identifier appears in the next round's minutes or proposer reasoning. A second, E-DAIC-specific proxy checks whether the next round adds the source family named in a cited entry's proposed action. These measures record citation and action matching; they do not establish that an entry caused a revision.

Figure~\ref{fig:rsi-source-memory-endpoints}(e) shows the two memory-use measures. Table~\ref{tab:rsi-memory-use} reports their run-level counts. For the test curves, final deployment scores are carried forward after a run ends; unrecorded intermediate rounds remain excluded from that round's mean.

\begin{table}[!htbp]
\caption{Secondary L4 slope contrasts per round. Each arm has three runs per task. Brackets are unadjusted 95\% run-bootstrap intervals; task metrics are kept separate.}
\label{tab:rsi-secondary-test}
\centering
\begingroup
\mrAppendixTableStyle
\begin{tabular}{>{\raggedright\arraybackslash}p{.35\linewidth}*{2}{>{\centering\arraybackslash}p{\dimexpr(.65\linewidth-6\tabcolsep)/2\relax}}}
\toprule
\rowcolor{mrAppHeader}
\textbf{Task / test metric} & \textbf{Evolving $-$ random} & \textbf{Evolving $-$ no memory} \\
\midrule
\rowcolor{mrAppBand}
E-DAIC binary / Macro-F1 & $+0.0004$ [-0.0176, 0.0185] & $-0.0017$ [-0.0206, 0.0173] \\
E-DAIC PHQ-8 / CCC & $+0.0024$ [-0.0023, 0.0071] & $+0.0045$ [0.0013, 0.0079] \\
\rowcolor{mrAppBand}
MPDD-25 Elder 1\,s ternary / official & $+0.0076$ [-0.0093, 0.0245] & $+0.0002$ [-0.0121, 0.0125] \\
\bottomrule
\end{tabular}
\endgroup
\end{table}

\begin{table}[!htbp]
\caption{Research-memory use in the nine evolving eight-round runs. Operations count entries added, revised and retired over the run. Citation rate is the fraction of active entries whose identifier appears in the next round's minutes or proposals. Action matching counts activated entries whose next action names a source family and whose next round adds that family, over the activated entries that name one; this source-family proxy is computed for E-DAIC only.}
\label{tab:rsi-memory-use}
\centering
\begingroup
\mrAppendixTableStyle
\begin{tabular}{>{\raggedright\arraybackslash}p{.28\linewidth}*{4}{>{\centering\arraybackslash}p{\dimexpr(.72\linewidth-10\tabcolsep)/4\relax}}}
\toprule
\rowcolor{mrAppHeader}
\textbf{Task} & \textbf{Seed} & \shortstack{\textbf{Add / revise}\\\textbf{/ retire}} & \textbf{Citation rate} & \shortstack{\textbf{Action-matched}\\\textbf{use}} \\
\midrule
\rowcolor{mrAppBand}
E-DAIC binary & 0 & 14 / 12 / 2 & 0.73 & 1 / 7 \\
E-DAIC binary & 1 & 11 / 15 / 2 & 0.82 & 6 / 22 \\
\rowcolor{mrAppBand}
E-DAIC binary & 2 & 17 / 10 / 1 & 0.72 & 6 / 21 \\
E-DAIC PHQ-8 & 0 & 14 / 15 / 0 & 0.80 & 1 / 5 \\
\rowcolor{mrAppBand}
E-DAIC PHQ-8 & 1 & 14 / 12 / 1 & 0.66 & 7 / 33 \\
E-DAIC PHQ-8 & 2 & 8 / 6 / 2 & 0.81 & 6 / 14 \\
\rowcolor{mrAppBand}
MPDD-25 Elder 1\,s ternary & 0 & 10 / 8 / 3 & 0.83 & --- \\
MPDD-25 Elder 1\,s ternary & 1 & 11 / 17 / 1 & 0.76 & --- \\
\rowcolor{mrAppBand}
MPDD-25 Elder 1\,s ternary & 2 & 12 / 14 / 4 & 0.86 & --- \\
\bottomrule
\end{tabular}
\endgroup
\end{table}

\clearpage
\subsubsection{Exploration and Execution Diagnostics}
\label{app:rsi_execution_diagnostics}

Table~\ref{tab:rsi-exploration-budget} reports additional evaluation effort, while Figure~\ref{fig:rsi-exploration-effort}(a,c,e) shows retained development progress. Newly explored candidates can vary while the incumbent stays unchanged. Across the original and amended protocols, all 45 distinct runs with saved test replays reproduce their recorded final test score. Five amended E-DAIC binary runs each contain one failed review: evolving memory at seeds 0--2 and review without memory at seeds 1 and 2. Their available rounds are analyzed as recorded. The selected binary seed-1 case contains 11 memory additions, 15 revisions, and two retirements; the illustrated entry concerns promotion margins and bootstrap uncertainty.

Figure~\ref{fig:rsi-source-memory-endpoints}(f) gives the number of E-DAIC runs contributing new candidates at each round. The six runs in each arm comprise two tasks with three seeds each. Later-round means can use fewer runs, and the mean curves should be read alongside their sampling bands and contribution counts. Table~\ref{tab:rsi-diagnostics} lists review and exploration counts for all 27 amended runs.

\begin{table}[!htbp]
\caption{Review and exploration counts for every amended eight-round run. Reviews counts scheduled review rounds, including failures. Requested / declined counts explicit follow-up requests for a revision and requests that still produced none. New rounds counts revision rounds with newly evaluated candidates. Random revision exhausts its drawable MPDD-2025 space after five rounds.}
\label{tab:rsi-diagnostics}
\centering
\begingroup
\mrAppendixDenseTableStyle
\begin{tabular}{>{\raggedright\arraybackslash}p{.24\linewidth}>{\raggedright\arraybackslash}p{.12\linewidth}*{6}{>{\centering\arraybackslash}p{\dimexpr(.64\linewidth-16\tabcolsep)/6\relax}}}
\toprule
\rowcolor{mrAppHeader}
Task & Arm & Seed & Reviews & Failed & \shortstack{Requested\\/ declined} & \shortstack{New\\rounds} & Candidates \\
\midrule
\rowcolor{mrAppBand}
E-DAIC binary & Memory & 0 & 8 & 1 & 3 / 1 & 6 & 448 \\
E-DAIC binary & Memory & 1 & 8 & 1 & 1 / 0 & 7 & 512 \\
\rowcolor{mrAppBand}
E-DAIC binary & Memory & 2 & 8 & 1 & 2 / 1 & 6 & 512 \\
E-DAIC binary & No memory & 0 & 8 & 0 & 6 / 0 & 8 & 512 \\
\rowcolor{mrAppBand}
E-DAIC binary & No memory & 1 & 8 & 1 & 4 / 0 & 7 & 480 \\
E-DAIC binary & No memory & 2 & 8 & 1 & 5 / 0 & 7 & 480 \\
\rowcolor{mrAppBand}
E-DAIC binary & Random & 0 & --- & --- & --- & 8 & 512 \\
E-DAIC binary & Random & 1 & --- & --- & --- & 8 & 512 \\
\rowcolor{mrAppBand}
E-DAIC binary & Random & 2 & --- & --- & --- & 8 & 512 \\
E-DAIC PHQ-8 & Memory & 0 & 8 & 0 & 3 / 0 & 8 & 512 \\
\rowcolor{mrAppBand}
E-DAIC PHQ-8 & Memory & 1 & 8 & 0 & 4 / 1 & 7 & 512 \\
E-DAIC PHQ-8 & Memory & 2 & 8 & 0 & 4 / 0 & 8 & 512 \\
\rowcolor{mrAppBand}
E-DAIC PHQ-8 & No memory & 0 & 8 & 0 & 2 / 0 & 8 & 512 \\
E-DAIC PHQ-8 & No memory & 1 & 8 & 0 & 3 / 0 & 8 & 512 \\
\rowcolor{mrAppBand}
E-DAIC PHQ-8 & No memory & 2 & 8 & 0 & 0 / 0 & 8 & 512 \\
E-DAIC PHQ-8 & Random & 0 & --- & --- & --- & 8 & 512 \\
\rowcolor{mrAppBand}
E-DAIC PHQ-8 & Random & 1 & --- & --- & --- & 8 & 512 \\
E-DAIC PHQ-8 & Random & 2 & --- & --- & --- & 8 & 512 \\
\rowcolor{mrAppBand}
MPDD-25 Elder 1\,s ternary & Memory & 0 & 8 & 0 & 3 / 0 & 8 & 512 \\
MPDD-25 Elder 1\,s ternary & Memory & 1 & 8 & 0 & 1 / 0 & 8 & 512 \\
\rowcolor{mrAppBand}
MPDD-25 Elder 1\,s ternary & Memory & 2 & 8 & 0 & 2 / 1 & 7 & 512 \\
MPDD-25 Elder 1\,s ternary & No memory & 0 & 8 & 0 & 3 / 0 & 8 & 512 \\
\rowcolor{mrAppBand}
MPDD-25 Elder 1\,s ternary & No memory & 1 & 8 & 0 & 7 / 1 & 7 & 512 \\
MPDD-25 Elder 1\,s ternary & No memory & 2 & 8 & 0 & 5 / 0 & 8 & 512 \\
\rowcolor{mrAppBand}
MPDD-25 Elder 1\,s ternary & Random & 0 & --- & --- & --- & 5 & 416 \\
MPDD-25 Elder 1\,s ternary & Random & 1 & --- & --- & --- & 5 & 416 \\
\rowcolor{mrAppBand}
MPDD-25 Elder 1\,s ternary & Random & 2 & --- & --- & --- & 5 & 416 \\
\bottomrule
\end{tabular}
\endgroup
\end{table}

\clearpage
\section{Efficiency Analysis}
\label{app:efficiency}

\FloatBarrier
\subsection{Performance--Cost Comparison}
\label{app:baseline-cost}

Figure~\ref{fig:efficiency-analysis}(a) compares E-DAIC scores with LLM calls. The \method{} coordinate is the mean over its five seed runs (Macro-F1 0.6768, 11.4 calls per run); Gemini, MDAgents, and DepressionAgent use 56, 544, and 268 calls for the 56 test participants, respectively.

\begin{figure}[!htbp]
\centering
\includegraphics[width=\linewidth]{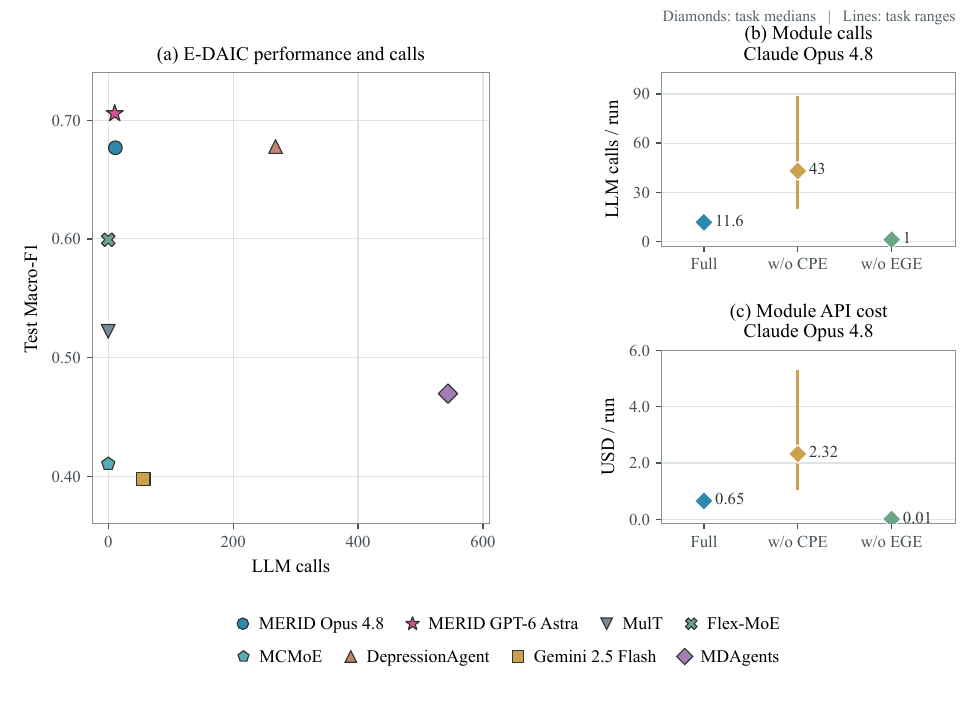}
\caption{LLM efficiency. (a) E-DAIC test Macro-F1 versus recorded calls: \method{} uses the means over its five seed runs; inference baselines cover 56 test participants. DepressionAgent uses its scored 268-call run. Non-LLM methods appear at zero calls. (b,c) Diamonds show medians of nine task-level means from the module ablations, with full \method{} from the five-seed runs of Table~\ref{tab:main-results}; vertical lines show their minimum-to-maximum ranges, not uncertainty intervals. Local fitting and feature extraction are outside this accounting.}
\label{fig:efficiency-analysis}
\label{fig:efficiency-calls}
\end{figure}

Figure~\ref{fig:score-cost}(b,c) plots \method{}'s five-seed mean scores from Table~\ref{tab:main-results} and the baselines' reported scores. MPDD-2026 averages the Elder and Young track scores, each defined as the mean of binary Macro-F1, ternary Macro-F1, and PHQ-9 CCC. E-DAIC uses binary Macro-F1 on its 56 test participants. MulT, Flex-MoE, and MCMoE appear at zero LLM cost; their local computation is outside this accounting.

\method{}'s horizontal coordinate records pipeline-development LLM usage averaged over the same five seed runs: \$0.6035 per run for E-DAIC and \$1.0201 per seed for MPDD-2026, where a seed's cost sums its Elder and Young runs. The fitted predictors require no per-participant LLM call. Inference baselines instead incur LLM costs for the plotted test cohort. Local fitting, feature extraction, and non-LLM inference costs are excluded. MPDD-2026 Gemini 2.5 Flash uses 35,975 input and 5,731 output tokens; MDAgents uses 55,143 and 4,323. The source plots apply their recorded rates of \$0.30/\$2.50 and \$5/\$30 per million input/output tokens, respectively, yielding \$0.02512 and \$0.405405.

The E-DAIC DepressionAgent score (0.6778) is from the scored 268-call run. Its \$5.140055 cost comes from a separate accounting rerun of the same model and prompts on the same 56 participants: 263 calls, 482,611 input tokens, and 109,080 output tokens, accounted at \$5/\$25 per million input/output tokens. That rerun agrees on 53 of 56 predictions and has Macro-F1 0.6725. The plotted score and cost thus come from separate runs. E-DAIC Gemini and MDAgents are omitted from the dollar-cost panels because their corresponding token counts are unavailable.

For Kaggle Video, \method{}'s development record contains 12 calls, 93,405 input tokens, and 9,137 output tokens, giving \$0.69545 at the recorded \$5/\$25 rates per million input/output tokens. The 240-clip baseline records contain 20,561 output tokens for Gemini 2.5 Flash and 69,603 for MedAgent-Pro. Their input-token totals are unavailable, so total dollar costs are not inferred from these output counts.

\FloatBarrier
\subsection{Module-Level LLM Costs}
\label{app:module_costs}

\textbf{Cost accounting.} Call counts and USD costs are averaged over the logged runs within each task and configuration: the five main-benchmark runs for full \method{} and the available logged runs for each removal. The reported medians summarize nine task-level means. For the removals, the accounting sample excludes recovered runs without cost-dump rows and therefore differs from the complete performance sample. Costs cover LLM API usage and exclude local model fitting. Across tasks, full \method{} averages 10.8--12.4 calls and \$0.60--\$0.70 per run. The corresponding ranges are 20--89 calls and \$1.06--\$5.29 without CPE, and 0--2 calls and \$0.00--\$0.16 without EGE.

Table~\ref{tab:module-costs} reports the task-level means and logged run counts; Figure~\ref{fig:efficiency-analysis}(b,c) shows their medians and ranges.

\begin{table}[!htbp]
\caption{LLM usage per run of the module ablations: mean LLM calls / USD over the logged runs of each task and configuration (Claude Opus 4.8 at \$5/\$25 per million input/output tokens). Local feature extraction and model fitting are excluded. Full \method{} uses the five runs of Table~\ref{tab:main-results}; for the removals, recovered runs without a cost record are not included, so their sample differs from the performance sample of Table~\ref{tab:ablation-full}. Runs gives the number of logged runs per configuration (full / w/o CPE / w/o EGE).}
\label{tab:module-costs}
\centering
\begingroup
\mrAppendixDenseTableStyle
\begin{tabular}{>{\raggedright\arraybackslash}p{.3\linewidth}*{4}{>{\centering\arraybackslash}p{\dimexpr(.7\linewidth-10\tabcolsep)/4\relax}}}
\toprule
\rowcolor{mrAppHeader}
\textbf{Task} & \cellcolor{mrAppOurs}\textbf{\method{} (full)} & \textbf{w/o CPE} & \textbf{w/o EGE} & \textbf{Runs} \\
\midrule
\rowcolor{mrAppBand}
E-DAIC binary & \cellcolor{mrAppOurs}11.4 / \$0.60 & 83.0 / \$4.67 & 2.0 / \$0.16 & 5 / 1 / 3 \\
E-DAIC PHQ-8 & \cellcolor{mrAppOurs}11.4 / \$0.61 & 41.0 / \$2.12 & 1.0 / \$0.06 & 5 / 1 / 3 \\
\rowcolor{mrAppBand}
Mental Health & \cellcolor{mrAppOurs}10.8 / \$0.64 & 89.0 / \$5.09 & 0.0 / \$0.00 & 5 / 1 / 3 \\
MPDD-25 Elder 1\,s binary & \cellcolor{mrAppOurs}12.2 / \$0.65 & 89.0 / \$5.29 & 1.0 / \$0.01 & 5 / 1 / 1 \\
\rowcolor{mrAppBand}
MPDD-25 Elder 1\,s ternary & \cellcolor{mrAppOurs}12.4 / \$0.70 & 20.0 / \$1.06 & 1.0 / \$0.01 & 5 / 1 / 1 \\
MPDD-25 Elder 1\,s quinary & \cellcolor{mrAppOurs}12.2 / \$0.69 & 33.0 / \$1.89 & 1.0 / \$0.01 & 5 / 1 / 1 \\
\rowcolor{mrAppBand}
MPDD-25 Elder 5\,s binary & \cellcolor{mrAppOurs}12.4 / \$0.67 & 43.0 / \$2.32 & 1.0 / \$0.01 & 5 / 1 / 1 \\
MPDD-25 Elder 5\,s ternary & \cellcolor{mrAppOurs}11.6 / \$0.62 & 21.0 / \$1.09 & 1.0 / \$0.01 & 5 / 1 / 1 \\
\rowcolor{mrAppBand}
MPDD-25 Elder 5\,s quinary & \cellcolor{mrAppOurs}11.6 / \$0.67 & 44.0 / \$2.49 & 1.0 / \$0.01 & 5 / 1 / 1 \\
\midrule
\rowcolor{mrAppBand}
Median over tasks & \cellcolor{mrAppOurs}11.6 / \$0.65 & 43.0 / \$2.32 & 1.0 / \$0.01 & --- \\
\bottomrule
\end{tabular}
\endgroup
\end{table}

\clearpage
\subsection{Search Budget and Realized Effort}
\label{app:budget_effort}

The candidate cap counts initialization and revision evaluations, as defined in Appendix~\ref{app:rsi_configuration}. Tables~\ref{tab:budget-effort} and~\ref{tab:rsi-exploration-budget} compare these caps with realized counts in the separate research-memory and eight-round studies. Table~\ref{tab:recorded_revision_counts} gives the recorded-step counts for the former study; Appendix~\ref{app:rsi_execution_diagnostics} relates the latter counts to review outcomes.

\begin{table}[!htbp]
\caption{Realized candidate evaluations in the 54-run research-memory comparison. Each arm covers nine tasks and three seeds. Means are over runs; ranges span their totals. Revision evaluations exclude initialization.}
\label{tab:budget-effort}
\centering
\begingroup
\mrAppendixTableStyle
\begin{tabular}{>{\raggedright\arraybackslash}p{.23\linewidth}*{5}{>{\centering\arraybackslash}p{\dimexpr(.77\linewidth-12\tabcolsep)/5\relax}}}
\toprule
\rowcolor{mrAppHeader}
Memory & Runs & Cap & Total mean & Total range & Revision mean \\
\midrule
\rowcolor{mrAppBand}
Memory off & 27 & 320 & 264.4 & 139--320 & 33.7 \\
Evolving memory & 27 & 320 & 272.9 & 139--320 & 42.7 \\
\bottomrule
\end{tabular}
\endgroup
\end{table}

\begin{table}[!htbp]
\centering
\caption{Exploration effort in the amended eight-round study. Each row covers three seeds. The cap includes initialization and revisions. Counts include all recorded candidate evaluations. The total range is the observed minimum--maximum across the three runs.}
\label{tab:rsi-exploration-budget}
\begingroup\mrAppendixTableStyle
\begin{tabular}{>{\raggedright\arraybackslash}p{.30\linewidth}>{\raggedright\arraybackslash}p{.14\linewidth}*{4}{>{\centering\arraybackslash}p{\dimexpr(.56\linewidth-12\tabcolsep)/4\relax}}}
\toprule
\rowcolor{mrAppHeader}
Task & Arm & Cap & Initial mean & Total mean & Total range \\
\midrule
\rowcolor{mrAppBand}
E-DAIC binary & Memory & 512 & 254.7 & 490.7 & 448--512 \\
E-DAIC binary & No memory & 512 & 256.0 & 490.7 & 480--512 \\
\rowcolor{mrAppBand}
E-DAIC binary & Random & 512 & 256.0 & 512.0 & 512--512 \\
E-DAIC PHQ-8 & Memory & 512 & 177.3 & 512.0 & 512--512 \\
\rowcolor{mrAppBand}
E-DAIC PHQ-8 & No memory & 512 & 182.0 & 512.0 & 512--512 \\
E-DAIC PHQ-8 & Random & 512 & 186.7 & 512.0 & 512--512 \\
\rowcolor{mrAppBand}
MPDD-25 Elder 1\,s ternary & Memory & 512 & 256.0 & 512.0 & 512--512 \\
MPDD-25 Elder 1\,s ternary & No memory & 512 & 256.0 & 512.0 & 512--512 \\
\rowcolor{mrAppBand}
MPDD-25 Elder 1\,s ternary & Random & 512 & 256.0 & 416.0 & 416--416 \\
\bottomrule
\end{tabular}\endgroup
\end{table}

\clearpage
\section{Complete Benchmark Results}
\label{app:full-results}

Table~\ref{tab:full-results} gives complete metrics. \method{}'s cells are means over the five seed runs behind Table~\ref{tab:main-results} (Appendix~\ref{app:seed_study}), except Kaggle Video, which is outside that study and reports its retained run. The appendix ranks summarize individual tasks within each dataset using Macro-F1, RMSE, or the official MPDD-2025 task score. Table~\ref{tab:main-results} summarizes its displayed columns, including CCC and MPDD-2025 track aggregates, over the methods reported there.

\begin{table}[p]
\caption{Complete performance metrics, continued over three parts. \method{} reports means over seeds 0--4 (Kaggle Video: retained run); baselines are unchanged. This part reports E-DAIC, MPDD-2026, Mental Health, and Kaggle. Arrows indicate preferred directions. Best and second-best distinct values among the methods shown are bold and underlined, including ties; dashes denote unavailable results. Bal.\ Acc. and W.\ F1 denote balanced accuracy and weighted F1. Avg.\ Rank uses one primary metric per task: Macro-F1 for classification, RMSE for PHQ regression, and the official score for MPDD-2025. Per-task ranks use midranks among reported methods; averages require complete task coverage and exclude track aggregates. MPDD-2026 denotes MPDD-AVG 2026. Its MulT and Flex-MoE PHQ-9 estimates are derived from ternary-class midpoints. Kaggle evaluates acted-emotion proxy labels; \method{} uses actor-disjoint nested cross-validation with review disabled.}
\label{tab:full-results}
\centering
\begingroup
\mrAppendixDenseTableStyle
\begin{tabular}{>{\raggedright\arraybackslash}m{.115\linewidth}>{\raggedright\arraybackslash}m{.06\linewidth}>{\raggedright\arraybackslash}m{.105\linewidth}>{\raggedright\arraybackslash}m{.125\linewidth}>{\centering\arraybackslash}m{.075\linewidth}>{\centering\arraybackslash}m{.07\linewidth}>{\centering\arraybackslash}m{.09\linewidth}>{\centering\arraybackslash}m{.09\linewidth}>{\centering\arraybackslash}m{.09\linewidth}>{\centering\arraybackslash}m{\dimexpr.18\linewidth-20\tabcolsep\relax}}
\toprule
\rowcolor{mrAppHeader}
\textbf{Dataset} & \textbf{Cohort} & \textbf{Task} & \textbf{Metric} & \textbf{Official} & \textbf{MulT} & \textbf{Flex-MoE} & \textbf{Gemini 2.5}\newline\textbf{Flash} & \textbf{MDAgents} & \cellcolor{mrAppOurs}\textbf{Ours} \\
\midrule
\rowcolor{mrAppBand}
\textbf{E-DAIC} & -- & Binary & Accuracy~$\uparrow$ &  & 0.5893 & \underline{0.7143} & 0.6607 & 0.6429 & \cellcolor{mrAppOurs}\textbf{0.7393} \\
 &  &  & Bal. Acc.~$\uparrow$ &  & 0.5226 & \underline{0.5958} & 0.4744 & 0.4947 & \cellcolor{mrAppOurs}\textbf{0.6702} \\
\rowcolor{mrAppBand}
 &  &  & Macro-F1~$\uparrow$ &  & 0.5221 & \underline{0.5993} & 0.3978 & 0.4697 & \cellcolor{mrAppOurs}\textbf{0.6768} \\
 &  &  & W. F1~$\uparrow$ &  & 0.5925 & \underline{0.6836} & 0.5541 & 0.5887 & \cellcolor{mrAppOurs}\textbf{0.7326} \\
\rowcolor{mrAppBand}
 &  &  & $\kappa$~$\uparrow$ &  & 0.0445 & \underline{0.2209} & -0.0683 & -0.0127 & \cellcolor{mrAppOurs}\textbf{0.3554} \\
\cmidrule(l){3-10}
 &  & PHQ-8 & RMSE~$\downarrow$ &  & 6.9902 & 7.0394 & 7.7113 & \underline{6.6201} & \cellcolor{mrAppOurs}\textbf{5.9098} \\
\rowcolor{mrAppBand}
 &  &  & MAE~$\downarrow$ &  & 5.6641 & 5.6417 & 6.2143 & \underline{5.3482} & \cellcolor{mrAppOurs}\textbf{4.6520} \\
 &  &  & CCC~$\uparrow$ &  & -0.0047 & 0.0135 & -0.0917 & \underline{0.0520} & \cellcolor{mrAppOurs}\textbf{0.3268} \\
\rowcolor{mrAppBand}
 &  &  & $R^2$~$\uparrow$ &  & -0.2022 & -0.2192 & -0.4630 & \underline{-0.0783} & \cellcolor{mrAppOurs}\textbf{0.1403} \\
\cmidrule(l){3-10}
 &  & \emph{Avg. Rank} & $\downarrow$ &  & \underline{3.00} & \underline{3.00} & 5.00 & \underline{3.00} & \cellcolor{mrAppOurs}\textbf{1.00} \\
\midrule
\rowcolor{mrAppBand}
\textbf{MPDD-2026} & Elder & Binary & Accuracy~$\uparrow$ &  & \textbf{0.6957} & \underline{0.6522} & 0.3478 & 0.3043 & \cellcolor{mrAppOurs}0.5217 \\
 &  &  & Macro-F1~$\uparrow$ &  & \textbf{0.5165} & 0.3947 & 0.3002 & 0.2333 & \cellcolor{mrAppOurs}\underline{0.4559} \\
\rowcolor{mrAppBand}
 &  &  & $\kappa$~$\uparrow$ &  & \textbf{0.1006} & -0.0824 & \underline{0.0390} & 0.0000 & \cellcolor{mrAppOurs}-0.0858 \\
\cmidrule(l){3-10}
 &  & Ternary & Accuracy~$\uparrow$ &  & \textbf{0.6957} & \textbf{0.6957} & 0.1739 & 0.1739 & \cellcolor{mrAppOurs}\underline{0.6087} \\
\rowcolor{mrAppBand}
 &  &  & Macro-F1~$\uparrow$ &  & \underline{0.2735} & \underline{0.2735} & 0.1262 & 0.0988 & \cellcolor{mrAppOurs}\textbf{0.2978} \\
 &  &  & $\kappa$~$\uparrow$ &  & 0.0000 & 0.0000 & \underline{0.0180} & 0.0000 & \cellcolor{mrAppOurs}\textbf{0.0524} \\
\cmidrule(l){3-10}
\rowcolor{mrAppBand}
 &  & PHQ-9 & RMSE~$\downarrow$ &  & \textbf{4.5095} & \textbf{4.5095} & 12.9653 & 8.2090 & \cellcolor{mrAppOurs}\underline{5.2920} \\
 &  &  & MAE~$\downarrow$ &  & \textbf{3.3475} & \textbf{3.3475} & 12.0880 & 7.6242 & \cellcolor{mrAppOurs}\underline{4.0878} \\
\rowcolor{mrAppBand}
 &  &  & CCC~$\uparrow$ &  & 0.0000 & 0.0000 & \underline{0.0154} & -0.0054 & \cellcolor{mrAppOurs}\textbf{0.1321} \\
 &  &  & $R^2$~$\uparrow$ &  & \textbf{-0.0088} & \textbf{-0.0088} & -7.3392 & -2.3431 & \cellcolor{mrAppOurs}\underline{-0.3893} \\
\cmidrule(l){2-10}
\rowcolor{mrAppBand}
 & Young & Binary & Accuracy~$\uparrow$ &  & \underline{0.5000} & \underline{0.5000} & \underline{0.5000} & 0.3636 & \cellcolor{mrAppOurs}\textbf{0.5455} \\
 &  &  & Macro-F1~$\uparrow$ &  & 0.3333 & \underline{0.4990} & 0.3333 & 0.2667 & \cellcolor{mrAppOurs}\textbf{0.5455} \\
\rowcolor{mrAppBand}
 &  &  & $\kappa$~$\uparrow$ &  & \underline{0.0000} & \underline{0.0000} & \underline{0.0000} & -0.2727 & \cellcolor{mrAppOurs}\textbf{0.0909} \\
\cmidrule(l){3-10}
 &  & Ternary & Accuracy~$\uparrow$ &  & 0.3636 & \underline{0.4545} & 0.2273 & 0.2727 & \cellcolor{mrAppOurs}\textbf{0.5909} \\
\rowcolor{mrAppBand}
 &  &  & Macro-F1~$\uparrow$ &  & 0.1905 & \underline{0.2083} & 0.1978 & 0.1429 & \cellcolor{mrAppOurs}\textbf{0.4327} \\
 &  &  & $\kappa$~$\uparrow$ &  & -0.1241 & 0.0000 & \underline{0.0508} & -0.2438 & \cellcolor{mrAppOurs}\textbf{0.2954} \\
\cmidrule(l){3-10}
\rowcolor{mrAppBand}
 &  & PHQ-9 & RMSE~$\downarrow$ &  & 5.8275 & \underline{5.6060} & 12.0845 & 7.4547 & \cellcolor{mrAppOurs}\textbf{4.2054} \\
 &  &  & MAE~$\downarrow$ &  & 4.5023 & \underline{3.8641} & 11.1895 & 6.7159 & \cellcolor{mrAppOurs}\textbf{3.4348} \\
\rowcolor{mrAppBand}
 &  &  & CCC~$\uparrow$ &  & \underline{0.0760} & 0.0000 & -0.0124 & -0.1121 & \cellcolor{mrAppOurs}\textbf{0.6050} \\
 &  &  & $R^2$~$\uparrow$ &  & -0.2912 & \underline{-0.1949} & -4.5525 & -1.1130 & \cellcolor{mrAppOurs}\textbf{0.3276} \\
\cmidrule(l){3-10}
\rowcolor{mrAppBand}
 &  & \emph{Avg. Rank} & $\downarrow$ &  & 2.58 & \underline{2.17} & 4.08 & 4.67 & \cellcolor{mrAppOurs}\textbf{1.50} \\
\midrule
\textbf{Mental Health} & -- & Status & Accuracy~$\uparrow$ &  & --- & \underline{0.9738} & 0.0887 & 0.0775 & \cellcolor{mrAppOurs}\textbf{0.9890} \\
\rowcolor{mrAppBand}
 &  &  & Bal. Acc.~$\uparrow$ &  & --- & \underline{0.9692} & 0.2912 & 0.2697 & \cellcolor{mrAppOurs}\textbf{0.9881} \\
 &  &  & Macro-F1~$\uparrow$ &  & --- & \underline{0.9548} & 0.0843 & 0.0607 & \cellcolor{mrAppOurs}\textbf{0.9655} \\
\rowcolor{mrAppBand}
 &  &  & W. F1~$\uparrow$ &  & --- & \underline{0.9739} & 0.0694 & 0.0358 & \cellcolor{mrAppOurs}\textbf{0.9894} \\
 &  &  & $\kappa$~$\uparrow$ &  & --- & \underline{0.9611} & -0.0786 & -0.0855 & \cellcolor{mrAppOurs}\textbf{0.9837} \\
\cmidrule(l){3-10}
\rowcolor{mrAppBand}
 &  & \emph{Avg. Rank} & $\downarrow$ &  & --- & \underline{2.00} & 3.00 & 4.00 & \cellcolor{mrAppOurs}\textbf{1.00} \\
\midrule
\textbf{Kaggle} & -- & Video & Accuracy~$\uparrow$ &  & --- & --- & \underline{0.4208} & --- & \cellcolor{mrAppOurs}\textbf{0.7667} \\
\rowcolor{mrAppBand}
 &  &  & Bal. Acc.~$\uparrow$ &  & --- & --- & \underline{0.5174} & --- & \cellcolor{mrAppOurs}\textbf{0.7622} \\
 &  &  & Macro-F1~$\uparrow$ &  & --- & --- & \underline{0.3236} & --- & \cellcolor{mrAppOurs}\textbf{0.7593} \\
\rowcolor{mrAppBand}
 &  &  & W. F1~$\uparrow$ &  & --- & --- & \underline{0.2723} & --- & \cellcolor{mrAppOurs}\textbf{0.7677} \\
 &  &  & $\kappa$~$\uparrow$ &  & --- & --- & \underline{0.0280} & --- & \cellcolor{mrAppOurs}\textbf{0.5189} \\
\cmidrule(l){3-10}
\rowcolor{mrAppBand}
 &  & Audio--visual & Accuracy~$\uparrow$ &  & --- & \underline{0.6292} & 0.4000 & 0.6000 & \cellcolor{mrAppOurs}\textbf{0.7875} \\
 &  &  & Bal. Acc.~$\uparrow$ &  & --- & \underline{0.6233} & 0.4965 & 0.5017 & \cellcolor{mrAppOurs}\textbf{0.7927} \\
\rowcolor{mrAppBand}
 &  &  & Macro-F1~$\uparrow$ &  & --- & \underline{0.6201} & 0.2966 & 0.3845 & \cellcolor{mrAppOurs}\textbf{0.7836} \\
 &  &  & W. F1~$\uparrow$ &  & --- & \underline{0.6319} & 0.2427 & 0.4574 & \cellcolor{mrAppOurs}\textbf{0.7894} \\
\rowcolor{mrAppBand}
 &  &  & $\kappa$~$\uparrow$ &  & --- & \underline{0.2419} & -0.0056 & 0.0041 & \cellcolor{mrAppOurs}\textbf{0.5695} \\
\cmidrule(l){3-10}
 &  & \emph{Avg. Rank} & $\downarrow$ &  & --- & --- & \underline{3.00} & --- & \cellcolor{mrAppOurs}\textbf{1.00} \\
\bottomrule
\end{tabular}
\endgroup
\end{table}

\begin{table}[p]
\ContinuedFloat
\caption[]{Complete performance metrics (continued): MPDD-2025 Elder. All seven task metrics and all three track aggregates are shown. W.\ Acc. denotes inverse-class-frequency-weighted accuracy, and W.\ F1 denotes support-weighted F1. Task Acc.\ averages Accuracy and W.\ Acc.; Task F1 averages Macro-F1 and W.\ F1; Official averages Task Acc.\ and Task F1. Track metrics average the corresponding task metrics over both windows. Flex-MoE cells report the means of three runs. Ours Track Acc.\ and Track F1 are computed from the reported task components. Ours W.\ Acc.\ was not included in the available main-run summary and is left unreported. Formatting follows the first part.}
\label{tab:full-results-mpdd2025-elder}
\centering
\begingroup
\mrAppendixDenseTableStyle
\begin{tabular}{>{\raggedright\arraybackslash}m{.115\linewidth}>{\raggedright\arraybackslash}m{.06\linewidth}>{\raggedright\arraybackslash}m{.105\linewidth}>{\raggedright\arraybackslash}m{.125\linewidth}>{\centering\arraybackslash}m{.075\linewidth}>{\centering\arraybackslash}m{.07\linewidth}>{\centering\arraybackslash}m{.09\linewidth}>{\centering\arraybackslash}m{.09\linewidth}>{\centering\arraybackslash}m{.09\linewidth}>{\centering\arraybackslash}m{\dimexpr.18\linewidth-20\tabcolsep\relax}}
\toprule
\rowcolor{mrAppHeader}
\textbf{Dataset} & \textbf{Cohort} & \textbf{Task} & \textbf{Metric} & \textbf{Official} & \textbf{MulT} & \textbf{Flex-MoE} & \textbf{Gemini 2.5}\newline\textbf{Flash} & \textbf{MDAgents} & \cellcolor{mrAppOurs}\textbf{Ours} \\
\midrule
\rowcolor{mrAppBand}
\textbf{MPDD-2025} & Elder & 1\,s Binary & Accuracy~$\uparrow$ & 0.6916 & \textbf{0.7841} & 0.7416 & 0.7357 & 0.1982 & \cellcolor{mrAppOurs}\underline{0.7789} \\
 &  &  & Macro-F1~$\uparrow$ & 0.5900 & \underline{0.6318} & 0.6266 & 0.5048 & 0.1809 & \cellcolor{mrAppOurs}\textbf{0.6964} \\
\rowcolor{mrAppBand}
 &  &  & W. Acc.~$\uparrow$ & \underline{0.6360} & 0.6332 & \textbf{0.6597} & 0.5055 & 0.5035 & \cellcolor{mrAppOurs}--- \\
 &  &  & W. F1~$\uparrow$ & 0.7222 & \underline{0.7852} & 0.7604 & 0.7238 & 0.1038 & \cellcolor{mrAppOurs}\textbf{0.7987} \\
\rowcolor{mrAppBand}
 &  &  & Task Acc.~$\uparrow$ & 0.6638 & \underline{0.7086} & 0.7006 & 0.6206 & 0.3509 & \cellcolor{mrAppOurs}\textbf{0.7712} \\
 &  &  & Task F1~$\uparrow$ & 0.6561 & \underline{0.7085} & 0.6935 & 0.6143 & 0.1424 & \cellcolor{mrAppOurs}\textbf{0.7475} \\
\rowcolor{mrAppBand}
 &  &  & Official~$\uparrow$ & 0.6600 & \underline{0.7086} & 0.6971 & 0.6174 & 0.2466 & \cellcolor{mrAppOurs}\textbf{0.7593} \\
\cmidrule(l){3-10}
 &  & 1\,s Ternary & Accuracy~$\uparrow$ & \underline{0.4802} & 0.3877 & 0.4302 & 0.4405 & 0.3524 & \cellcolor{mrAppOurs}\textbf{0.6291} \\
\rowcolor{mrAppBand}
 &  &  & Macro-F1~$\uparrow$ & \underline{0.4777} & 0.3068 & 0.4265 & 0.2617 & 0.1936 & \cellcolor{mrAppOurs}\textbf{0.6147} \\
 &  &  & W. Acc.~$\uparrow$ & \textbf{0.5183} & 0.3218 & \underline{0.4800} & 0.3375 & 0.3370 & \cellcolor{mrAppOurs}--- \\
\rowcolor{mrAppBand}
 &  &  & W. F1~$\uparrow$ & \underline{0.4676} & 0.3728 & 0.4093 & 0.3263 & 0.2090 & \cellcolor{mrAppOurs}\textbf{0.6294} \\
 &  &  & Task Acc.~$\uparrow$ & \underline{0.4993} & 0.3547 & 0.4551 & 0.3890 & 0.3447 & \cellcolor{mrAppOurs}\textbf{0.6371} \\
\rowcolor{mrAppBand}
 &  &  & Task F1~$\uparrow$ & \underline{0.4726} & 0.3398 & 0.4179 & 0.2940 & 0.2013 & \cellcolor{mrAppOurs}\textbf{0.6220} \\
 &  &  & Official~$\uparrow$ & \underline{0.4859} & 0.3472 & 0.4365 & 0.3415 & 0.2730 & \cellcolor{mrAppOurs}\textbf{0.6295} \\
\cmidrule(l){3-10}
\rowcolor{mrAppBand}
 &  & 1\,s Quinary & Accuracy~$\uparrow$ & 0.5419 & \textbf{0.7225} & \underline{0.6960} & 0.4229 & 0.1630 & \cellcolor{mrAppOurs}0.5612 \\
 &  &  & Macro-F1~$\uparrow$ & 0.1936 & 0.1678 & \underline{0.2763} & 0.1553 & 0.0606 & \cellcolor{mrAppOurs}\textbf{0.3776} \\
\rowcolor{mrAppBand}
 &  &  & W. Acc.~$\uparrow$ & 0.1702 & \underline{0.2000} & \textbf{0.2793} & 0.1976 & 0.1866 & \cellcolor{mrAppOurs}--- \\
 &  &  & W. F1~$\uparrow$ & 0.5701 & \underline{0.6061} & \textbf{0.6432} & 0.4626 & 0.0640 & \cellcolor{mrAppOurs}0.5624 \\
\rowcolor{mrAppBand}
 &  &  & Task Acc.~$\uparrow$ & 0.3560 & 0.4612 & \textbf{0.4877} & 0.3103 & 0.1748 & \cellcolor{mrAppOurs}\underline{0.4777} \\
 &  &  & Task F1~$\uparrow$ & 0.3818 & 0.3869 & \underline{0.4598} & 0.3089 & 0.0623 & \cellcolor{mrAppOurs}\textbf{0.4700} \\
\rowcolor{mrAppBand}
 &  &  & Official~$\uparrow$ & 0.3689 & 0.4241 & \underline{0.4737} & 0.3096 & 0.1186 & \cellcolor{mrAppOurs}\textbf{0.4738} \\
\cmidrule(l){3-10}
 &  & 5\,s Binary & Accuracy~$\uparrow$ & 0.7269 & 0.6652 & \textbf{0.8253} & \underline{0.7753} & 0.1894 & \cellcolor{mrAppOurs}0.7648 \\
\rowcolor{mrAppBand}
 &  &  & Macro-F1~$\uparrow$ & 0.5774 & 0.4449 & \textbf{0.7126} & 0.5032 & 0.1673 & \cellcolor{mrAppOurs}\underline{0.6677} \\
 &  &  & W. Acc.~$\uparrow$ & \underline{0.5886} & 0.4430 & \textbf{0.7695} & 0.5099 & 0.5080 & \cellcolor{mrAppOurs}--- \\
\rowcolor{mrAppBand}
 &  &  & W. F1~$\uparrow$ & 0.7402 & 0.6714 & \textbf{0.8270} & 0.7413 & 0.0794 & \cellcolor{mrAppOurs}\underline{0.7839} \\
 &  &  & Task Acc.~$\uparrow$ & 0.6577 & 0.5541 & \textbf{0.7974} & 0.6426 & 0.3487 & \cellcolor{mrAppOurs}\underline{0.7402} \\
\rowcolor{mrAppBand}
 &  &  & Task F1~$\uparrow$ & 0.6588 & 0.5581 & \textbf{0.7698} & 0.6223 & 0.1234 & \cellcolor{mrAppOurs}\underline{0.7258} \\
 &  &  & Official~$\uparrow$ & 0.6583 & 0.5561 & \textbf{0.7836} & 0.6324 & 0.2360 & \cellcolor{mrAppOurs}\underline{0.7330} \\
\cmidrule(l){3-10}
\rowcolor{mrAppBand}
 &  & 5\,s Ternary & Accuracy~$\uparrow$ & 0.4361 & 0.3392 & \underline{0.5022} & 0.4537 & 0.3480 & \cellcolor{mrAppOurs}\textbf{0.6802} \\
 &  &  & Macro-F1~$\uparrow$ & 0.4425 & 0.2063 & \underline{0.4436} & 0.2534 & 0.1824 & \cellcolor{mrAppOurs}\textbf{0.6708} \\
\rowcolor{mrAppBand}
 &  &  & W. Acc.~$\uparrow$ & \underline{0.4617} & 0.2566 & \textbf{0.4770} & 0.3361 & 0.3352 & \cellcolor{mrAppOurs}--- \\
 &  &  & W. F1~$\uparrow$ & 0.4317 & 0.2746 & \underline{0.4991} & 0.3238 & 0.1930 & \cellcolor{mrAppOurs}\textbf{0.6806} \\
\rowcolor{mrAppBand}
 &  &  & Task Acc.~$\uparrow$ & 0.4489 & 0.2979 & \underline{0.4896} & 0.3949 & 0.3416 & \cellcolor{mrAppOurs}\textbf{0.6940} \\
 &  &  & Task F1~$\uparrow$ & 0.4371 & 0.2405 & \underline{0.4713} & 0.2886 & 0.1877 & \cellcolor{mrAppOurs}\textbf{0.6757} \\
\rowcolor{mrAppBand}
 &  &  & Official~$\uparrow$ & 0.4430 & 0.2692 & \underline{0.4805} & 0.3418 & 0.2646 & \cellcolor{mrAppOurs}\textbf{0.6849} \\
\cmidrule(l){3-10}
 &  & 5\,s Quinary & Accuracy~$\uparrow$ & \underline{0.6476} & \textbf{0.7225} & 0.5830 & 0.5639 & 0.1630 & \cellcolor{mrAppOurs}0.4899 \\
\rowcolor{mrAppBand}
 &  &  & Macro-F1~$\uparrow$ & \underline{0.2343} & 0.1678 & 0.1938 & 0.1480 & 0.0615 & \cellcolor{mrAppOurs}\textbf{0.3023} \\
 &  &  & W. Acc.~$\uparrow$ & \textbf{0.2762} & 0.2000 & \underline{0.2193} & 0.1561 & 0.1866 & \cellcolor{mrAppOurs}--- \\
\rowcolor{mrAppBand}
 &  &  & W. F1~$\uparrow$ & \textbf{0.6259} & \underline{0.6061} & 0.5742 & 0.5345 & 0.0649 & \cellcolor{mrAppOurs}0.5058 \\
 &  &  & Task Acc.~$\uparrow$ & \textbf{0.4619} & \underline{0.4612} & 0.4011 & 0.3600 & 0.1748 & \cellcolor{mrAppOurs}0.3993 \\
\rowcolor{mrAppBand}
 &  &  & Task F1~$\uparrow$ & \textbf{0.4301} & 0.3869 & 0.3840 & 0.3413 & 0.0632 & \cellcolor{mrAppOurs}\underline{0.4040} \\
 &  &  & Official~$\uparrow$ & \textbf{0.4460} & \underline{0.4241} & 0.3926 & 0.3506 & 0.1190 & \cellcolor{mrAppOurs}0.4017 \\
\cmidrule(l){3-10}
\rowcolor{mrAppBand}
 &  & Track & Track Acc.~$\uparrow$ & 0.5146 & 0.4730 & \underline{0.5553} & 0.4529 & 0.2893 & \cellcolor{mrAppOurs}\textbf{0.6199} \\
 &  &  & Track F1~$\uparrow$ & 0.5061 & 0.4368 & \underline{0.5327} & 0.4116 & 0.1300 & \cellcolor{mrAppOurs}\textbf{0.6075} \\
\rowcolor{mrAppBand}
 &  &  & Track Score~$\uparrow$ & 0.5104 & 0.4549 & \underline{0.5440} & 0.4322 & 0.2096 & \cellcolor{mrAppOurs}\textbf{0.6137} \\
\bottomrule
\end{tabular}
\endgroup
\end{table}

\begin{table}[p]
\ContinuedFloat
\caption[]{Complete performance metrics (continued): MPDD-2025 Young. Metric definitions, aggregation, and formatting follow the Elder part. Flex-MoE cells report the means of three runs. Avg.\ Rank covers all ten MPDD-2025 tasks across Elder and Young, using the official task score once per task. Ours W.\ Acc.\ was not included in the available main-run summary and is left unreported.}
\label{tab:full-results-mpdd2025-young}
\centering
\begingroup
\mrAppendixDenseTableStyle
\begin{tabular}{>{\raggedright\arraybackslash}m{.115\linewidth}>{\raggedright\arraybackslash}m{.06\linewidth}>{\raggedright\arraybackslash}m{.105\linewidth}>{\raggedright\arraybackslash}m{.125\linewidth}>{\centering\arraybackslash}m{.075\linewidth}>{\centering\arraybackslash}m{.07\linewidth}>{\centering\arraybackslash}m{.09\linewidth}>{\centering\arraybackslash}m{.09\linewidth}>{\centering\arraybackslash}m{.09\linewidth}>{\centering\arraybackslash}m{\dimexpr.18\linewidth-20\tabcolsep\relax}}
\toprule
\rowcolor{mrAppHeader}
\textbf{Dataset} & \textbf{Cohort} & \textbf{Task} & \textbf{Metric} & \textbf{Official} & \textbf{MulT} & \textbf{Flex-MoE} & \textbf{Gemini 2.5}\newline\textbf{Flash} & \textbf{MDAgents} & \cellcolor{mrAppOurs}\textbf{Ours} \\
\midrule
\rowcolor{mrAppBand}
\textbf{MPDD-2025} & Young & 1\,s Binary & Accuracy~$\uparrow$ & \underline{0.5644} & 0.4924 & \textbf{0.6048} & 0.4848 & 0.5455 & \cellcolor{mrAppOurs}0.4038 \\
 &  &  & Macro-F1~$\uparrow$ & \underline{0.5446} & 0.4491 & \textbf{0.5749} & 0.3450 & 0.4762 & \cellcolor{mrAppOurs}0.4018 \\
\rowcolor{mrAppBand}
 &  &  & W. Acc.~$\uparrow$ & \underline{0.5644} & 0.4924 & \textbf{0.6048} & 0.4848 & 0.5455 & \cellcolor{mrAppOurs}--- \\
 &  &  & W. F1~$\uparrow$ & \underline{0.5446} & 0.4491 & \textbf{0.5749} & 0.3450 & 0.4762 & \cellcolor{mrAppOurs}0.4018 \\
\rowcolor{mrAppBand}
 &  &  & Task Acc.~$\uparrow$ & \underline{0.5644} & 0.4924 & \textbf{0.6048} & 0.4848 & 0.5455 & \cellcolor{mrAppOurs}0.4038 \\
 &  &  & Task F1~$\uparrow$ & \underline{0.5446} & 0.4491 & \textbf{0.5749} & 0.3450 & 0.4762 & \cellcolor{mrAppOurs}0.4018 \\
\rowcolor{mrAppBand}
 &  &  & Official~$\uparrow$ & \underline{0.5545} & 0.4708 & \textbf{0.5899} & 0.4149 & 0.5108 & \cellcolor{mrAppOurs}0.4028 \\
\cmidrule(l){3-10}
 &  & 1\,s Ternary & Accuracy~$\uparrow$ & 0.4015 & \textbf{0.5000} & 0.4760 & \underline{0.4773} & 0.4242 & \cellcolor{mrAppOurs}0.3795 \\
\rowcolor{mrAppBand}
 &  &  & Macro-F1~$\uparrow$ & 0.3415 & 0.2222 & \underline{0.3511} & 0.2304 & 0.2763 & \cellcolor{mrAppOurs}\textbf{0.3590} \\
 &  &  & W. Acc.~$\uparrow$ & 0.3432 & 0.3333 & \textbf{0.3801} & 0.3249 & \underline{0.3662} & \cellcolor{mrAppOurs}--- \\
\rowcolor{mrAppBand}
 &  &  & W. F1~$\uparrow$ & \underline{0.4004} & 0.3333 & \textbf{0.4171} & 0.3298 & 0.3404 & \cellcolor{mrAppOurs}0.3799 \\
 &  &  & Task Acc.~$\uparrow$ & 0.3724 & \underline{0.4167} & \textbf{0.4280} & 0.4011 & 0.3952 & \cellcolor{mrAppOurs}0.3610 \\
\rowcolor{mrAppBand}
 &  &  & Task F1~$\uparrow$ & \underline{0.3709} & 0.2778 & \textbf{0.3841} & 0.2801 & 0.3083 & \cellcolor{mrAppOurs}0.3694 \\
 &  &  & Official~$\uparrow$ & \underline{0.3716} & 0.3472 & \textbf{0.4061} & 0.3406 & 0.3518 & \cellcolor{mrAppOurs}0.3653 \\
\cmidrule(l){3-10}
\rowcolor{mrAppBand}
 &  & 5\,s Binary & Accuracy~$\uparrow$ & \textbf{0.5682} & 0.4242 & 0.5013 & \underline{0.5227} & 0.5038 & \cellcolor{mrAppOurs}0.3848 \\
 &  &  & Macro-F1~$\uparrow$ & \textbf{0.5678} & 0.3185 & 0.3361 & \underline{0.4888} & 0.4105 & \cellcolor{mrAppOurs}0.3832 \\
\rowcolor{mrAppBand}
 &  &  & W. Acc.~$\uparrow$ & \textbf{0.5682} & 0.4242 & 0.5013 & \underline{0.5227} & 0.5038 & \cellcolor{mrAppOurs}--- \\
 &  &  & W. F1~$\uparrow$ & \textbf{0.5678} & 0.3185 & 0.3361 & \underline{0.4888} & 0.4105 & \cellcolor{mrAppOurs}0.3832 \\
\rowcolor{mrAppBand}
 &  &  & Task Acc.~$\uparrow$ & \textbf{0.5682} & 0.4242 & 0.5013 & \underline{0.5227} & 0.5038 & \cellcolor{mrAppOurs}0.3848 \\
 &  &  & Task F1~$\uparrow$ & \textbf{0.5678} & 0.3185 & 0.3361 & \underline{0.4888} & 0.4105 & \cellcolor{mrAppOurs}0.3832 \\
\rowcolor{mrAppBand}
 &  &  & Official~$\uparrow$ & \textbf{0.5680} & 0.3714 & 0.4187 & \underline{0.5058} & 0.4572 & \cellcolor{mrAppOurs}0.3840 \\
\cmidrule(l){3-10}
 &  & 5\,s Ternary & Accuracy~$\uparrow$ & 0.3523 & \textbf{0.5265} & \underline{0.4545} & 0.4205 & 0.3826 & \cellcolor{mrAppOurs}0.3629 \\
\rowcolor{mrAppBand}
 &  &  & Macro-F1~$\uparrow$ & 0.2289 & 0.2812 & \textbf{0.3188} & 0.2602 & 0.2334 & \cellcolor{mrAppOurs}\underline{0.3176} \\
 &  &  & W. Acc.~$\uparrow$ & 0.2519 & \textbf{0.3595} & \underline{0.3578} & 0.3342 & 0.3374 & \cellcolor{mrAppOurs}--- \\
\rowcolor{mrAppBand}
 &  &  & W. F1~$\uparrow$ & 0.3147 & \textbf{0.3988} & \underline{0.3834} & 0.3321 & 0.2786 & \cellcolor{mrAppOurs}0.3377 \\
 &  &  & Task Acc.~$\uparrow$ & 0.3021 & \textbf{0.4430} & \underline{0.4062} & 0.3773 & 0.3600 & \cellcolor{mrAppOurs}0.3475 \\
\rowcolor{mrAppBand}
 &  &  & Task F1~$\uparrow$ & 0.2718 & \underline{0.3400} & \textbf{0.3511} & 0.2961 & 0.2560 & \cellcolor{mrAppOurs}0.3276 \\
 &  &  & Official~$\uparrow$ & 0.2869 & \textbf{0.3915} & \underline{0.3786} & 0.3367 & 0.3080 & \cellcolor{mrAppOurs}0.3376 \\
\cmidrule(l){3-10}
\rowcolor{mrAppBand}
 &  & Track & Track Acc.~$\uparrow$ & \underline{0.4518} & 0.4441 & \textbf{0.4851} & 0.4465 & 0.4511 & \cellcolor{mrAppOurs}0.3743 \\
 &  &  & Track F1~$\uparrow$ & \textbf{0.4388} & 0.3464 & \underline{0.4116} & 0.3525 & 0.3628 & \cellcolor{mrAppOurs}0.3705 \\
\rowcolor{mrAppBand}
 &  &  & Track Score~$\uparrow$ & \underline{0.4453} & 0.3952 & \textbf{0.4483} & 0.3995 & 0.4069 & \cellcolor{mrAppOurs}0.3724 \\
\cmidrule(l){3-10}
 &  & \emph{Avg. Rank} & $\downarrow$ & 2.80 & 3.70 & \textbf{2.30} & 4.50 & 5.10 & \cellcolor{mrAppOurs}\underline{2.60} \\
\bottomrule
\end{tabular}
\endgroup
\end{table}

\paragraph{Baseline prediction diagnostics.}
The saved MCMoE predictions assign all 56 E-DAIC test participants to the majority class, yielding accuracy 0.6964 and Macro-F1 0.4105. A cross-entropy-head check with corrected deterministic evaluation masking produces the same constant classification. PHQ-8 predictions are nearly constant, with standard deviation 0.00027 and CCC approximately zero. On MPDD-2026 Elder, ternary predictions likewise assign every participant to the majority class, yielding accuracy 0.6957 and Macro-F1 0.2735. MulT and Flex-MoE produce the same constant ternary predictions in these runs. These diagnostics characterize the evaluated implementations; they do not establish an inherent limitation of the architectures.

\paragraph{Additional raw-video comparison.}
\phantomsection\label{app:experiments:medagent-pro}
Table~\ref{tab:medagent-pro-video} compares MedAgent-Pro~\citep{wang2025medagentpro}, direct Gemini 2.5 Flash prediction, and \method{} on Kaggle video-only emotion classification. MedAgent-Pro uses Gemini 2.5 Flash to process sampled video frames.

\begin{table}[!htbp]
\caption{Kaggle raw-video emotion classification. Best and second-best scores are bold and underlined. \method{} uses actor-disjoint nested cross-validation with review disabled.}
\label{tab:medagent-pro-video}
\centering
\begingroup
\mrAppendixTableStyle
\begin{tabular}{>{\raggedright\arraybackslash}p{.48\linewidth}>{\centering\arraybackslash}p{\dimexpr.27\linewidth-4\tabcolsep\relax}}
\toprule
\rowcolor{mrAppHeader}
\textbf{Method} & \textbf{Macro-F1 $\uparrow$} \\
\midrule
\rowcolor{mrAppBand}
Gemini 2.5 Flash & 0.3236 \\
MedAgent-Pro & \underline{0.7080} \\
\rowcolor{mrAppOurs}
\method{} (Ours) & \textbf{0.7593} \\
\bottomrule
\end{tabular}
\endgroup
\end{table}

\clearpage
\section{Supplementary Experimental Setup}
\label{app:baseline_setup}

\FloatBarrier
\subsection{Datasets and Evaluation Protocols}
\label{app:datasets_protocols}

\paragraph{MPDD-AVG 2026.}
\phantomsection\label{app:mpdd2026_protocol}
MPDD\_2026 contains 175 labeled training subjects and 45 held-out test subjects. Track~1 comprises 87 training and 23 test subjects from the elderly cohort; Track~2 comprises 88 training and 22 test subjects from the young cohort. MulT followed the official per-track 90/10 stratified split, yielding 78/9 training/validation subjects for Track~1 and 79/9 for Track~2. Flex-MoE combined the two cohorts during training, yielding 157/18/45 training/validation/test subjects per task, and metrics were subsequently reported by track. Classification metrics were Accuracy, Macro-F1, and Cohen's $\kappa$; PHQ-9 metrics were RMSE, MAE, CCC, and $R^2$. MulT and Flex-MoE did not include dedicated PHQ-9 regression heads in these runs. Their submission-compatible PHQ-9 values were derived from ternary-class midpoints (3, 8, and 16) and are not reported as dedicated regression results.

\paragraph{MPDD 2025.}
\phantomsection\label{app:mpdd2025}
The release provides event-level A+V+P features computed over 1-s and 5-s windows. Track~1 contains 337 training and 227 test events and includes binary, ternary, and quinary classification; Track~2 contains 264 training and 264 test events and includes binary and ternary classification.

For each task, the official score is the mean of task Accuracy and task F1. Task Accuracy is itself the mean of ordinary Accuracy and inverse-class-frequency-weighted Accuracy, while task F1 is the mean of Macro-F1 and support-weighted F1. Each track score is the arithmetic mean of its task scores.

\paragraph{E-DAIC.}
\phantomsection\label{app:edaic}
E-DAIC contains 275 participants with the official 163/56/56 train/development/test split. The two local directories named \texttt{waic} and \texttt{waic\_woz} were verified to be duplicate copies of the same release and were therefore not treated as independent datasets. We aligned three participant-level modalities: 23-dimensional OpenSMILE~2.3.0 eGeMAPS audio frames, 49 OpenFace~2.1.0 pose/gaze/action-unit features, and participant transcripts. Audio and visual sequences were resampled to 64 time steps. Their pooled representations concatenated the mean, standard deviation, 25th percentile, median, and 75th percentile. Transcripts were represented using a 128-dimensional unigram/bigram TF--IDF vocabulary fitted only on the official training split.

\paragraph{Kaggle\_public.}
\phantomsection\label{app:kaggle}
\texttt{Kaggle\_public} denotes the media subsets obtained from the Kaggle collection \emph{Multimodal dataset for depression analysis}~\citep{s3programmerlead2025multimodal}. The audio recordings originate from the Toronto Emotional Speech Set (TESS)~\citep{pichorafuller2020tess} and retain its OAF/YAF speaker prefixes. The video and aligned audio--visual recordings originate from the Ryerson Audio-Visual Database of Emotional Speech and Song (RAVDESS)~\citep{livingstone2018ravdess} and retain its seven-field filenames and actor identifiers. The protocols below use subsets of these original corpora.

We used two leakage-controlled protocols. The audio-only protocol contained 399 samples and used two leave-group-out folds: training on YAF and testing on OAF, followed by the reverse direction. Its binary labels follow the archive's \texttt{Normal} and \texttt{Depression} directory grouping of acted speech and are treated as emotion proxies. Audio was resampled to 16~kHz and represented using 20 MFCCs, MFCC deltas, RMS statistics, zero-crossing rate, and duration. After standardization, we fitted Logistic Regression ($C=1$) and an RBF SVM ($C=10$, $\gamma=\texttt{scale}$) with seed~37.

The aligned audio--video protocol contained 240 MP4 files from 24 actors. Audio and frames were extracted from the same MP4 file. Each video was represented by ImageNet-pretrained ResNet-18 embeddings from eight uniformly sampled frames; the corresponding audio was represented by an eight-step, 40-dimensional MFCC+delta sequence. Pooled vectors concatenated the mean, standard deviation, 25th percentile, median, and 75th percentile. Six actor-disjoint folds each held out four actors (40 samples). Classical A+V Logistic Regression and RBF SVM used nested actor-grouped cross-validation for hyperparameter selection. In each fold, Flex-MoE used 16 training actors, four validation actors, and four test actors, with three seeds and the settings in Appendix~\ref{app:mpdd2026_baselines}. The filename-derived target distinguishes neutral (\texttt{01}) from calm (\texttt{02}) affect and is therefore an emotion proxy, not a clinical depression diagnosis.

\paragraph{Mental Health Multimodal.}
\phantomsection\label{app:mental_health}
The tabular dataset contains 4,000 rows and four classes: Healthy, Mild Stress, Moderate Stress, and Severe Stress. Its 21 numeric variables were grouped into questionnaire (5), behavioral (4), visual (4), audio (4), and physiological (4) modalities. We used a fixed class-stratified split of 3,200 training and 800 test rows with \texttt{random\_state=2026}. Because the dataset provides no participant identifier, this is a row-level split and is not described as participant-disjoint. The archive's media files lack a documented row-level key and were consequently not joined to the table.

\FloatBarrier
\subsection{Baseline Implementations}
\label{app:baseline_implementations}

\FloatBarrier
\subsubsection{MPDD-AVG 2026}
\label{app:mpdd2026_baselines}

\paragraph{Direct LLM.}
We used Gemini~2.5~Flash as a zero-shot direct-prediction baseline. When a modality contained a temporal tensor, we first averaged it over time. Features were then concatenated across files and truncated or zero-padded to 256 values. For each modality, the prompt contained only its availability, dimensionality, finite-value ratio, mean, standard deviation, minimum, maximum, quartiles, $\ell_2$ norm, and first 12 values. The four summaries were serialized into a single JSON case bundle. No raw audio, video frame, or feature tensor was submitted to the model. The requested output contained a binary prediction, a ternary prediction, a PHQ-9 estimate, confidence, and a short rationale. Decoding was deterministic (temperature~0), Gemini's thinking budget was disabled, and the output budget was 1,024 tokens. Responses were cached by sample and prompt, and execution was terminated after two API failures.

\paragraph{MDAgents.}
We adapted the open-source MDAgents workflow at \href{https://github.com/mitmedialab/MDAgents/commit/3adbd760ca809b4e7b0c1085d68314b6e7d91e1b}{\texttt{3adbd76}}. GPT-5.5 received the same A+V+G+P feature-summary bundle used by Direct LLM. The adaptive workflow first assigned each case to basic, intermediate, or advanced difficulty and then invoked the corresponding solo or collaborative medical-agent path. In the retained MPDD\_2026 run, 34 cases were assigned to the basic path and 11 to the intermediate path. A strict parser converted the final response to the common binary, ternary, PHQ-9, confidence, and rationale schema. This design matched modality access between the two LLM-based baselines; it did not equate their LLM backbones or inference procedures.

\paragraph{MulT.}
We used the original Multimodal Transformer implementation at \href{https://github.com/yaohungt/Multimodal-Transformer/commit/a670936824ee722c8494fd98d204977a1d663c7a}{\texttt{a670936}}. Its three streams were assigned to personality (the original text branch), audio, and video. Gait was omitted because the original architecture supports three input streams. Track~1 used MFCC audio and DenseNet visual features, whereas Track~2 used MFCC64 audio and DenseNet visual features. Each audio and video sequence was standardized within the sample and linearly resampled to 128 time steps; the 1,024-dimensional personality embedding formed the third stream. We trained four classifiers independently (two tracks $\times$ binary/ternary classification) using the repository defaults: 40 epochs, batch size 24, Adam with learning rate $10^{-3}$, five attention levels, five heads, and seed~1111.

\paragraph{Flex-MoE.}
We used Flex-MoE at \href{https://github.com/UNITES-Lab/flex-moe/commit/21688077488e58c1a634ca7122cffdb28cae3bc9}{\texttt{2168807}} together with FastMoE (commit \texttt{55af4f9}). All A+V+G+P modalities were included. Temporal features were mean-pooled within each file and concatenated across files, and unavailable modalities were replaced by learned missing-modality embeddings. The model used hidden size 128, 16 patches, 16 experts, top-4 routing, one router, one fusion layer, one prediction layer, four attention heads, and dropout 0.5. We minimized cross-entropy plus $0.01$ times the gate-balancing loss using Adam with learning rate $10^{-3}$, batch size 16, five warm-up epochs, and 50 total epochs. Binary and ternary models were trained with seeds 0, 1, and 2. The best checkpoint for each seed was selected by validation accuracy, and test predictions were combined by majority vote.

\FloatBarrier
\subsubsection{Official MPDD 2025 Reproduction}
\label{app:mpdd2025_baseline}

We reproduced all ten tasks from the official \href{https://github.com/hacilab/MPDD}{MPDD repository}.  We retained the official architecture, feature combinations, batch sizes, learning rates, epoch counts, and checkpoint selection criterion (validation weighted F1). One run was performed for each task using the repository seed and \texttt{PYTHONHASHSEED=3407}. Table~\ref{tab:mpdd2025_setup} lists the task-specific settings.

\begin{table}[!htbp]
\caption{Task-specific settings for the official MPDD 2025 reproduction.}
\label{tab:mpdd2025_setup}
\centering
\begingroup
\mrAppendixTableStyle
\begin{tabular}{>{\raggedright\arraybackslash}p{.07\linewidth}>{\raggedright\arraybackslash}p{.14\linewidth}>{\raggedright\arraybackslash}p{.13\linewidth}>{\raggedright\arraybackslash}p{.115\linewidth}>{\centering\arraybackslash}p{.07\linewidth}>{\centering\arraybackslash}p{.065\linewidth}>{\centering\arraybackslash}p{.22\linewidth}>{\centering\arraybackslash}p{\dimexpr.19\linewidth-16\tabcolsep\relax}}
\toprule
\rowcolor{mrAppHeader}
    \textbf{Track} & \textbf{Window/task} & \textbf{Audio} & \textbf{Video} & \shortstack{\textbf{Max.}\\\textbf{length}} & \textbf{Batch} & \textbf{Learning rate} & \textbf{Epochs} \\
    \midrule
\rowcolor{mrAppBand}
    Elder & 1-s binary  & MFCC      & OpenFace & 26 & 1  & $2\times10^{-5}$          & 200 \\
    Elder & 1-s ternary & OpenSMILE & ResNet   & 26 & 2  & $4.5835899\times10^{-6}$  & 400 \\
\rowcolor{mrAppBand}
    Elder & 1-s quinary & OpenSMILE & DenseNet & 26 & 1  & $1.6570181\times10^{-5}$  & 400 \\
    Elder & 5-s binary  & OpenSMILE & ResNet   & 5  & 2  & $1.8\times10^{-5}$        & 200 \\
\rowcolor{mrAppBand}
    Elder & 5-s ternary & wav2vec   & OpenFace & 5  & 16 & $1.4797182\times10^{-4}$  & 400 \\
    Elder & 5-s quinary & MFCC      & ResNet   & 5  & 16 & $9.8\times10^{-5}$        & 200 \\
    \midrule
\rowcolor{mrAppBand}
    Young & 1-s binary  & wav2vec   & OpenFace & 25 & 16 & $6\times10^{-5}$          & 500 \\
    Young & 1-s ternary & MFCC      & DenseNet & 25 & 8  & $2.6\times10^{-4}$        & 500 \\
\rowcolor{mrAppBand}
    Young & 5-s binary  & OpenSMILE & ResNet   & 5  & 24 & $5\times10^{-5}$          & 500 \\
    Young & 5-s ternary & MFCC      & DenseNet & 5  & 8  & $4\times10^{-4}$          & 500 \\
    \bottomrule
\end{tabular}
\endgroup
\end{table}

\paragraph{E-DAIC baselines.}
The baselines comprised Logistic Regression and an RBF SVM for PHQ-8 binary classification; Ridge and an RBF SVR for PHQ-8 score regression; A+V+T MulT; A+V+T Flex-MoE; zero-shot Direct LLM; and the original MDAgents workflow. Classical hyperparameters were selected on the official development set using Macro-F1 for classification and MAE for regression. The selected classical models were then refitted on training plus development data and evaluated once on the test set. MulT used 40 epochs, batch size 24, Adam with learning rate $10^{-3}$, five attention levels, five heads, gradient clipping at 0.8, and seed~1111; checkpoints were selected by development Macro-F1 or MAE. Flex-MoE used the architecture and optimizer described in Appendix~\ref{app:mpdd2026_baselines}, with three seeds and probability averaging for classification or prediction averaging for regression.

\paragraph{Mental Health baselines.}
Logistic Regression and an RBF SVM used standardized early-fusion features, with five-fold stratified model selection on the 3,200-row training partition. Flex-MoE further split this pool into 2,800 training and 400 validation rows, retained the 800-row test set, and averaged class probabilities across three seeds. Direct LLM and the original MDAgents workflow were evaluated zero-shot on the fixed test partition using compact summaries of all five modality groups.

\paragraph{Raw-video MedAgent-Pro.}
For the raw-video baseline, we used MedAgent-Pro at \href{https://github.com/jinlab-imvr/MedAgent-Pro/commit/cea2721795bf748516b98fb959f581b502b1d186}{\texttt{cea2721}}. Six uniformly sampled frames were tiled into a contact-sheet JPEG. A task-level Planner generated four visual indicators. For each case, the VLM Analyzer evaluated each indicator, and the Summary Module reduced its output to Yes/No evidence. The Pro-Decider then produced normalized indicator weights and a decision threshold. Gemini~2.5~Flash was used with deterministic decoding. MedAgent-Pro received sampled video frames; in the aligned A+V experiment, Direct LLM and MDAgents received feature summaries. Comparisons across these input formats reflect the complete evaluated systems and do not isolate the effect of the agent workflow.

\paragraph{MCMoE.}
We include MCMoE~\citep{xu2026mcmoe} with its published ordinal score head. The evaluated runner requires sequence inputs, which are incompatible with the Mental Health tabular interface. Its Kaggle audio--visual runner also lacks a video feature stream. These two configurations are marked as \emph{n/a} in Table~\ref{tab:main-results}.

\paragraph{Unreported configurations.}
Dashes indicate configurations without reported results. The current MulT interface requires modality-specific feature sequences. No evaluated adaptation is available for Mental Health's tabular covariates or the single visual stream in the Kaggle video-only task. Aligned Kaggle audio--visual results also require a two-modality sequence interface that was not implemented in the evaluated runner. No video-only Flex-MoE result is available, and raw-video MDAgents requires a visual-input implementation. The official MPDD-2025 baseline and the MFCC-based Logistic Regression and RBF SVM audio baselines are detailed in Appendices~\ref{app:mpdd2025_baseline} and~\ref{app:kaggle}, respectively.

\FloatBarrier
\subsection{Search Spaces and Run Configurations}
\label{app:shared_controls}

\label{app:run_details}
\paragraph{Run configurations.}
Each run specifies the eligible input sources, task and selection metric, candidate evaluation budget $B$, and enabled review components. Candidate selection uses a designated development split or cross-validation within training data. Subject identifiers define validation groups where available, while Mental Health uses stratified row-level splits as described in Appendix~\ref{app:mental_health}. All recipe--predictor candidates within a task share the checked dataset adapter, validation splits, and selection metric. Review uses candidate rankings, prediction diagnostics, and execution logs. Appendix~\ref{app:rsi_configuration} details initialization and revision budgets, selection settings, and saved run records.

For the reported supplementary classical baselines, the classification search space was Logistic Regression $C\in\{0.01,0.1,1,10\}$ and RBF SVM $C\in\{0.1,1,10\}$ with $\gamma\in\{\texttt{scale},0.01,0.1\}$. Regression used Ridge with $\alpha\in\{0.1,1,10,100\}$ and RBF SVR with $C\in\{0.1,1,10\}$, $\gamma\in\{\texttt{scale},0.01\}$, and $\epsilon\in\{0.1,1\}$. Model selection was confined to the training/development data specified for each benchmark.

All API baselines used zero-shot prompts and deterministic decoding (temperature~0). Prompts used in the later extension experiments removed sample identifiers before submission. Except where specified in Appendix~\ref{app:baseline_implementations}, Direct LLM used three workers, one bounded retry, and at most 256 output tokens. The MPDD-2026 Direct LLM run instead used the 1,024-token limit specified in Appendix~\ref{app:mpdd2026_baselines}. Original-workflow MDAgents retained its basic/intermediate/advanced adaptive routing, with at most 256 tokens per API call and one retry. It used 2 workers for E-DAIC, 12 for Kaggle\_public, and 16 for Mental Health Multimodal. API responses were cached by a hash of the model, prompt, and sample to prevent duplicate billing during reruns.

\FloatBarrier
\subsection{Reproducibility and Compute}
\label{app:compute}

All repository revisions were pinned as reported above. GPU-based neural experiments used PyTorch~2.5.1 and one NVIDIA V100-SXM2 GPU with 32~GB memory. MPDD\_2026 Flex-MoE and the supplementary deep-learning jobs requested eight CPU cores and 64~GB host memory; each MPDD 2025 array task requested four CPU cores and 24~GB host memory. The retained MPDD\_2026 MulT prediction run used CPU, and API baselines did not use a GPU. Seeds were 1111 for MulT, 0/1/2 for Flex-MoE, 3407 for MPDD 2025, 2026 for supplementary classical splits and searches, and 37 for the Kaggle audio experiment.

\clearpage
\section{Theoretical Analysis}
\label{sec:mrtheory-appendix}

In \method{}'s RSI loop, each retained pipeline becomes the reference for subsequent modifications. An unreliable replacement can therefore affect later verification, while a conservative acceptance margin can postpone useful changes. We analyze EGE's inheritance rule to establish how it preserves measured progress and how much performance can be lost by rejecting small gains. The analysis allows proposals to depend on accumulated experience, the expanding candidate space, and revised research memory. We proceed in two stages. Section~\ref{app:rsi_stored_score_guarantees} derives the effect of one verification decision and propagates it through the comparison sequence. Section~\ref{app:rsi_population_improvement} combines that result with a uniform evaluation-error condition to bound population performance.

\FloatBarrier
\subsection{Sequential Inheritance Process}
\label{app:rsi_inheritance_process}

Fix corpus $\mathcal D$, development protocol $\Pi$, and task metric $\mu$, with larger scores preferred. Let $\Theta$ denote the class of complete pipelines that the run can evaluate. For a pipeline $\theta\in\Theta$, write $\widehat S(\theta)=\widehat S_\mu(\theta;\mathcal D)\in\mathbb R$ for its finite development score under the fixed protocol $\Pi$ in~\eqref{eq:joint-evaluation}. Scores are computed once and stored. We consider additive revisions that preserve the eligibility of evaluated pipelines. The incumbent changes only through the verification rule. A source exclusion that invalidates candidates starts a new analysis segment from the reselected valid incumbent.

\begin{mrdefinition}[Sequential inheritance]
\label{def:mrtheory-incumbent}
Let $J_0$ be the first valid reference pipeline. Index valid candidate comparisons by $n=1,2,\ldots$, including comparisons during initialization. Let $\theta_n$ be the candidate at comparison $n$, $\tau_n\geq0$ its acceptance margin, and $J_n$ the pipeline retained after that comparison. The verifier applies
\begin{equation}
J_n=\begin{cases}
\theta_n, & \widehat S(\theta_n)>\widehat S(J_{n-1})+\tau_n,\\
J_{n-1}, & \text{otherwise}.
\end{cases}
\label{eq:mrtheory-rule}
\end{equation}
Failed or invalid evaluations leave the incumbent unchanged and are omitted from this comparison index. Let $\mathcal G_n=\{J_0,\theta_1,\ldots,\theta_n\}$ be the evaluated set, with $\mathcal G_0=\{J_0\}$. Let $\overline\tau_n=\max_{1\leq k\leq n}\tau_k$ be the largest applied margin, with $\overline\tau_0=0$. Let $M_n$ count incumbent replacements through comparison $n$, with $M_0=0$.
\end{mrdefinition}

At the start of review round $t$, let $n_t$ count completed valid comparisons, including initialization. Then the incumbent in the main method is $I_t=J_{n_t}$. This indexing connects candidate-level verification to successive RSI states. The next candidate and margin can depend on all previous outcomes. In particular, updating revision history $\mathcal L_t$ or research memory $\mathcal K_t$ can change the proposal sequence without changing the inheritance rule.

\FloatBarrier
\subsection{Guarantees for Retained Development Scores}
\label{app:rsi_stored_score_guarantees}

\begin{mrproposition}[Retention under adaptive proposals]
\label{prop:mrtheory-margin}
For every sequence of valid comparisons governed by~\eqref{eq:mrtheory-rule}, the following properties hold.
\begin{enumerate}
\item[(i)] Stored incumbent scores satisfy $\widehat S(J_n)\geq\widehat S(J_{n-1})$ for every $n\geq1$.
\item[(ii)] For every $n\geq0$, the gap to the best evaluated pipeline satisfies
\begin{equation}
\boxed{0\leq\max_{\theta\in\mathcal G_n}\widehat S(\theta)-\widehat S(J_n)\leq\overline\tau_n.}
\label{eq:mrtheory-pf-hatbest}
\end{equation}
\item[(iii)] Suppose all stored scores lie in a common interval of width $W>0$. If every margin through comparison $n$ is at least $\tau_{\min}>0$, then $M_n<W/\tau_{\min}$.
\end{enumerate}
\end{mrproposition}

\begin{proof}
\textbf{Step 1: monotonicity of the retained score.} Fix a comparison $n\geq1$. If the candidate is accepted,~\eqref{eq:mrtheory-rule} gives
\begin{equation*}
\widehat S(J_n)=\widehat S(\theta_n)
>\widehat S(J_{n-1})+\tau_n
\geq\widehat S(J_{n-1}),
\end{equation*}
where the last inequality uses $\tau_n\geq0$. If the candidate is rejected, the rule gives $J_n=J_{n-1}$, so $\widehat S(J_n)-\widehat S(J_{n-1})=0$. Thus each comparison leaves the retained score unchanged or increases it. Applying this conclusion at comparisons $1,\ldots,n$ proves part (i):
\begin{equation}
\widehat S(J_0)\leq\widehat S(J_1)\leq\cdots\leq\widehat S(J_n).
\label{eq:mrtheory-pf-monotone}
\end{equation}

\textbf{Step 2: recurrence for the score gap.} To prove part (ii), define $d_n$ as the difference between the best evaluated score and the retained score after comparison $n$:
\begin{equation}
d_n=\max_{\theta\in\mathcal G_n}\widehat S(\theta)-\widehat S(J_n).
\label{eq:mrtheory-gap-definition}
\end{equation}
The retained pipeline belongs to $\mathcal G_n$, so $d_n\geq0$. At comparison $n\geq1$, let $g_n=\widehat S(\theta_n)-\widehat S(J_{n-1})$ denote the candidate's measured gain. Rearranging these two definitions gives $\max_{\theta\in\mathcal G_{n-1}}\widehat S(\theta)=\widehat S(J_{n-1})+d_{n-1}$ and $\widehat S(\theta_n)=\widehat S(J_{n-1})+g_n$. Since $\mathcal G_n=\mathcal G_{n-1}\cup\{\theta_n\}$, the best evaluated score updates as
\begin{align*}
\max_{\theta\in\mathcal G_n}\widehat S(\theta)
&=\max\!\left\{\max_{\theta\in\mathcal G_{n-1}}\widehat S(\theta),\widehat S(\theta_n)\right\}\\
&=\max\!\left\{\widehat S(J_{n-1})+d_{n-1},\widehat S(J_{n-1})+g_n\right\}\\
&=\widehat S(J_{n-1})+\max\{d_{n-1},g_n\}.
\end{align*}
The second equality substitutes the previous gap and the candidate gain. The third extracts the common incumbent score from both arguments of the maximum. Substituting this expression into~\eqref{eq:mrtheory-gap-definition} gives a common update for both verification outcomes:
\begin{equation}
\begin{aligned}
d_n
&=\widehat S(J_{n-1})+\max\{d_{n-1},g_n\}-\widehat S(J_n)\\
&=\max\{d_{n-1},g_n\}-\bigl[\widehat S(J_n)-\widehat S(J_{n-1})\bigr].
\end{aligned}
\label{eq:mrtheory-gap-update}
\end{equation}
If the candidate is accepted, $J_n=\theta_n$ and $g_n>\tau_n\geq0$. The bracket in~\eqref{eq:mrtheory-gap-update} therefore equals $g_n$, yielding
\begin{align*}
d_n
&=\max\{d_{n-1},g_n\}-g_n\\
&=\max\{d_{n-1}-g_n,g_n-g_n\}\\
&=\max\{d_{n-1}-g_n,0\}\\
&\leq\max\{d_{n-1}-\tau_n,0\} \tag{acceptance condition}\\
&\leq\max\{d_{n-1},0\}=d_{n-1}.
\end{align*}
Subtracting $g_n$ from a maximum subtracts it from both arguments. The first inequality uses $g_n>\tau_n$ and monotonicity of the maximum. The second uses $\tau_n\geq0$, and the final equality uses $d_{n-1}\geq0$. If the candidate is instead rejected, $J_n=J_{n-1}$ and $g_n\leq\tau_n$. The bracket in~\eqref{eq:mrtheory-gap-update} is now zero, so
\begin{align*}
d_n
&=\max\{d_{n-1},g_n\}-0\\
&=\max\{d_{n-1},g_n\}\\
&\leq\max\{d_{n-1},\tau_n\}.
\end{align*}
The last inequality substitutes $g_n\leq\tau_n$ into the second argument of the maximum. In the accepted case, we also have $d_n\leq d_{n-1}\leq\max\{d_{n-1},\tau_n\}$. Both outcomes therefore give
\begin{equation}
d_n\leq\max\{d_{n-1},\tau_n\}.
\label{eq:mrtheory-gap-recursion}
\end{equation}

\textbf{Step 3: propagation through the comparison sequence.} We apply~\eqref{eq:mrtheory-gap-recursion} inductively. For the base case, $\mathcal G_0=\{J_0\}$ gives
\begin{equation*}
d_0=\max_{\theta\in\{J_0\}}\widehat S(\theta)-\widehat S(J_0)
=\widehat S(J_0)-\widehat S(J_0)=0=\overline\tau_0.
\end{equation*}
Suppose $d_{n-1}\leq\overline\tau_{n-1}$ for some $n\geq1$. The largest applied margin satisfies $\overline\tau_n=\max\{\overline\tau_{n-1},\tau_n\}$. Using the recurrence and then the induction hypothesis yields
\begin{align*}
0\leq d_n
&\leq\max\{d_{n-1},\tau_n\}\\
&\leq\max\{\overline\tau_{n-1},\tau_n\} \tag{induction hypothesis}\\
&=\overline\tau_n.
\end{align*}
This proves $0\leq d_n\leq\overline\tau_n$ for every $n\geq0$. Substituting the definition of $d_n$ gives part (ii):
\begin{equation*}
0\leq\max_{\theta\in\mathcal G_n}\widehat S(\theta)-\widehat S(J_n)\leq\overline\tau_n.
\end{equation*}

\textbf{Step 4: number of replacements.} To establish part (iii), we sum the gains over accepted replacements. The claim holds when $M_n=0$ because $W/\tau_{\min}>0$. For $M_n\geq1$, let $k_1<\cdots<k_{M_n}$ be the comparison indices at which replacements occur. Denote the successive retained pipelines by $J^{(0)}=J_0$ and $J^{(m)}=J_{k_m}$ for $m=1,\ldots,M_n$. Rejections between these indices leave the pipeline unchanged, so $J_{k_m-1}=J^{(m-1)}$. Acceptance at $k_m$ gives $J^{(m)}=\theta_{k_m}$. Substituting both identities into the score difference yields
\begin{align*}
\widehat S(J^{(m)})-\widehat S(J^{(m-1)})
&=\widehat S(\theta_{k_m})-\widehat S(J_{k_m-1})\\
&>\tau_{k_m}\geq\tau_{\min}.
\end{align*}
Since $J^{(M_n)}=J_n$, summing over all replacements gives
\begin{align*}
\widehat S(J_n)-\widehat S(J_0)
&=\widehat S(J^{(M_n)})-\widehat S(J^{(0)})\\
&=\sum_{m=1}^{M_n}\bigl[\widehat S(J^{(m)})-\widehat S(J^{(m-1)})\bigr] \tag{telescoping}\\
&>\sum_{m=1}^{M_n}\tau_{k_m} \tag{acceptance rule}\\
&\geq\sum_{m=1}^{M_n}\tau_{\min}\\
&=M_n\tau_{\min}.
\end{align*}
The telescoping sum cancels each intermediate retained score. The strict inequality sums the acceptance conditions, and the next inequality uses $\tau_{k_m}\geq\tau_{\min}$. Write the common score interval as $[a,a+W]$, where $a\in\mathbb R$ is its lower endpoint. Then $\widehat S(J_n)\leq a+W$ and $\widehat S(J_0)\geq a$, so
\begin{equation*}
M_n\tau_{\min}<\widehat S(J_n)-\widehat S(J_0)
\leq(a+W)-a=W.
\end{equation*}
Dividing by the positive quantity $\tau_{\min}$ gives $M_n<W/\tau_{\min}$, proving part (iii).
\end{proof}

\FloatBarrier
\subsection{Population Performance under Evaluation Error}
\label{app:rsi_population_improvement}

Proposition~\ref{prop:mrtheory-margin} establishes properties of stored development scores. Relating inherited changes to population performance additionally requires control of evaluation error. Let $S(\theta)\in\mathbb R$ denote the population score of pipeline $\theta$ under the same task metric and evaluation target.

\begin{mrassumption}[Uniform evaluation accuracy]
\label{ass:mrtheory-score}
For an error tolerance $\eta\geq0$, assume the event
\begin{equation}
\mathcal U_\eta=\Bigl\{\sup_{\theta\in\Theta}\bigl|\widehat S(\theta)-S(\theta)\bigr|\leq\eta\Bigr\}
\label{eq:mrtheory-uniform}
\end{equation}
holds. The bound covers every pipeline the run can evaluate, including pipelines proposed after observing previous results.
\end{mrassumption}

\begin{mrproposition}[Population consequences of inheritance]
\label{prop:mrtheory-population}
On $\mathcal U_\eta$, the following properties hold simultaneously for all comparisons.
\begin{enumerate}
\item[(i)] For every $n\geq0$,
\begin{equation}
S(J_n)\geq\max_{\theta\in\mathcal G_n}S(\theta)-2\eta-\overline\tau_n.
\label{eq:mrtheory-population-gap}
\end{equation}
\item[(ii)] Every accepted candidate $\theta_n$ satisfies
\begin{equation}
S(\theta_n)-S(J_{n-1})>\tau_n-2\eta.
\label{eq:mrtheory-population-gain}
\end{equation}
\end{enumerate}
Consequently, if $\tau_n\geq2\eta$ at every accepted comparison, every replacement strictly improves the population score. The retained population scores are then nondecreasing.
\end{mrproposition}

\begin{proof}
\textbf{Step 1: population gap of the retained pipeline.} Fix a realization on which $\mathcal U_\eta$ holds. Assumption~\ref{ass:mrtheory-score} gives $|S(\theta)-\widehat S(\theta)|\leq\eta$ for every $\theta\in\Theta$. Writing the absolute-value bound as $-\eta\leq S(\theta)-\widehat S(\theta)\leq\eta$ and adding $\widehat S(\theta)$ to each term yields
\begin{equation}
\widehat S(\theta)-\eta\leq S(\theta)\leq\widehat S(\theta)+\eta.
\label{eq:mrtheory-error-interval}
\end{equation}
To bound the retained pipeline's population gap, fix $n\geq0$ and choose $\theta_n^\star\in\operatorname*{arg\,max}_{\theta\in\mathcal G_n}S(\theta)$. Such a maximizer exists because $\mathcal G_n$ is finite and nonempty. The error interval gives $S(\theta_n^\star)-\widehat S(\theta_n^\star)\leq\eta$ and $\widehat S(J_n)-S(J_n)\leq\eta$. Adding and subtracting the two stored scores makes both error terms explicit:
\begin{align*}
\max_{\theta\in\mathcal G_n}S(\theta)-S(J_n)
&=S(\theta_n^\star)-S(J_n)\\
&=\bigl[S(\theta_n^\star)-\widehat S(\theta_n^\star)\bigr]\\
&\quad+\bigl[\widehat S(\theta_n^\star)-\widehat S(J_n)\bigr]\\
&\quad+\bigl[\widehat S(J_n)-S(J_n)\bigr]\\
&\leq\eta+\bigl[\widehat S(\theta_n^\star)-\widehat S(J_n)\bigr]+\eta \tag{uniform error}\\
&=2\eta+\widehat S(\theta_n^\star)-\widehat S(J_n)\\
&\leq2\eta+\max_{\theta\in\mathcal G_n}\widehat S(\theta)-\widehat S(J_n)\\
&=2\eta+d_n\\
&\leq2\eta+\overline\tau_n. \tag{retained-score bound}
\end{align*}
The first inequality bounds the two error terms by $\eta$ each. The second uses $\theta_n^\star\in\mathcal G_n$. The next equality is the gap definition~\eqref{eq:mrtheory-gap-definition}, and the final inequality applies Proposition~\ref{prop:mrtheory-margin}(ii). Moving $S(J_n)$ to the right and $2\eta+\overline\tau_n$ to the left gives part (i):
\begin{equation*}
S(J_n)\geq\max_{\theta\in\mathcal G_n}S(\theta)-2\eta-\overline\tau_n.
\end{equation*}

\textbf{Step 2: population gain of an accepted replacement.} To obtain part (ii), suppose $\theta_n$ is accepted at comparison $n$. The lower endpoint of~\eqref{eq:mrtheory-error-interval} bounds the candidate score. For the previous incumbent, negating the upper endpoint reverses its inequality. Thus
\begin{equation*}
S(\theta_n)\geq\widehat S(\theta_n)-\eta,
\qquad
-S(J_{n-1})\geq-\widehat S(J_{n-1})-\eta.
\end{equation*}
Adding these bounds and then applying the strict acceptance condition gives
\begin{align*}
S(\theta_n)-S(J_{n-1})
&\geq\bigl[\widehat S(\theta_n)-\eta\bigr]+\bigl[-\widehat S(J_{n-1})-\eta\bigr]\\
&=\widehat S(\theta_n)-\widehat S(J_{n-1})-2\eta\\
&>\bigl[\widehat S(J_{n-1})+\tau_n\bigr]-\widehat S(J_{n-1})-2\eta\\
&=\tau_n-2\eta.
\end{align*}
The strict inequality substitutes $\widehat S(\theta_n)>\widehat S(J_{n-1})+\tau_n$. The final equality cancels the previous incumbent's stored score and proves part (ii). If $\tau_n\geq2\eta$, acceptance also gives $J_n=\theta_n$, so
\begin{equation*}
S(J_n)-S(J_{n-1})
=S(\theta_n)-S(J_{n-1})
>\tau_n-2\eta\geq2\eta-2\eta=0.
\end{equation*}
Each accepted replacement therefore strictly improves the population score. Rejection gives $J_n=J_{n-1}$ and hence $S(J_n)-S(J_{n-1})=0$. Applying these two cases along the comparison sequence proves that the retained population scores are nondecreasing under the stated margin condition.
\end{proof}

The event $\mathcal U_\eta$ controls every pipeline in $\Theta$ at once. On that event, the proof applies to every adaptively chosen candidate and every comparison index. If $\mathcal U_\eta$ holds with probability at least $1-\delta$ for $\delta\in(0,1)$, both population conclusions hold simultaneously with at least that probability. At review boundary $t$, substituting $n=n_t$ and $I_t=J_{n_t}$ transfers the conclusions to the inherited pipeline. No additional probability bound is needed for each round.

\FloatBarrier
\subsection{Implications and Scope}
\label{app:rsi_theory_implications}

Substituting $I_t=J_{n_t}$ applies the results to the pipeline inherited by each RSI successor. A rejected candidate's score can exceed the incumbent's by at most the applied margin at rejection. Later inherited scores can only increase, so repeated rejections do not accumulate a separate loss at every round. The largest margin enters the bound once, even when experience and memory change later proposals. With zero margins, the verifier retains a maximum-score pipeline among those evaluated and preserves the incumbent on ties. The classic one-standard-error rule instead selects a simpler candidate within one standard error of the best score. It can select a lower-scoring pipeline and is a different selection procedure from~\eqref{eq:mrtheory-rule} with $\tau_n=0$.

\paragraph{Margin choice and scope.} EGE sets the margin using the incumbent's subject-bootstrap score variation, as defined in~\eqref{eq:method-retention}. The one-standard-error setting in Table~\ref{tab:selection-rule} uses 1,000 resamples of whole subjects and their saved predictions. This estimates subject-sampling variation on the selection set. It does not establish the uniform error bound $\eta$ or the condition $\tau_n\geq2\eta$. Monotonicity and the stored-score gap bound apply to finite real scores, including negative RMSE. The replacement-count bound additionally requires a finite score interval. These results characterize verification and pipeline inheritance along the realized search. The effect of research-memory updates on the quality of future proposals is an empirical property of the improver.

\clearpage
\section{Implementation Details}
\label{app:method_details}

\FloatBarrier
\subsection{Adapter Construction and Execution Checks}
\label{app:method_execution}

\textbf{Adapter construction.} GSC's inspection report records available sources, file schemas, subject--label links, feature shapes, and value statistics. The input specification separates predictive observations from measurements used to define the target. The first checked adapter output fixes ordered subject IDs, row IDs, and targets for each split. Subsequent recipe outputs preserve this ordering while allowing feature widths to change. Execution checks detect incompatible shapes, nonfinite values, and record mismatches. Recipe probes compare configurations expected to yield different feature widths. Failed checks guide adapter repair. Passing code is retained and checked before reuse.

\FloatBarrier
\subsection{Review and Revision Workflow}
\label{app:rsi_schedule}

\phantomsection
\label{sec:method:audit}
\textbf{Experience records and role packets.} Each attempted recipe--predictor pair records its configuration, predictions, score, runtime, and any execution error. Candidate fingerprints identify attempted pairs so additive revisions skip repeated evaluations. Shared review context contains the incumbent, measured ranking, prediction diagnostics, search coverage, run constraints, and remaining budget. Planner packets expose the current recipe and predictor spaces. Engineer packets contain adapter information, feature widths, and loading errors. Reviewer packets contain evaluated candidates and prediction diagnostics. Roles conduct bounded read-only checks. The Chair receives their reports together with the authoritative ranking, diagnostics, and search coverage. Active research memory is supplied directly to the roles and Chair when enabled. A programmatic check rejects packets containing test fields.

\textbf{Revision targets and history reuse.} Executable targets comprise recipe-space additions, predictor-space additions, and supported source exclusions. Recipe additions use operations supported by the checked adapter. A verified exclusion removes the source and invalidates dependent scored candidates, including an affected incumbent. The system then reselects among valid candidates. LLM weights, role definitions, the adapter implementation, labels, and evaluation protocol remain fixed. Each round records the Chair's request, previous incumbent, execution status, selected candidate, replacement outcome, and evaluation count. Execution failures also enter the round record. All review roles receive earlier records. The Chair receives a history summary and is instructed to avoid repeating unsuccessful requests. The next proposer receives the request, current validated candidate, run constraints, and active memory when enabled. Full reports and tool checks are stored separately as review minutes.

The development loop separates proposing an experiment from evaluating its outcome and retaining its consequences. Initialization checks the adapter and evaluates the initial recipe--predictor combinations. The resulting incumbent $I_0$ and experimental feedback $\mathcal{E}_0$ provide the first review context. Revision history $\mathcal{L}_0$ and research memory $\mathcal{K}_0$ start empty. Table~\ref{tab:rsi_execution} gives the order within each subsequent round. The candidate space $\mathcal{Q}_t$ persists across rounds. Valid additions remain available even when the incumbent is unchanged.

\begin{table}[!htbp]
\caption{Execution order of a memory-enabled improvement round. Learning also records nonpromoting and failed revisions.}
\label{tab:rsi_execution}
\centering
\begingroup
\mrAppendixTableStyle
\begin{tabular}{>{\raggedright\arraybackslash}p{0.16\linewidth}>{\raggedright\arraybackslash}p{\dimexpr0.84\linewidth-4\tabcolsep\relax}}
\toprule
\rowcolor{mrAppHeader}
\textbf{Stage} & \textbf{Operation and retained consequence} \\
\midrule
\rowcolor{mrAppBand}
\textbf{Review} & Roles analyze current feedback, prior outcomes, and active memory. The Chair issues a revision request or ends review. \\
\textbf{Construct} & The proposer turns a supported request into recipe or predictor additions. Source exclusions remove dependent candidates before reselection. \\
\rowcolor{mrAppBand}
\textbf{Evaluate} & CPE evaluates untried pairs under the fixed protocol and remaining candidate budget. Each trial records predictions, its score, or an execution error. \\
\textbf{Verify} & The configured selection rule determines incumbent replacement. The round record retains the request and measured outcome. \\
\rowcolor{mrAppBand}
\textbf{Learn} & Distillation proposes conditional lessons. Reconciliation adds, revises, or retires memory entries using the recorded evidence. \\
\textbf{Re-enter} & The next round receives the expanded space, retained pipeline, revision history, and updated memory. The loop ends on Chair acceptance or a budget limit. \\
\bottomrule
\end{tabular}
\endgroup
\end{table}

\textbf{Replacement and learning decisions.} Pipeline replacement uses the measured candidate score and the configured acceptance rule. Memory revision uses the completed round record and supporting experimental evidence. An unsuccessful candidate can therefore contribute a lesson without replacing the incumbent. Terminal review outcomes also pass through memory processing before the run ends. The next round reads only active memory entries. Retired entries remain in the audit record.

\FloatBarrier
\subsection{Research-Memory Implementation}
\label{app:research_memory}

The memory implementation follows the distillation and reconciliation cycle of RSIAgent~\citep{zhu2026rsiagent}. It maintains a structured bank within one dataset, task, and run. The bank records its scope, active entries, evidence, and update history. Each lesson identifies its applicable conditions and a proposed next experiment. Table~\ref{tab:research_memory_schema} summarizes the stored fields.

\begin{table}[!htbp]
\caption{Research-memory fields. Numeric evidence remains attached to the underlying execution records.}
\label{tab:research_memory_schema}
\centering
\begingroup
\mrAppendixTableStyle
\begin{tabular}{>{\raggedright\arraybackslash}p{0.25\linewidth}>{\raggedright\arraybackslash}p{\dimexpr0.75\linewidth-4\tabcolsep\relax}}
\toprule
\rowcolor{mrAppHeader}
\textbf{Field} & \textbf{Meaning} \\
\midrule
\rowcolor{mrAppBand}
\texttt{id} & Stable identity used to revise or retire an existing entry. \\
\texttt{condition} & Data and pipeline conditions under which the lesson applies. \\
\rowcolor{mrAppBand}
\texttt{lesson} & An interpretation of the recorded experiment, marked as model synthesis. \\
\texttt{next\_action} & A proposed experiment for extending or checking the interpretation. \\
\rowcolor{mrAppBand}
\texttt{evidence\_ids} & References to measured candidate records, round outcomes, or diagnostics. \\
\texttt{status}, \texttt{versions} & Active or retired status, with the previous contents retained after each update. \\
\bottomrule
\end{tabular}
\endgroup
\end{table}

\textbf{Evidence selection.} The learning packet includes the previous and current incumbents when their records remain available. It also includes up to eight highest-scoring and four lowest-scoring valid candidates from the current round. Up to four execution failures supply failure evidence. Duplicate candidate identifiers collapse to a single record. The packet includes the round outcome, run constraints, and prediction diagnostics. Candidate evidence contains configurations, validity, scores, feature widths, and errors. Memory processing reuses these records without additional candidate evaluations.

\textbf{Distillation.} The distiller receives active memory and the current evidence packet. It returns at most four lesson drafts in a structured response. Each draft contains a condition, interpretation, next action, and one to six evidence identifiers. At least one identifier must refer to the current packet. The prompt distinguishes execution failure, valid nonpromotion, and incumbent replacement. It also requests a follow-up comparison for unresolved hypotheses. Each text field is limited to 600 characters. The requested output limit is 2,200 tokens.

\textbf{Reconciliation.} A second LLM call compares the drafts with active memory. It returns at most eight operations with an output limit of 1,600 tokens. An addition creates a new entry from a draft. A revision replaces an entry's contents and accumulates its evidence references. A retirement removes an entry from active context while preserving its versions. Revision and retirement operations must identify an existing active entry. Each target can appear only once in an update. An empty operation list leaves memory unchanged.

\textbf{Validation and persistence.} Programmatic checks validate response fields, evidence identifiers, target identities, and the active-memory limit. The default bank admits at most 16 active entries. A malformed operation rejects the complete update and retains the previous bank. Failed updates receive an explicit status. Successful updates are written atomically before becoming available to the next round. Each active entry exposes its latest three evidence references in the prompt. The persistent record retains the full reference history. Current execution evidence and run constraints take precedence over a conflicting memory interpretation.

\FloatBarrier
\subsection{Budget Accounting and Run Records}
\label{app:rsi_configuration}

\textbf{Candidate budget.} In review-enabled runs, the command-line initialization budget is $B_{\mathrm{init}}$. The reserve fraction $\rho$ adds evaluations for subsequent revisions. The total candidate cap is
\begin{equation}
B=B_{\mathrm{init}}+\left\lceil\rho B_{\mathrm{init}}\right\rceil.
\label{eq:app-rsi-budget}
\end{equation}
Here $\rho$ defaults to $0.25$ in the common development runner. Initialization can use at most $B_{\mathrm{init}}$ evaluations before review begins. Later rounds share the remaining total budget. With review disabled, the initial search can use the full cap $B$. The ledger counts attempted candidate evaluations, including failed evaluations. Distillation and reconciliation add at most two LLM calls per review iteration. Their API usage is recorded separately from candidate evaluations and model fits.

\textbf{Selection configuration.} The \texttt{incumbent} selection mode implements the sequential replacement rule in~\eqref{eq:method-retention}. Its uncertainty multiplier is configured through \texttt{promotion\_se}. The bootstrap estimate uses 1,000 subject resamples and retains all records of each sampled subject. The implementation also supports \texttt{one\_se\_of\_best} selection and an optional terminal reliability gate. Run configurations identify these choices separately. Final ensemble aggregation is an additional evaluation setting. Comparisons must specify both the in-loop selection rule and the final prediction rule.

\textbf{Memory configuration and artifacts.} Setting \texttt{research\_memory=evolving} enables both learning stages. Setting \texttt{research\_memory=off} retains the history-only loop. The default remains \texttt{off} for existing configurations. The common \texttt{run\_corpus} and MPDD-2025 \texttt{run\_method} development paths support the extension. The separate MPDD-2026 \texttt{run\_agent} path has not integrated memory learning. Each enabled run writes \texttt{research\_memory.json} with entry versions, evidence, operations, and update status. Round records identify the memory version used for the decision. Review minutes and the execution trajectory retain the associated reports and API usage.

\textbf{Experimental configurations.} Research-memory learning is disabled in the five-seed main benchmark and module ablations. The evolving-memory arms of the 54-run comparison and eight-round RSI study enable distillation and reconciliation, as detailed in Appendices~\ref{app:rsi_trajectories} and~\ref{app:rsi-eight-round}.

\clearpage
\subsection{System prompts}
\label{app:prompts}

This section lists the prompts behind every LLM call of the reported runs, read verbatim from the source of the code version used for the five-seed main results (Appendix~\ref{app:seed_study}); both base models receive the same text. User messages carry run-specific evidence, such as the dataset inspection, candidate scores and earlier decisions, serialized as JSON. The MPDD-2026 head proposer carries its instructions in the user message, so its template is shown with angle-bracketed placeholders for the run data. Braced fields in the corpus recipe prompt (\texttt{\{aggs\}}, \texttt{\{menu\}}) are filled with the corpus's aggregations and extractor menu, and the two \texttt{\%s} fields of the investigation clause with the list of read-only checks and the check budget. Prompts of configurations not reported in this paper (a single-agent reviewer, an operation agent, a modality prior and legacy scripts) are omitted.

\Needspace{10\baselineskip}
\paragraph{Interface construction.}

\begin{casebox}{interface}{Benchmark metric (corpora without a fixed selection metric)\mrpromptsource{merid/run\_method.py EVAL\_METRIC\_SYSTEM}}{lst:prompt-merid-run-method-py-eval-metric-system-benchmark-metric-corpora-without-a-fixed-selection-metric}
You are handed a clinical multimodal dataset. Determine WHAT THE PUBLISHED LITERATURE ON THIS CORPUS REPORTS as its headline evaluation metric. This is a question about the benchmark's convention — not about your preference, and not about what would make a model look good.
It matters because the metric a search OPTIMISES must be the metric the work is JUDGED by. Optimising an accuracy-flavoured score on a corpus with heavy class imbalance selects models that predict the base rate, which then lose on the balanced metric the field reports.
YOU MUST SEARCH before answering. The convention is a published fact about a specific benchmark, and recall of it is not reliable: asked about one corpus four times this step returned two different metrics, and the majority-dominated one then collapsed a 5-class task onto the majority label while still looking fine on the metric it had been chosen by. Find the paper or challenge page that defines the corpus and read what its results table reports. Begin with `SEARCHED: yes` only if the search actually returned something you could read. If it returned nothing usable - no results, a quota or tool limit, a fetch you were not allowed to make - then begin with `SEARCHED: no` and answer from recall. Both are acceptable answers; what is not acceptable is reporting a search you could not complete as though it had confirmed something, because an unverified answer is kept for the current run only and asked again next time, and that distinction is what the first word tells us. Answer on the FIRST line exactly: `METRIC: <name> - <why the field reports this one>` where <name> is one of the OPTIONS listed in the prompt. You may not define a metric of your own; note that `official_composite` means the challenge score ((acc+bal_acc)/2 + (macro_f1+weighted_f1)/2)/2 that some of these corpora publish as their headline column - name it when that IS what the corpus reports, and do not name it merely because it looks comprehensive. name one of the options given. Note on the regression options: they are written so that HIGHER IS BETTER, so an error metric appears negated (neg_rmse). Naming it still means the field reports RMSE - the sign is only the search convention.
\end{casebox}
\begin{casebox}{interface}{Feature-recipe space, corpus entry\mrpromptsource{merid/run\_corpus.py RECIPE\_SYSTEM}}{lst:prompt-merid-run-corpus-py-recipe-system-feature-recipe-space-corpus-entry}
You DESIGN the candidate feature-recipe SPACE for a SMALL-SAMPLE (p>>n) clinical multimodal task. Reason from what the inspection shows ACTUALLY EXISTS on this disk, then output ONLY a JSON list of candidate recipes.
A recipe is a flat dict. Each key is a SOURCE you will ask the loader for; its value is the choice for that source, or null to drop it. One reserved key: `agg`, one of {aggs}, saying how a variable-length series becomes one fixed-length vector.
{menu}WHERE THE NAMES COME FROM. If the corpus ships extracted features, the source names are the extractor directories, verbatim - do not invent names and do not carry names over from another corpus. If it ships RAW RECORDINGS, there are no directories to read, so you name the representations yourself (for example an acoustic key whose values are different MFCC or log-mel settings); the loader will be told to implement exactly the names you use here, so keep them descriptive and stable.
WHAT A GOOD SPACE COVERS. The point is to span the plausible range, not to bet on one end of it. Include each source ALONE (that is how the search learns which one carries the signal, and on the corpora seen so far the answer has differed between tasks on the SAME data), the combinations worth trying, both a high-dimensional and a low-dimensional option where the corpus offers both - at small n a several-thousand-column source often contributes variance rather than signal, and a twenty-column one can be too coarse - and at least one temporal-dynamics (meanstd_delta) candidate. SOURCES THAT DEFINE THE LABEL RATHER THAN PREDICT IT: some candidates are not evidence about the outcome, they ARE the instrument the outcome was assigned with - a questionnaire the diagnosis was made from, a score the classes were thresholded out of. Include one and the search returns a near-perfect number that measures the label's own definition rather than any ability to detect it. Say which of your sources you believe sit UPSTREAM of the label and why, and make sure the space still contains recipes that answer the question WITHOUT them, so both readings are available. Where the sources are COLUMN GROUPS of one table, cover every column between them and include one recipe with all groups on: a group left out of the space can never be found, and on one corpus the omitted group was the one carrying the signal. Include 8-14 DISTINCT recipes; a leak-free search picks the winner, so propose a good diverse SPACE rather than one guess. Output ONLY the JSON list.
\end{casebox}
\begin{casebox}{interface}{Feature-recipe space, MPDD-2025 entry\mrpromptsource{merid/run\_method.py RECIPE\_SPACE\_SYSTEM}}{lst:prompt-merid-run-method-py-recipe-space-system-feature-recipe-space-mpdd-2025-entry}
You DESIGN the candidate feature-recipe SPACE for a SMALL-SAMPLE (p>>n) multimodal depression task. Reason from the available extractors + their content stats and the grounding, then output ONLY a JSON list of candidate recipes. Each recipe is a dict: {"audio": "wav2vec"|"opensmile"|"mfccs"|null, "video": "densenet"|"resnet"|"openface"|null, "agg": "meanstd"|"meanstd_minmax"|"meanstd_p25p75"|"meanstd_delta", "ind_emb": bool, "meta": bool}. meta = structured personality/context (big5+family+disease, the 'P'); ind_emb = the individualEmbedding vector; null audio/video DROPS that modality; 'meanstd_delta' adds frame-to-frame dynamics. COVER modality-combination ablations (audio-only, video-only, personality-only, A+V, A+V+P) AND the extractor/aggregation choices worth testing — including at least one TEMPORAL-DYNAMICS ('meanstd_delta') candidate, since plain mean/std discards how the signal evolves. Include 8-14 DISTINCT recipes. An OOF search picks the winner — propose a good, diverse SPACE, not one guess. Output ONLY the JSON list.
\end{casebox}
\begin{casebox}{interface}{Classifier space\mrpromptsource{merid/run\_method.py CLF\_SPACE\_SYSTEM}}{lst:prompt-merid-run-method-py-clf-space-system-classifier-space}
You DESIGN the candidate CLASSIFIER space for a SMALL-SAMPLE (p>>n) multimodal depression task. Reason from the data regime and the grounding, then output ONLY a JSON list of candidate scikit-learn configs. Each config: {"model": "svc_linear"|"svc_rbf"|"logreg"|"gb"|"rf"|"et", "k": 32|64|128|256 (SelectKBest features), "C": 0.03..1.0 (svc/logreg), "lr": 0.03..0.1 (gb), "depth": 1..3 (gb)}. Span model FAMILIES and regularisation strengths appropriate for p>>n at n~100-350. Include 8-14 DISTINCT configs. A robust repeated-CV OOF search picks the winner — propose a good, diverse SPACE, not one guess. Output ONLY the JSON list.
\end{casebox}
\begin{casebox}{interface}{Regressor space\mrpromptsource{merid/run\_method.py REG\_SPACE\_SYSTEM}}{lst:prompt-merid-run-method-py-reg-space-system-regressor-space}
You DESIGN the candidate REGRESSOR space for a SMALL-SAMPLE (p>>n) multimodal depression-severity task: predict a continuous clinical score, not a class. Reason from the data regime and the grounding, then output ONLY a JSON list of candidate scikit-learn configs. Each config: {"model": "ridge"|"svr_linear"|"svr_rbf"|"gbr"|"rfr"|"etr", "k": 32|64|128|256 (SelectKBest by f_regression), "alpha": 1..300 (ridge), "C": 0.03..3.0 (svr), "eps": 0.5..3.0 (svr), "lr": 0.03..0.1 (gbr), "depth": 1..3 (gbr)}. THE REGIME IS THE WHOLE PROBLEM HERE. At n~150-350 with p in the hundreds or thousands, an under-regularised fit interpolates the training scores and returns a NEGATIVE out-of-sample R2 — worse than predicting the training mean for every subject. The published baselines on this task sit exactly there (R2 from -0.08 to -0.46, CCC within +/-0.05 of zero), so a space that leans toward strong shrinkage is not timidity, it is the only region where anything beats a constant. Span shrinkage strengths across at least two orders of magnitude, and include at least one very heavily regularised candidate. Include 8-14 DISTINCT configs. Output ONLY the JSON list.
\end{casebox}
\begin{casebox}{interface}{Model-code generation decision\mrpromptsource{merid/run\_method.py CODEGEN\_SYSTEM}}{lst:prompt-merid-run-method-py-codegen-system-model-code-generation-decision}
You decide whether the STANDARD scikit-learn families (linear/RBF SVM, logistic regression, gradient boosting, random forest, extra trees — each behind VarianceThreshold + StandardScaler + SelectKBest) are enough for the observed regime, or whether a method they cannot express would suit it better.
Answer on the FIRST line: `CODEGEN: no — <why the standard families already cover this regime>` or `CODEGEN: yes — <the specific property of THIS regime they cannot express>`.
Only if yes: output ONE ```python``` block defining `build()` that returns a fresh, UNFITTED scikit-learn-compatible estimator (needs .fit(X, y) and .predict(X)). Compose it from sklearn/numpy only — Pipeline, feature selection, decomposition (PLS/PCA), calibration, custom BaseEstimator, ensembling. It must be self-contained, take no arguments, and be cheap enough for repeated CV. Do NOT reach for a deep network: n is a few hundred.
Whatever you produce will be scored by the SAME leak-free repeated per-subject OOF as every standard candidate and only used if it actually wins — so propose it because the regime calls for it, not to look sophisticated. Saying `no` is a perfectly good answer.
\end{casebox}
\begin{casebox}{interface}{Data adapter derivation and repair\mrpromptsource{merid/adapter\_agent.py ADAPTER\_SYSTEM}}{lst:prompt-merid-adapter-agent-py-adapter-system-data-adapter-derivation-and-repair}
Follow the task-specific corpus schema and recipe contract supplied in the task. The MPDD layout and source names below are examples for MPDD tasks only; never impose them on another corpus. Implement a named pretrained encoder only using that actual encoder or its verified precomputed features; reject unsupported recipes instead of returning a proxy under the original encoder's name. You are an ML data engineer for a SMALL-SAMPLE (n~100-350) p>>n multimodal depression task. Write ONE Python function `load(recipe=None)` that reads the dataset into a fixed dict interface. CRITICAL — you do NOT choose the feature recipe. An OUTER loop calls load() repeatedly with different `recipe` dicts and picks the best by leak-free per-subject OOF. Implement the given recipe FAITHFULLY and DETERMINISTICALLY (same recipe -> byte-identical features every call). No 'vibes' choices, no picking extractors yourself. The `recipe` dict keys:   audio: 'wav2vec' | 'opensmile' | 'mfccs' | None -> read that Audio/<ext> dir, or DROP audio if None;   video: 'resnet' | 'densenet' | 'openface' | None -> read that Visual/<ext> dir, or DROP video if None;     (at least one of audio/video/ind_emb/meta is always present; skip a modality whose recipe value is None);   agg: per-segment frame aggregate (apply the SAME agg to audio and video), one of:        'meanstd' -> concat[mean,std];        'meanstd_minmax' -> concat[mean,std,min,max];        'meanstd_p25p75' -> concat[mean,std,25th percentile,75th percentile];        'meanstd_delta' -> concat[mean,std,mean(|frame-to-frame diff|)] (captures TEMPORAL DYNAMICS;          for a 1-frame segment the delta block is zeros);   ind_emb: bool -> if True, append the per-subject individualEmbedding vector to each segment;   meta: bool -> if True, append a STRUCTURED per-subject metadata vector parsed from     personalized_train.json (train) / personalized_test.json (test) — each a dict keyed by subject id     whose entry has big5_traits, family_factors, disease. Build meta = concat[ big5_traits values in the     fixed order (Extraversion,Agreeableness,Openness,Neuroticism,Conscientiousness),     family_factors['Financial_Stress'], family_factors['Family_Members'], float(disease) ] = 8 dims. This     is the personality/context signal (the 'P' in A+V+G+P). Missing subject -> zeros; same 8-dim layout for train & test. If recipe is None, default to {'audio':'wav2vec','video':'densenet','agg':'meanstd','ind_emb':True,'meta':False} (the simplest reliable recipe — used only for the smoke test; the outer loop then calls the full grid). CASE-INSENSITIVE dir match: the extractor subdir may be cased differently than the recipe value (e.g. recipe video 'resnet' can be the directory 'ResNet'); resolve the actual subdir by listing the parent and matching the recipe name case-insensitively, so EVERY recipe value loads a non-empty array. Build train_X/test_X = concat[ audio_agg (if audio not None), video_agg (if video not None), (individualEmbedding if ind_emb), (meta 8-dim if meta) ], 2D [n,p], SAME p for train and test. Use ONLY numpy, os, glob, json. `load(recipe)` MUST return a dict with keys: train_X (float [n,p]), one int label vector per task the SPEC asks for (e.g. train_binary, and train_ternary only where that task exists), test_X (float [m,p]), test_ids (list [m]), and ALWAYS train_subject (subject id per train segment) for grouped CV. If the test set has labels, ALSO return test_binary/test_ternary aligned to test_ids. Include train_quinary/test_quinary (5-class) ONLY if a quinary/penta label exists; include train_phq/test_phq ONLY if PHQ exists. Output ONLY one ```python``` block.
\end{casebox}

\Needspace{10\baselineskip}
\paragraph{Search.}
\begin{casebox}{search}{MPDD-2025 method selection (advisory)\mrpromptsource{merid/registry.py select\_method}}{lst:prompt-merid-registry-py-select-method-mpdd-2025-method-selection-advisory}
You are an AutoML research agent. Select one method from the executable catalog for this dataset profile. The invoking runner has explicitly declared which implementations it can execute. Reason from the observed regime and return the chosen method NAME alone on the FINAL line. Choose only an available catalog name.
\end{casebox}
\begin{casebox}{search}{MPDD-2026 head proposer: system prompt\mrpromptsource{merid/mpdd2026\_joint\_search.py HeadProposer.propose}}{lst:prompt-merid-mpdd2026-joint-search-py-headproposer-propose-mpdd-2026-head-proposer-system-prompt}
You propose the next unseen MPDD2026 experiment.
\end{casebox}
\begin{casebox}{search}{MPDD-2026 head proposer: user-message template\mrpromptsource{merid/mpdd2026\_joint\_search.py HeadProposer.propose}}{lst:prompt-merid-mpdd2026-joint-search-py-headproposer-propose-mpdd-2026-head-proposer-user-message-template}
Propose the NEXT NEW component experiment to evaluate, not the current winner. Previously evaluated components are already cached: never return a choice in HISTORY. The objective is the joint track=(best binary Macro-F1 + raw PHQ CCC + best hierarchical ternary Macro-F1)/3. Binary/PHQ component_score sums two metrics while hierarchical component_score is one metric; their raw scores are not directly comparable across families. Seek expected improvement to the cached joint score by exploring an unseen, affordable configuration. Both families can improve the result; do not keep proposing the family with the larger current sum. binary_rule selects where the binary cut on the predicted PHQ score comes from: "oof_tuned" fits it on the development out-of-fold, "label_cut" uses the cut this corpus's binary label is defined by. Both are seeded and already measured; propose a value for it deliberately rather than copying the seed. In gp_stats_pz, audio/video extractor aliases do not add audio/video feature blocks; however extractor-dependent pair availability/counts can still differ. Do not assume feature equivalence from mode names alone. The fixed metric contract is <metric contract>. Candidate families are binary_phq with parameters <binary_phq parameter grid>, or hierarchical_ternary with parameters <hierarchical_ternary parameter grid>. [only when axes are held fixed: The remaining axes are HELD FIXED for this run and must be returned exactly as given: <held-fixed axes>. A proposal that moves one is rejected. ]Use the shared fixed-fold development protocol. Output the complete JSON envelope {"head_family":...,"parameters":...} on the final nonempty line.
PROFILE:
<profile JSON>
HISTORY:
<history JSON>
\end{casebox}

\Needspace{10\baselineskip}
\paragraph{Review.}
\begin{casebox}{review}{Planner\mrpromptsource{merid/retrospective.py ROLE\_SYSTEM['planner']}}{lst:prompt-merid-retrospective-py-role-system-planner-planner}
You are the planner in a post-run review of an automated ML pipeline, and the decisions under review are your own: which metric this corpus is judged by, which feature sources exist and how they were grouped, which model families were worth trying, whether a custom method was needed. You are shown those decisions and the reasoning you gave for them.
Report, in at most 8 short lines: what you assumed, and which of those assumptions the run could have invalidated. Be specific about anything you flagged at the time and want checked — a source you suspected of sitting upstream of the label, a modality you expected to matter. Do not state a fact about the data from memory when you can check it.
THEN FINISH WITH ONE LINE BEGINNING 'NEXT:'. Name the one direction this search has not tried that you would spend the remaining budget on, drawn from the allowed sources and how they are grouped into recipes - a source combination absent from the candidates that scored well, a grouping you did not think of when you wrote the space. The review exists to make the NEXT round of candidates better, not only to catch mistakes, and a review that can only report that nothing is wrong has done half its job. You are shown what the search already covered: which model families were scored and how each did, which allowed families were never reached, which sources appear in the candidates that scored well, and whether the last round moved the incumbent. Be concrete - what to add - and give the number in your evidence that makes you expect it to beat the incumbent. The addition is scored on the same venue under the same promotion rule, so being wrong costs evaluations and retains the incumbent; the thing to avoid is not a wrong guess but an ungrounded one. If nothing in the evidence supports a direction, write 'NEXT: none' and say why. Never propose reading held-out data, changing the metric, the labels, the partition, or restoring an excluded source.
SUPPLIED SELECTION CONTRACT GOVERNS THIS REVIEW. Follow the supplied selection/revision policy and immutable source, task and grouping constraints. When it specifies strict improvement with ties retaining the incumbent, use that policy. Apply a one-standard-error rule only when the supplied contract explicitly specifies it AND computed standard-error evidence is supplied. If no selection rule or uncertainty computation is supplied, say it is unavailable; do not assume one was used. Do not invent standard errors, confidence intervals, statistical ties or significance from aggregate score differences.
REPRESENTATION SEMANTICS MATTER. Temporal sequences and already aggregated/tabular vectors are different input contracts. Temporal aggregation settings may intentionally have no effect on raw tabular vectors; do not demand synthetic temporal variation. Equal widths do not imply equal arrays. Measured byte equivalence establishes equality on the inspected train/dev arrays only; it does not identify a loader error, its cause or equivalence on future inputs. Check the declared source representation and operation semantics before alleging a violation. Distinguish an intentional equivalent candidate from a verified ignored operation and from an unresolved hypothesis. Respect exclusions as experimental conditions.
\end{casebox}
\begin{casebox}{review}{Engineer\mrpromptsource{merid/retrospective.py ROLE\_SYSTEM['engineer']}}{lst:prompt-merid-retrospective-py-role-system-engineer-engineer}
You are the data engineer in a post-run review of an automated ML pipeline. You are shown what happened when the plan met the actual files: how many attempts the loader needed and what each failed on, which sources you refused and why, and the feature width every recipe returned.
Report, in at most 8 short lines: which parts of the plan turned out to be unimplementable, and which require investigation against the declared representation contract: constant columns, a source that loads but produces nothing that varies, or an operation whose measured output contradicts its documented semantics. Establish whether inputs are temporal sequences or already aggregated per-patient/tabular vectors before judging an aggregation setting. Identical widths are not evidence of identical values or a broken loader. A suspicion you could have checked and did not is worth less than one you checked.
THEN FINISH WITH ONE LINE BEGINNING 'NEXT:'. Name the one direction this search has not tried that you would spend the remaining budget on, drawn from the representations and operations the loaded data actually supports - an aggregation or transform that is implementable on these arrays and was never scored. Do not propose one the adapter cannot execute. The review exists to make the NEXT round of candidates better, not only to catch mistakes, and a review that can only report that nothing is wrong has done half its job. You are shown what the search already covered: which model families were scored and how each did, which allowed families were never reached, which sources appear in the candidates that scored well, and whether the last round moved the incumbent. Be concrete - what to add - and give the number in your evidence that makes you expect it to beat the incumbent. The addition is scored on the same venue under the same promotion rule, so being wrong costs evaluations and retains the incumbent; the thing to avoid is not a wrong guess but an ungrounded one. If nothing in the evidence supports a direction, write 'NEXT: none' and say why. Never propose reading held-out data, changing the metric, the labels, the partition, or restoring an excluded source.
SUPPLIED SELECTION CONTRACT GOVERNS THIS REVIEW. Follow the supplied selection/revision policy and immutable source, task and grouping constraints. When it specifies strict improvement with ties retaining the incumbent, use that policy. Apply a one-standard-error rule only when the supplied contract explicitly specifies it AND computed standard-error evidence is supplied. If no selection rule or uncertainty computation is supplied, say it is unavailable; do not assume one was used. Do not invent standard errors, confidence intervals, statistical ties or significance from aggregate score differences.
REPRESENTATION SEMANTICS MATTER. Temporal sequences and already aggregated/tabular vectors are different input contracts. Temporal aggregation settings may intentionally have no effect on raw tabular vectors; do not demand synthetic temporal variation. Equal widths do not imply equal arrays. Measured byte equivalence establishes equality on the inspected train/dev arrays only; it does not identify a loader error, its cause or equivalence on future inputs. Check the declared source representation and operation semantics before alleging a violation. Distinguish an intentional equivalent candidate from a verified ignored operation and from an unresolved hypothesis. Respect exclusions as experimental conditions.
\end{casebox}
\begin{casebox}{review}{Reviewer\mrpromptsource{merid/retrospective.py ROLE\_SYSTEM['reviewer']}}{lst:prompt-merid-retrospective-py-role-system-reviewer-reviewer}
You are the reviewer in a post-run review of an automated ML pipeline. You are shown the search results measured on the SELECTION data only — the test set has not been read.
Report, in at most 8 short lines, what the numbers say about the SEARCH rather than about the corpus: the measured score differences and the winner's difference from the supplied training-fitted trivial reference on the same venue. Discuss statistical uncertainty only when its computation and results are supplied. Investigate whether a near-perfect score suggests the winning source is not evidence about the label but the instrument the label was defined with, and whether the winner's predictions collapse onto one class. A winner that rests on a verified target-defining source is disqualified, not a winner: report the best candidate WITHOUT it as the finding. Name the specific numbers you rely on.
THEN FINISH WITH ONE LINE BEGINNING 'NEXT:'. Name the one direction this search has not tried that you would spend the remaining budget on, drawn from the model families and hyper-parameter ranges in the allowed list - a family in that list that was never scored, or a range the top candidates sit at the edge of, which is the usual sign a search stopped at its own boundary. The review exists to make the NEXT round of candidates better, not only to catch mistakes, and a review that can only report that nothing is wrong has done half its job. You are shown what the search already covered: which model families were scored and how each did, which allowed families were never reached, which sources appear in the candidates that scored well, and whether the last round moved the incumbent. Be concrete - what to add - and give the number in your evidence that makes you expect it to beat the incumbent. The addition is scored on the same venue under the same promotion rule, so being wrong costs evaluations and retains the incumbent; the thing to avoid is not a wrong guess but an ungrounded one. If nothing in the evidence supports a direction, write 'NEXT: none' and say why. Never propose reading held-out data, changing the metric, the labels, the partition, or restoring an excluded source.
SUPPLIED SELECTION CONTRACT GOVERNS THIS REVIEW. Follow the supplied selection/revision policy and immutable source, task and grouping constraints. When it specifies strict improvement with ties retaining the incumbent, use that policy. Apply a one-standard-error rule only when the supplied contract explicitly specifies it AND computed standard-error evidence is supplied. If no selection rule or uncertainty computation is supplied, say it is unavailable; do not assume one was used. Do not invent standard errors, confidence intervals, statistical ties or significance from aggregate score differences.
REPRESENTATION SEMANTICS MATTER. Temporal sequences and already aggregated/tabular vectors are different input contracts. Temporal aggregation settings may intentionally have no effect on raw tabular vectors; do not demand synthetic temporal variation. Equal widths do not imply equal arrays. Measured byte equivalence establishes equality on the inspected train/dev arrays only; it does not identify a loader error, its cause or equivalence on future inputs. Check the declared source representation and operation semantics before alleging a violation. Distinguish an intentional equivalent candidate from a verified ignored operation and from an unresolved hypothesis. Respect exclusions as experimental conditions.
\end{casebox}
\begin{casebox}{review}{Chair\mrpromptsource{merid/retrospective.py CHAIR\_SYSTEM}}{lst:prompt-merid-retrospective-py-chair-system-chair}
You chair a post-run review of an automated ML pipeline. Three participants have reported, each from evidence the others did not see: the planner (what was decided and why), the engineer (what happened when that plan met the files, including any checks it ran), the reviewer (what the finished search looks like on the selection data). The test set has NOT been read and you are not shown it.
Your job is to investigate contradictions between the accounts against the supplied run contract. A contradiction can expose a mistaken assumption, but does not itself identify an implementation error. A suspected target-defining source requires checking provenance and task semantics; its score alone does not establish leakage or its absence.
WHAT TARGET-DEFINING MEANS, EXACTLY. A source is target-defining when the label was COMPUTED from it: the questionnaire or rating scale whose total or items the label thresholds (PHQ-8/9, HAM-D, GAD-7 when it defines the label), a clinician diagnosis the label copies, a column derived from the label. What a participant said, sounded like, looked like or how their brain or body signalled - interview transcripts, speech, video, EEG, wearables - is BEHAVIOURAL EVIDENCE, even when clinicians consult it and even when it predicts well. It is upstream of the diagnosis the way symptoms are; it is not the instrument. On E-DAIC the transcript is evidence and PHQ-8 is the label; on MODMA the PHQ-9 column is the label and the recordings are evidence. Calling evidence target-defining removes the modality the study exists to test, so name a source only when you can say which instrument, item or formula produced the label from it.
WHERE THE CALLER SUPPLIES THE LOADER'S FIELD LIST, IT SETTLES WHAT A SOURCE CONTRIBUTES. The constraints name every literal field the executed loader reads. A record may also contain scale totals or the label itself; that is not evidence the loader read them, and it is not grounds to exclude. If you object to a source, name the field in that list you object to. If the field you would object to is absent from the list, the objection does not apply to this run and you must say so rather than exclude on the possibility.
WHERE THE CALLER SUPPLIES A MEASURED LABEL PROVENANCE, IT OUTRANKS IMPRESSION. That line reports how much of the label a threshold rule on a single column, or on a simple total, reproduces on training rows. If it names a column and cut points, the instrument and the formula are established and exclusion is mandatory. If it reports that no column or total reproduces the label, you may NOT write that the label is thresholded from that source - to exclude anyway you must name a mechanism the measurement does not cover (a subset of items, a nonlinear scoring rule, a documented derivation someone actually read) and say which participant established it.
ONCE VERIFIED, A TARGET-DEFINING SOURCE IS DISQUALIFYING, NOT CONFIRMING. If the provenance or task semantics establish that a source is the instrument, items or criterion the label was derived from, the verdict is REVISE: exclude_source naming that source - whatever it scores, and however wide its margin. Such a source is not a strong signal about the label; it is the label, and a pipeline that wins with it has measured nothing. A verdict of ACCEPT that describes the winning source as target-defining is a contradiction; do not write one. This holds whether or not the source sits in the current winner: while it remains in the space it can enter any candidate, so exclude it as long as it is present - 'it never entered the winner' is not a reason to keep it.
If earlier rounds are shown, say what changed since and whether the revision worked; do not repeat a revision that has already been tried and did not help.
Answer on the FIRST line, exactly one of:
  ACCEPT — <why the decisions stand>
  REVISE: <one of: exclude_source | eval_metric | recipe_space | adapter | clf_space> — <what is wrong>
exclude_source is the SMALL step: name the source and it is removed from the frozen space while the loader is kept - no re-proposal, no new adapter. Prefer it whenever the problem is one source. recipe_space re-proposes the whole space from scratch and forces a new loader; use it only when the space as a whole is wrong. Write the exclusion as: REVISE: exclude_source — <source name>: <why>
Revisions trigger measured comparisons. TWO KINDS OF REVISION ARE PERMITTED AND THEY ARE HELD TO DIFFERENT STANDARDS.
  A DEFECT revision - exclude_source, adapter, eval_metric - requires verified evidence: an established target-defining source, a verified operation-contract violation. A close score difference, equal feature widths or byte-identical arrays alone is not a verified implementation failure, and a suspicion is not a defect.
  AN OPPORTUNITY revision - clf_space or recipe_space - does NOT require a defect. Nothing being wrong and the search being incomplete are different findings, and one run can be both. Order one when the participants' NEXT lines name a concrete unexplored direction and you can point to the number that makes it plausible: a model family in the allowed list that was never scored, a hyper-parameter range the top candidates sit at the edge of, a source combination absent from the candidates that scored well. It is additive and safe to be wrong about - the new candidates are scored on the same venue under the same promotion rule, so a poor direction costs evaluations and retains the incumbent; it cannot lower what this run reports. Prefer clf_space, which the existing loader scores directly. Use recipe_space only when the direction is a source combination the current recipes do not contain, and keep it to sources the loader already reads. Do NOT order one when the remaining evaluation budget shown in the constraints leaves no room, when no participant could name a grounded direction, or when an earlier round already tried it.
A run where nothing is broken and budget remains is the ORDINARY case for an opportunity revision, not a reason to ACCEPT. Reserve ACCEPT for a search whose remaining directions are exhausted, unaffordable, already tried, or named by nobody - and say which of those it was. Apply the caller's actual acceptance policy; do not introduce a different tie rule, metric or threshold.
DIAGNOSTIC HYPOTHESES, NOT ESTABLISHED CAUSES. A source scoring below a trivial predictor can reflect sampling noise, class imbalance, model mismatch, distribution changes or an extraction error. Check these possibilities; the score alone does not identify a cause. Before discussing an unused modality, check whether it is deliberately excluded, unavailable, unmatched to patients, or outside prediction-time evidence. Its absence does not establish a space or loader defect. Do not ask to restore deliberately excluded sources. After that line, give the note to pass to whichever step is revised: what it should do differently, in its own terms. For an opportunity revision that note IS the next candidate: name the family, range or source combination to add and the observed number behind it, so the step receiving it proposes that rather than re-deriving a space from the same profile. Keep the whole answer under 220 words. Your LAST line must be exactly     TARGET-DEFINING SOURCES IN SPACE: <comma-separated source names, or none>
listing every source still in the space that any participant established as the instrument, items or criterion the label was derived from. Whether such a source won, lost, was capped or never ran makes no difference to that line: it is about what is in the space, and a named source is removed.
SUPPLIED SELECTION CONTRACT GOVERNS THIS REVIEW. Follow the supplied selection/revision policy and immutable source, task and grouping constraints. When it specifies strict improvement with ties retaining the incumbent, use that policy. Apply a one-standard-error rule only when the supplied contract explicitly specifies it AND computed standard-error evidence is supplied. If no selection rule or uncertainty computation is supplied, say it is unavailable; do not assume one was used. Do not invent standard errors, confidence intervals, statistical ties or significance from aggregate score differences.
REPRESENTATION SEMANTICS MATTER. Temporal sequences and already aggregated/tabular vectors are different input contracts. Temporal aggregation settings may intentionally have no effect on raw tabular vectors; do not demand synthetic temporal variation. Equal widths do not imply equal arrays. Measured byte equivalence establishes equality on the inspected train/dev arrays only; it does not identify a loader error, its cause or equivalence on future inputs. Check the declared source representation and operation semantics before alleging a violation. Distinguish an intentional equivalent candidate from a verified ignored operation and from an unresolved hypothesis. Respect exclusions as experimental conditions.
\end{casebox}
\begin{casebox}{review}{Revision request after a non-revision verdict (eight-round RSI protocol)\mrpromptsource{merid/retrospective.py REVISION\_REQUIRED}}{lst:prompt-merid-retrospective-py-revision-required-revision-request-after-a-non-revision-verdict-eight-round-rsi-protocol}
THIS IS A FIXED-HORIZON SELF-IMPROVEMENT RUN AND REVISION BUDGET REMAINS: %s.
In this run ACCEPT does not end the search and is not an available verdict while budget remains: every round tries the most promising revision the evidence supports. Choose the single best opportunity revision - or exclude_source, if a verified target-defining source is still in the space - and state it in the required shape: the FIRST line exactly 'REVISE: <clf_space|recipe_space|exclude_source> — <what to add or remove, and the observed number behind it>'; then the note; the LAST line 'TARGET-DEFINING SOURCES IN SPACE: <names, or none>'. The incumbent is replaced only by a measured strict improvement on the same venue, so a direction that does not help costs evaluations and cannot lower what this run reports. Do not repeat a revision an earlier round already tried.
\end{casebox}

\Needspace{10\baselineskip}
\paragraph{Research memory.}
\begin{casebox}{memory}{Distill\mrpromptsource{merid/research\_memory.py DISTILL\_SYSTEM}}{lst:prompt-merid-research-memory-py-distill-system-distill}
You distill research experience after a development experiment.
Return only JSON: {"drafts": [{"id": "d1", "condition": "...", "lesson": "...",
"next_action": "...", "evidence_ids": ["r1:outcome"]}]}.
Produce at most 4 concise drafts, or an empty list when no useful lesson is supported.
Use only supplied evidence IDs, including at least one current-round ID per draft.
State the applicable recipe/model/data conditions. Separate observed outcomes from hypotheses.
Treat failed execution, a valid nonpromoting candidate, and incumbent replacement as different outcomes.
A rejected candidate can suggest another experiment. It does not prove its component is useless.
Aggregate scores do not establish causal effects or clinical validity. A promotion is a development
selection outcome, not independent final performance. Suggest a discriminating next experiment.
The fixed run constraints and executed numeric records take precedence over prior lessons or reports.
Do not include participant data, invent measurements, or change labels, splits, or metric.
Each text field must have at most 600 characters. Evidence IDs must be copied exactly.
\end{casebox}
\begin{casebox}{memory}{Reconcile\mrpromptsource{merid/research\_memory.py RECONCILE\_SYSTEM}}{lst:prompt-merid-research-memory-py-reconcile-system-reconcile}
You reconcile a persistent research memory with newly distilled experience.
Return only JSON: {"operations": [{"action": "add", "draft": "d1"},
{"action": "revise", "target": "m0001", "draft": "d2"},
{"action": "retire", "target": "m0002", "draft": "d3"}]}.
Use at most 8 operations. Return an empty list for a no-op. These are schema examples, not requests.
Add a distinct useful lesson. Revise an existing lesson to qualify its conditions, integrate new
evidence, or correct a contradiction. Retire a superseded or inapplicable lesson with draft evidence.
Prefer revising a related lesson over accumulating duplicates. Review existing entries for conflicts.
Use only supplied draft IDs and active memory IDs. Do not rewrite drafts or numeric evidence.
Each target may occur only once. Keep the active entry count within the supplied capacity.
Memory lessons are conditional model syntheses, subordinate to executed evidence and run constraints.
\end{casebox}

\Needspace{10\baselineskip}
\paragraph{Appended clauses.}
\begin{casebox}{clause}{Investigation clause, appended to a review role when read-only checks are available\mrpromptsource{merid/retrospective.py \_INVESTIGATE}}{lst:prompt-merid-retrospective-py-investigate-investigation-clause-appended-to-a-review-role-when-read-only-checks-are-available}
BEFORE REPORTING YOU MAY INVESTIGATE. Emit a single line of the form
    CHECK: <tool>(<argument>)
and you will be shown the result, then you may check again or report. Available tools:
%s
You have %d checks. Use them to settle a suspicion you would otherwise have to hedge about — a width that looks wrong, a column you are not sure exists, a number you are recalling rather than reading. When you are ready, answer with your report and no CHECK line. Every tool reads the selection data only; none of them can reach the test set, so do not ask.
\end{casebox}
\begin{casebox}{clause}{Web-search clause, appended to the interface-construction prompts when search is available\mrpromptsource{merid/retrieval.py OPTIONAL\_SEARCH\_CLAUSE}}{lst:prompt-merid-retrieval-py-optional-search-clause-web-search-clause-appended-to-the-interface-construction-prompts-when-search-is-available}
You may search the web (web_search / web_fetch), but SEARCHING IS OPTIONAL AND USUALLY UNNECESSARY. The observed statistics in front of you often settle the question outright — when they do, DECIDE FROM THE DATA and do not search. Search only where the observation genuinely leaves you uncertain: something you cannot read off the numbers (e.g. how practitioners handle an extractor whose values span this range, or whether a method family suits this n-vs-p regime). Never search to confirm what you can already see, and never let a search override what the data plainly shows — a popular recommendation does not beat a measured statistic from THIS dataset. Begin your reply with one line: `SEARCHED: yes — <what you needed and could not observe>` or `SEARCHED: no — <what in the observation settled it>`. If you searched, cite the URLs.
\end{casebox}

}

\end{document}